\documentclass[11pt]{article}
\usepackage{fullpage,amssymb,amsmath,url}
\usepackage{makeidx,mdwtab}  
\usepackage{placeins}
\usepackage{graphicx,upgreek,morenotations,rotating}
\usepackage{subfigure} 

\usepackage{algorithm}
\usepackage{algorithmic}

\usepackage[margin=1in]{geometry}
 
\usepackage{color}
\newcommand{\bY}{{\color{blue} Y}}
\newcommand{\rN}{{\color{red} N}}

\author{
  Richard Nock\\
NICTA \& the Australian National University\\
  \texttt{richard.nock@nicta.com.au}
\and
  Giorgio Patrini\\
NICTA \& the Australian National University\\
  \texttt{giorgio.patrini@anu.edu.au}
\and
  Arik Friedman\\
NICTA \& the University of New South Wales\\
  \texttt{arik.friedman@nicta.com.au}
}

\begin{document} 

\title{Rademacher Observations,  Private
  Data, and Boosting}

\maketitle 
\begin{abstract} 
The minimization of the logistic loss is a popular approach to batch
supervised learning. Our paper starts from the surprising
observation that, when fitting linear (or kernelized) classifiers, the
minimization of the logistic loss is \textit{equivalent} to the
minimization of an exponential \textit{rado}-loss computed (i) over transformed data
that we call Rademacher observations (rados), and (ii) over the \textit{same} classifier as the one of
the logistic loss. Thus, a classifier learnt from rados can be \textit{directly} used to classify
\textit{observations}. We provide a learning algorithm over
rados with
boosting-compliant convergence rates on the \textit{logistic loss} (computed over
examples). Experiments on
domains with up to millions of examples, backed up by
theoretical arguments, display
that learning over a small set of random rados 
can challenge the state of the art that learns over the
\textit{complete} set of examples. We show that rados comply with
various privacy requirements that make them good candidates for
machine learning in a privacy framework. We give several algebraic, geometric and
computational hardness results on reconstructing examples from
rados. We also show how it is possible to craft, and efficiently learn from,
rados in a differential privacy framework. Tests reveal that
learning from differentially private rados can compete with learning from
random rados, and hence with batch learning from examples, 
achieving non-trivial privacy vs accuracy tradeoffs.
\end{abstract} 

\section{Introduction}

This paper deals with the following fundamental question:
\begin{itemize}
\item [] \textit{What information is sufficient for learning, and what
    guarantees can it bring that regular data cannot} ?
\end{itemize}
By ``regular'', we mean the usual inputs provided to a learner. In our
context of batch supervised learning, this is a training set of examples,
each of which is an observation with a class, and learning means
inducing in reduced time an accurate function from observations to
classes, a \textit{classifier}. It turns out that we do not need the
detail of classes
to learn a classifier (linear or kernelized): an aggregate, whose size
is the dimension of the
observation space, is minimally sufficient, the mean operator \cite{pnrcAN}.

But do we need examples ?

This perhaps surprising and non-trivial question is becoming crucial
now that the nature of stored and processed signals intelligence data is
heavily debated in the public sphere \cite{lCU,seaBC}. In the context of machine learning (ML), the objective of being
accurate is more and more frequently subsumed by more complex
goals, sometimes involving challenging tradeoffs in which accuracy does not
ultimately appear in the topmost requirements. 
Privacy is one such crucial
goal \cite{djwPA,ecTE,gBP}. There are various models to capture the
privacy requirement, such as secure multi-party computation and
differential privacy (DP, \cite{drTA}). The former usually relies on cryptographic
protocols, which can be heavy
even for bare classification and simple algorithms \cite{bptgML}. The latter usually relies on the power of
randomization
to ensure that any ``local'' change cannot be
spotted from the output delivered \cite{drvBA,drTA}. In a ML setting, randomization can be performed at various stages, from
the examples to the output of a classifier. We focus on the
upstream stage of the process, \textit{i.e.} the input to the learner, which grants
the benefits that \textit{all} subsequent stages also comply with
differential privacy. Randomization has its
power: it also has its limits in this case, as it may significantly
degrade the performance of learners. 

The way we address this problem starts from a surprising
observation, whose relevance to supervised ML goes beyond
learning with private data: learning a linear
(or kernelized) classifier over examples throughout the minimization
of the expected logistic loss is equivalent to
learning \textit{the same classifier} by minimizing an exponential loss
over a complete set of transformed data that we call \textit{Rademacher observations},
rados. Each rado is the sum of \textit{edge vectors} over 
examples (edge = observation $\times$ label). We also show that
efficient learning from all rados may also be achieved when carried out over 
\textit{subsets} of all possible rados. 

This is our first contribution, and we expect it to be useful in
several other areas of supervised learning. In the context of
learning with private data, our other contributions can be summarized
as showing how rados may yield new privacy
guarantees --- not limited to differential privacy --- while
authorising
boosting-compliant rates for learning. 
More precisely, our second contribution is to propose a rado-based learning algorithm, which has boosting-compliant convergence rates
over the \textit{logistic loss computed over the examples}. Thus, we
learn an accurate classifier over rados, and the same classifier is accurate
over examples as well. 

The fact that efficient learning may be achieved through subset of
rados is interesting because it opens the problem of
designing this particular subset to address domain-specific requirements
that add to the ML accuracy requirement. Among our other
contributions, we provide one important
design example, showing how to build differentially private mechanisms for rado
delivery, such as when protecting specific sensitive
features in data. Experiments confirm in this case that learning from
differentially private 
rados may still be
competitive with learning from examples.
We provide another design which pairs to our
rado-based boosting algorithm, with the crucial property that when examples have
been DP-protected by the popular Gaussian mechanism \cite{drTA}, the
joint pair (rado delivery design, boosting
algorithm) may achieve convergence rates \textit{comparable to the
noise-free} setting with high probability, even over strong DP
protection regimes.
Our last contribution is to show that rados may protect the privacy of the
original examples not only in the DP framework, but also from several
algebraic, geometric and even computational-complexity theoretic
standpoints. 

The remainder of this paper is organized as follows. Section \textsection\ref{srasl} presents
Rademacher observations, shows the equivalence between learning from
examples and learning from rados, and how learning from subsets of
rados may be sufficient for efficient learning; \textsection\ref{sbur}
presents our rado-based boosting algorithm, and 
\textsection\ref{exp_boost_rado} presents experiments with this
algorithm; \textsection\ref{sradp} presents our results in
DP models, \textsection\ref{exp_dp_rado} presents related
experiments; \textsection\ref{sec_hardness} provides results on
the hardness of reconstructing examples from rados from 
algebraic, geometric and computational standpoints. 
To keep a readable paper, proofs and additional experiments are given
in two separate appendices available in Section \ref{app_proof_proofs}
(proofs) and Section \ref{app_exp_expes}
(experiments).

\section{Rados and supervised learning}\label{srasl}

Let $[n] = \{1, 2, ..., n\}$. We are given a set of $m$ examples
${\mathcal{S}} \defeq \{(\ve{x}_i, y_i), i \in [m]\}$, where $\ve{x}_i
\in {\mathcal{X}} \subseteq {\mathbb{R}}^d$ is an observation and $y_i
\in \{-1,1\}$ is a label, or class. ${\mathcal{X}}$ is the domain. A linear
classifier $\ve{\theta}
\in {\Theta}$ for some fixed ${\Theta} \subseteq {\mathbb{R}}^d$ gives
a label to $\ve{x} \in {\mathcal{X}}$ equal to the sign of $\ve{\theta}^\top \ve{x} \in {\mathbb{R}}$. Our results
can be lifted to kernels (at least with finite
dimension feature maps) following standard arguments
\cite{qsclEL}.
We let $\Sigma_m \defeq
\{-1,1\}^m$.
\begin{definition}
For any $\ve{\sigma} \in \Sigma_m$, the Rademacher observation
$\ve{\rado}_{\ve{\sigma}}$ with signature $\ve{\sigma}$ is
$\ve{\rado}_{\ve{\sigma}} \defeq (1/2) \cdot \sum_i (\sigma_i + y_i)
\ve{x}_i$.
\end{definition}
The simplest way to randomly sample rados is to pick $\ve{\sigma}$ as
i.i.d. Rademacher variables, hence the name.
Reference to ${\mathcal{S}}$ is implicit in the definition of $\ve{\rado}_{\ve{\sigma}}$.
A Rademacher observation sums \textit{edge vectors} (the terms $y_i\ve{x}_i$),
over the subset of examples for which $y_i = \sigma_i$. When
$\ve{\sigma} = \ve{y}$ is the vector of classes,
$\ve{\rado}_{\ve{\sigma}} = m \ve{\mu}_{{\mathcal{S}}}$ is $m$ times the mean operator
\cite{qsclEL,pnrcAN}. When $\ve{\sigma} = -\ve{y}$, we get the null
vector $\ve{\rado}_{\ve{\sigma}} = \ve{0}$. 
A popular approach to learn $\ve{\theta}$ over ${\mathcal{S}}$ is to
minimize the surrogate risk $\logloss\left({\mathcal{S}},
  \ve{\theta}\right)$ built from the logistic loss (logloss):
\begin{eqnarray}
\logloss\left({\mathcal{S}}, \ve{\theta}\right) & \defeq & \frac{1}{m} \sum_{i}
\log\left(1+\exp\left(-y_i \ve{\theta}^\top
    \ve{x}_i\right)\right)\:\:. \label{deflogloss}
\end{eqnarray}
We define the \textit{exponential rado-risk}
$\explossrado({\mathcal{S}}, \ve{\theta}, \mathcal{U})$, computed on any
${\mathcal{U}} \subseteq \Sigma_m$ with cardinal $|{\mathcal{U}}| = n$, as:
\begin{eqnarray}
\explossrado({\mathcal{S}}, \ve{\theta}, \mathcal{U}) & \defeq & \frac{1}{n} \sum_{\ve{\sigma} \in {\mathcal{U}}} \exp\left( -
  \ve{\theta}^\top\ve{\rado}_{\ve{\sigma}}\right)\label{defExp}\:\:.
\end{eqnarray}
It turns out that $\logloss = g(\explossrado)$ for some continuous strictly
increasing $g$; hence, minimizing one criterion is equivalent to
minimizing the other and \textit{vice versa}. This is stated formally in the
following Lemma.
\begin{lemma}\label{lem_equivlogexp}
The following holds true, for any $\ve{\theta}$ and ${\mathcal{S}}$:
\begin{eqnarray}
\logloss({\mathcal{S}}, \ve{\theta}) & = & \log(2) + \frac{1}{m}
\log \explossrado({\mathcal{S}}, \ve{\theta}, \Sigma_m)\:\:. \label{eqq1}
\end{eqnarray}
\end{lemma}
(Proof in the Appendix, Subsection
\ref{proof_lem_equivlogexp}). 
Lemma \ref{lem_equivlogexp} shows that learning with examples
via the minimization of $\logloss\left({\mathcal{S}}, \ve{\theta}\right)$, and
learning with all rados via the minimization of
$\explossrado({\mathcal{S}}, \ve{\theta}, \Sigma_m)$, are essentially equivalent
tasks. Since the cardinal $|\Sigma_m| = 2^m$ is exponential, it is 
unrealistic, even on moderate-size samples, to pick that latter
option. This raises however a very interesting question: if we replace
$\Sigma_m$ by subset ${\mathcal{U}}$ of size $\ll 2^m$,what does the
relationship between examples and rados in eq. (\ref{eqq1}) become?
We answer this question under the setting that:
\begin{itemize}
\item [(i)] instead of $\Sigma_m$, we consider a predefined $\Sigma_r \subseteq \Sigma_m$;
\item [(ii)] instead of considering ${\mathcal{U}} = \Sigma_r$, we
  sample uniformly i.i.d. ${\mathcal{U}} \sim \Sigma_r$ for $n \geq 1$ rados.
\end{itemize}
While (ii) is directly targeted at reducing the number of rados, (i)
is an upper-level strategic design to tackle additional constraints,
such as differential privacy. We now need 
following definition of the \textit{logistic rado-risk}:
\begin{eqnarray}
\loglossrado\left({\mathcal{S}}, \ve{\theta}, \mathcal{U}\right) & \defeq & \log(2) + \frac{1}{m}
\log \explossrado({\mathcal{S}}, \ve{\theta}, \mathcal{U})\:\:, \label{deflogSU}
\end{eqnarray}
for any ${\mathcal{U}} \subseteq \Sigma_m$, so that
$\logloss\left({\mathcal{S}}, \ve{\theta}\right) =
\loglossrado\left({\mathcal{S}}, \ve{\theta}, \Sigma_m\right)$. We also
define the open ball ${\mathcal{B}}(\ve{0},r) \defeq \{\ve{x} \in {\mathbb{R}}^d :
\|\ve{x}\|_2 < r\}$.

\begin{theorem}\label{thm_concentration}
Assume $\Theta \subseteq {\mathcal{B}}(\ve{0},r_\theta)$, for some
$r_\theta > 0$. Let:
\begin{eqnarray*}
\varrho & \defeq & \frac{ \sup_{\ve{\theta}' \in \Theta}
\max_{\ve{\rado}_{\ve{\sigma}} \in \Sigma_r} \exp(-\ve{\theta}'^\top
\ve{\rado}_{\ve{\sigma}})}{\explossrado({\mathcal{S}}, \ve{\theta},
\Sigma_r)}\:\:,\\
\varrho' & \defeq & \frac{\explossrado({\mathcal{S}},
    \ve{\theta}, \Sigma_r)}{\explossrado({\mathcal{S}},
    \ve{\theta}, \Sigma_m)} \:\:,
\end{eqnarray*}
where $\Sigma_r$ follows (i) above. Then $\forall \upeta > 0$, there is probability $\geq 1 -
\upeta$ over the sampling of ${\mathcal{U}}$ in (ii) above that:
\begin{eqnarray}
\logloss\left({\mathcal{S}}, \ve{\theta}\right) & \leq & \loglossrado({\mathcal{S}}, \ve{\theta},
  \mathcal{U})  + Q - \frac{1}{m} \cdot \log\left(1 - \frac{q}{\sqrt{n}}\right) \:\:,\label{thc11}
\end{eqnarray}
with 
\begin{eqnarray}
q & = &  \Omega\left( \varrho \cdot \sqrt{r_\theta \max_{\Sigma_r}
  \left\|\ve{\rado}_{\ve{\sigma}}\right\|_2 +
  d\log \frac{2en}{d} + \log
  \frac{1}{\upeta}} \right)\label{eq001}
\end{eqnarray}
and $Q \defeq - (1/m) \cdot \log \varrho'$ satisfies $Q = 0$ if
$\Sigma_r = \Sigma_m$ and
\begin{eqnarray}
Q & \leq & r_\theta\left(\|\ve{\nabla}_{\ve{\theta}}
 \loglossrado\left({\mathcal{S}}, \ve{\theta}, \Sigma_m\right)\|_2 + \overline{\pi}_r\right)\label{bsupQ}
\end{eqnarray}
otherwise, letting $\overline{\pi}_r \defeq \left\|\expect_{\ve{\sigma}\sim \Sigma_r}
  (1/m) \cdot \ve{\rado}_{\ve{\sigma}}\right\|_2$. Furthermore, $\forall 0\leq \beta < 1/2$, if $m$ is sufficiently
large, then letting $\pi_r^* \defeq \max_{\Sigma_r}
  \left\|(1/m) \cdot \ve{\rado}_{\ve{\sigma}}\right\|_2$,
  ineq. (\ref{thc11}) becomes:
\begin{eqnarray}
\logloss\left({\mathcal{S}}, \ve{\theta}\right) & \leq & \loglossrado({\mathcal{S}}, \ve{\theta},
  \mathcal{U}) + Q\nonumber\\
 & &  + O\left( \frac{\varrho}{m^\beta}
    \cdot \sqrt{\frac{r_\theta \pi_r^*}{n} +
  \frac{d}{n m}\log \frac{2en}{d \upeta} } \right) \:\:.\label{thc22}
\end{eqnarray}
\end{theorem}
(Proof in the Appendix, Subsection
\ref{proof_thm_concentration}) 
Theorem
\ref{thm_concentration} does not depend on the algorithm that learns $\ve{\theta}$.
The right-hand side of ineq. (\ref{thc11}) shows two
penalties. $Q$ arises from the choice of $\Sigma_r$ and is therefore
structural. Regardless of $\Sigma_r$, when the classifier is
reasonably accurate
over all rados and expected examples edges in $\Sigma_r$ average to a ball of
reduced radius, the upperbound on $Q$ in ineq. (\ref{bsupQ}) can be
very small.
The other penalty, which depends on $q$, is statistical and
comes from the sampling in $\Sigma_r$. Theorem \ref{thm_concentration}
shows that when $\Sigma_r = \Sigma_m$, even when $n\ll m$, the minimization of $\loglossrado\left({\mathcal{S}}, \ve{\theta},
  \mathcal{U}\right)$ may still bring, with high probability, guarantees
on the minimization of $\logloss\left({\mathcal{S}},
  \ve{\theta}\right)$. Thus, a lightweight optimization procedure over
a small number of rados may bring guarantees on the minimization of the expected logloss
over \textit{examples} for the \textit{same} classifier. The following Section exhibits one such algorithm.
\begin{algorithm}[t]
\caption{Rado boosting ({\small \radoboost})}\label{algoRadoboost}
\begin{algorithmic}
\STATE  \textbf{Input} set of rados ${\mathcal{S}}^r \defeq
\{\ve{\rado}_{1},\ve{\rado}_{2}, ...,
\ve{\rado}_{n}\}$; $T\in {\mathbb{N}}_*$;
\STATE  Step 1 : let $\ve{\theta}_0 \leftarrow \ve{0}$, $\ve{w}_0
\leftarrow (1/n)\ve{1}$ ; 
\STATE  Step 2 : \textbf{for} $t = 1, 2, ..., T$
\STATE  \hspace{1.1cm} Step 2.1 : $[d] \ni \iota(t) \leftarrow
\textsc{\weak}({\mathcal{S}}^r, \ve{w}_t)$; 
\STATE  \hspace{1.1cm} Step 2.2 : let
\begin{eqnarray}
r_t & \leftarrow & \frac{1}{\rado_{*\iota(t)}}
\sum_{j=1}^{n} {w_{tj} \rado_{j \iota(t)}}\:\:;\label{defMu}\\
\alpha_{t} & \leftarrow & \frac{1}{2 \rado_{*\iota(t)}}
\log \frac{1 + r_t}{1 -
  r_t}\:\:;\label{defalpha}
\end{eqnarray}
\STATE  \hspace{1.1cm} Step 2.3 : \textbf{for} $j = 1, 2, ..., n$
\begin{eqnarray}
w_{(t+1)j} & \leftarrow & w_{tj} \cdot
\left(\frac{1-\frac{r_t \rado_{j \iota(t)}}{\rado_{*\iota(t)}}}{1-r^2_{t}}\right) \:\:;\label{defweights}
\end{eqnarray}
\STATE \textbf{Return} $\ve{\theta}_T$ defined by $\theta_{Tk} \defeq \sum_{t:\iota(t) = k} \alpha_{t}
\:\:, \forall k \in [d]$;
\end{algorithmic}
\end{algorithm}

\section{Boosting using rados}\label{sbur}

\begin{table}[t]
\begin{center}
{\scriptsize
\begin{tabular}{|crrc|r|rr|rr|c|c|}
\hline \hline
 & & & &  \multicolumn{1}{c|}{{\scriptsize \adaboostSS}} & \multicolumn{2}{c|}{{\scriptsize \adaboostSSS}} & 
 \multicolumn{2}{c|}{{\scriptsize \radoboost}} & & \\
Domain & \multicolumn{1}{c}{$m$} & \multicolumn{1}{c}{$d$}  &
100$\sigma$ & \multicolumn{1}{c|}{ err$\pm\sigma$} & \multicolumn{1}{c}{err$\pm\sigma$} & \multicolumn{1}{c|}{$\frac{n}{m}$} &
\multicolumn{1}{c}{err$\pm\sigma$} & \multicolumn{1}{c|}{$\frac{n}{2^m}$} & $p$ & $p'$\\ \hline 
Fertility & 100 & 9 & -- & 47.00$\pm$18.99 & 
44.00$\pm$16.47 & $0.50$ & 53.00$\pm$14.94 & {\tiny [$8$:$-28$]}
& 0.23 & 0.09
\\  
Haberman & 306 & 3 & -- & 25.72$\pm$10.62 & 33.01$\pm$9.58
& $0.50$ & 26.08$\pm$9.94 & {\tiny [$8$:$-90$]} & 0.70 
& 0.02\\
Transfusion & 748 & 4 & -- & 39.42$\pm$6.13 & 37.83$\pm$4.94
& $0.50$ & 39.29$\pm$5.76 &{\tiny [$7$:$-223$]} & 0.81 
& 0.36\\
Banknote & 1 372 & 4 & -- & 2.77$\pm$1.28 &
2.63$\pm$1.34 & $0.50$ & 14.21$\pm$3.22 
& {\tiny [$9$:$-411$]} & $\varepsilon$ & $\varepsilon$\\
Breast wisc & 699 & 9 & -- & 3.00$\pm$1.42 & 
3.43$\pm$2.25 & $0.50$ & 4.86$\pm$2.35 & {\tiny [$4$:$-208$]} &
0.03 & 0.13\\
Ionosphere & 351 & 33 & -- & 11.69$\pm$5.31 & 11.70$\pm$4.77 & $0.50$ &
15.40$\pm$9.93 & {\tiny [$2$:$-103$]} & 0.13 & 0.09\\
Sonar & 208 & 60 & -- & 26.88$\pm$9.36 & 25.43$\pm$6.61 & $0.50$ &
28.36$\pm$8.84 & {\tiny [$2$:$-60$]} & 0.76 & 0.42\\
Wine-red$^*$ & 1 599 & 11 & 1 & 26.14$\pm$3.10 & 26.39$\pm$3.15 & $0.50$
&
28.02$\pm$2.90 & {\tiny [$4$:$-479$]} & 0.05 & 
0.03\\
Abalone$^*$  & 4 177 & 8 & -- & 22.96$\pm$1.44 & 
23.20$\pm$1.44  & $0.24$ & 25.14$\pm$1.83 
& {\tiny [$3$:$-$[$1$:$3$]]} & $\varepsilon$ & $\varepsilon$\\
Wine-white$^*$ & 4 898 & 11 & 1 & 30.93$\pm$3.42 & 30.44$\pm$3.25 & $0.20$
&
32.48$\pm$3.55 & {\tiny [$3$:$-$[$1$:$3$]]} & $\varepsilon$ & 
$\varepsilon$\\
Magic$^*$ & 19 020 & 10 & -- & 21.07$\pm$0.98 & 20.91$\pm$0.99 & $0.05$ &
22.75$\pm$1.51 & {\tiny [$3$:$-$[$5$:$3$]]} & $\varepsilon$ & 
0.01\\
EEG & 14 980 & 14 & 14 & 46.04$\pm$1.38 & 44.36$\pm$1.99 & $0.07$ &
44.23$\pm$1.73 & {\tiny [$4$:$-$[$4$:$3$]]} & $\varepsilon$ & 
0.86\\
Hardware$^*$ & 28 179 & 95 & -- & 16.82$\pm$0.72 & 
16.76$\pm$0.73 & $0.04$ & 7.61$\pm$3.24 & {\tiny [$2$:$-$[$8$:$3$]]}
& 
$\varepsilon$ & $\varepsilon$  \\
Twitter$^*$ & 583 250 & 77 & 44 & 53.75$\pm$1.48 & 
53.09$\pm$11.23 & {\tiny [$1$:$-3$]} & 6.00$\pm$0.77 & {\tiny
  [$1$:$-$[$1$:$5$]]} &
$\varepsilon$ & $\varepsilon$  \\
SuSy & 5 000 000 & 17 & -- & 27.76$\pm$0.14 & 
27.43$\pm$0.19 & {\tiny [$2$:$-4$]} &  27.26$\pm$0.55 &  {\tiny
  [$1$:$-$[$1$:$6$]]} &
0.02 & 0.39  \\
Higgs & 11 000 000 & 28 & -- & 42.55$\pm$0.19 & 45.39$\pm$0.28 &
{\tiny [$9$:$-5$]} &
47.86$\pm$0.06 & {\tiny [$1$:$-$[$1$:$7$]]} & $\varepsilon$ & $\varepsilon$\\
\hline\hline
\end{tabular}
}
\end{center}
\caption{Comparison of \radoboost~($n$ random rados),
  \adaboostSS~\cite{ssIBj} (full training fold) and
  \adaboostSSS~($n$ random examples in training fold);
  domains ranked in increasing $d\cdot m$ value. Column ``$n/m$''
  (resp. ``$n/2^m$'') for \adaboostSSS~(resp \radoboost) is
  proportion of training data with respect to fold size (resp. full set of rados). Notation [$a$:$b$] is shorthand for
  $a \times 10^{b}$. Column ``$100\sigma$'' is the
  number of features with outlier values distant from the mean by more
  than $100\sigma$ in absolute value. 
Column $p$ (resp. $p'$) is 
  $p$-value for a two-tailed paired $t$-test on
  \adaboostSS~(resp. \adaboostSSS) vs \radoboost.  $\varepsilon$ means $<0.01$.}
  \label{tc1_full}
\end{table}

Algorithm \ref{algoRadoboost} provides a boosting algorithm,
\radoboost, that learns from a set of Rademacher
observations ${\mathcal{S}}^r \defeq
\{\ve{\rado}_{1},\ve{\rado}_{2}, ...,
\ve{\rado}_{n}\}$. Their (unknown) Rademacher assignments are denoted 
${\mathcal{U}} \defeq \{\ve{\sigma}_1, \ve{\sigma}_2,
..., \ve{\sigma}_n\} \subseteq \Sigma_m$. These
rados have been computed from some sample ${\mathcal{S}}$, unknown to \radoboost. In the statement of the algorithm, 
$\rado_{jk}$ denotes coordinate $k$ of $\ve{\rado}_{j}$, and
$\rado_{*k} \defeq \max_j |\rado_{jk}|$. More generally, the coordinates of some vector $\ve{z}
\in {\mathbb{R}}^d$ are denoted $z_1, z_2, ..., z_d$. Step 2.1 gets a
feature index $\iota(t)$ from a \textit{weak feature
index oracle}, $\weak$. In its general form, \weak~returns a feature
index maximizing $|r_t|$ in (\ref{defMu}). The weight
update was preferred to AdaBoost's because rados can have large
feature values and the weight update prevents numerical precision
errors that could otherwise occur using AdaBoost's exponential weight
update. We now prove a key Lemma on \radoboost, namely the fast
convergence of the exponential rado-risk
$\explossrado({\mathcal{S}}, \ve{\theta}, {\mathcal{U}})$ under a weak
learning assumption (\textbf{WLA}). We shall then obtain the convergence of the
logistic rado-risk (\ref{deflogSU}), and, via Theorem
\ref{thm_concentration}, the convergence with high probability of $\logloss\left({\mathcal{S}}, \ve{\theta}\right)$. 
\begin{itemize}
\item [(\textbf{WLA})] $\exists \upgamma > 0$ such that $\forall t\geq 1$, the feature returned by
  $\weak$  in Step 2.2 (\ref{defMu}) satisfies $|r_t| \geq \upgamma$.
\end{itemize}
\begin{lemma}\label{lem_radoboost}
Suppose the (\textbf{WLA}) holds. Then after $T$ rounds of boosting in
\radoboost,
the following upperbound holds on the exponential rado-loss of $\ve{\theta}_T$:
\begin{eqnarray}
\explossrado({\mathcal{S}}, \ve{\theta}_T, \mathcal{U}) & \leq & \exp\left(-T\upgamma^2/2\right)\:\:.\label{explossbound}
\end{eqnarray}
\end{lemma}
(Proof in the Appendix, Subsection
\ref{proof_lem_radoboost}) We now consider Theorem
\ref{thm_concentration} with $\Sigma_r = \Sigma_m$, and therefore $Q=0$. 
Blending Lemma \ref{lem_radoboost} and
Theorem \ref{thm_concentration} using (\ref{deflogSU}) yields that, under the (\textbf{WLA}),
we may observe with high probability (again, fixing $\Sigma_r =
\Sigma_m$, so $Q=0$ in Theorem \ref{thm_concentration}):
\begin{eqnarray}
\logloss\left({\mathcal{S}}, \ve{\theta}_T\right)  & \leq & \log(2) -
\frac{T\upgamma^2}{2m} + Q'\:\:,\label{Qbound}
\end{eqnarray}
where $Q'$ is the rightmost term in ineq. (\ref{thc11}) or
ineq. (\ref{thc22}). So provided $n \ll 2^m$ is sufficiently
large, minimizing the exponential
rado-risk over a \textit{subset of rados} brings a classifier whose
average logloss on the \textit{whole set of examples} may
decrease at rate $\Omega(\upgamma^2/m)$ under a weak learning
assumption made over \textit{rados} only. This rate competes with those
for direct approaches to boosting the logloss \cite{nnOT}, and we now
show that our weak learning assumption is also essentially equivalent
to the one done in boosting over
examples \cite{ssIBj}. Let us rewrite
$r_t(\ve{w})$ as the normalized edge in (\ref{defMu}), making explicit
the dependence in the current rado weights. Let
\begin{eqnarray}
r_t^{ex}(\tilde{\ve{w}}) & \defeq & \frac{1}{x_{*\iota(t)}} \sum_{i=1}^{m}
{w_{i} x_{i \iota(t)}} \label{exedge}
\end{eqnarray}
be the normalized edges for the same feature $\iota(t)$ as the one
picked in step 2.1 of \radoboost, but computed over examples using
some weight vector
$\tilde{\ve{w}} \in {\mathbb{P}}^m$; here, ${\mathbb{P}}^m$ is the
$m$-dim probability simplex and $x_{*\iota(t)} \defeq \max_i |x_{ik}|$.
\begin{lemma}\label{lem_wla}
$\forall \ve{w}_t
\in {\mathbb{P}}^n$, $\forall \upgamma > 0$, there exists
$\tilde{\ve{w}} \in {\mathbb{P}}^m$ and $\upgamma^{ex} > 0$
such that $|r_t(\ve{w}_t)| \geq \upgamma$ iff
$|r_t^{ex}(\tilde{\ve{w}})| \geq \upgamma^{ex}$.
\end{lemma}
(Proof in the Appendix, Subsection
\ref{proof_lem_wla}) The proof of the Lemma gives clues to explain why the
presence of outlier feature values may favor \radoboost.

\section{Basic experiments with \radoboost}\label{exp_boost_rado}

We have compared \radoboost~to its main contender,
\adaboostSS~\cite{ssIBj}, using the same weak learner; in \adaboostSS,
it returns a
feature maximizing $|r_t|$ as in eq. (\ref{exedge}). In these basic experiments, we have deliberately not
optimized the set of rados in which we sample ${\mathcal{U}}$ for \radoboost; hence,
we have $\Sigma_r = \Sigma_m$.

We have performed comparisons with 10 folds stratified
cross-validation (CV) on 16 domains of the UCI repository
\cite{blUR} of varying size. For space
considerations, Table \ref{tc1_full}
presents the results.
Each algorithm was ran for a total number of $T = 1000$
iterations; furthermore, the classifier kept for testing is the one minimizing
the empirical risk throughout the $T$ iterations; in
doing so, we also assessed the early convergence of algorithms. 
We fixed $n = \min\{1000, \mbox{train fold
  size}/2\}$.
Table \ref{tc1_full} displays that \radoboost~compares favourably to
\adaboostSS, and furthermore it tends to be all the better as $m$ and
$d$ increase. On some domains like Hardware and Twitter, the difference is impressive and clearly in
favor of \radoboost. As discussed for Lemma \ref{lem_wla}, we could interpret these comparatively very poor
performances of \adaboostSS~as the consequence of outlier features
that can trick \adaboostSS~in picking the wrong sign in the leveraging
coefficient $\alpha_t$ for a large
number of iterations if we use real-valued classifiers (see column
$100\sigma$ in Table \ref{tc1_full}). This drawback can be easily corrected
(Cf Appendix, Subsection
\ref{exp_tc1}) by enforcing
minimal $|r_t|$ values. This significantly improves \adaboostSS~on
Hardware and Twitter. The improvements observed on \radoboost~are even more
favorable.

\section{Rados and differential privacy}\label{sradp}

\begin{figure}[t]
\begin{center}
\includegraphics[width=0.4\columnwidth]{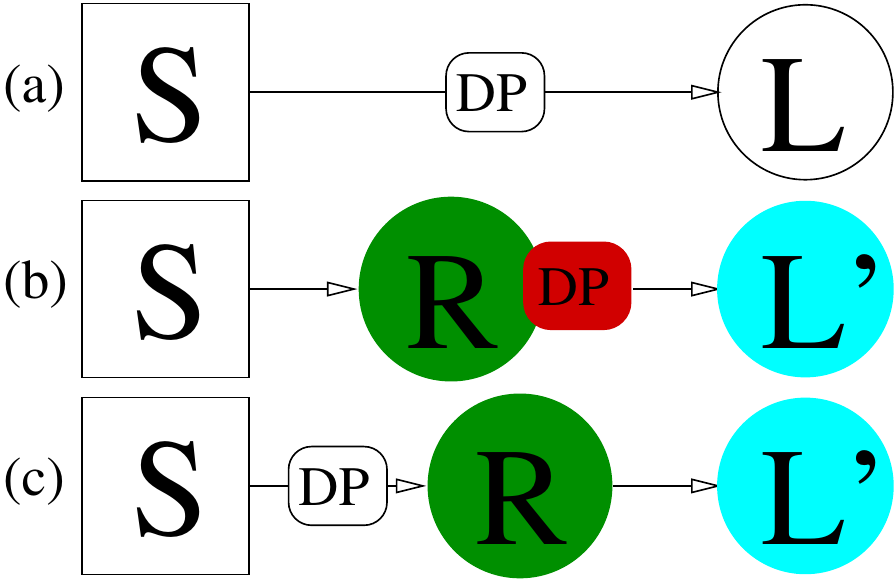}
\caption{Summary of the DP-related contributions of Section
  \ref{sradp} (in color). (a) : usual DP mechanism that protects
  examples (S) prior to delivery to learner (L); (b) : mechanism that crafts
  differentially private rados (R) from unprotected examples (\textsection \ref{sfwdp}); (c) : mechanism crafting rados from DP-compliant examples with objective to improve
performances of rado-based learner L'  (\textsection \ref{sbfdp}).}
\label{dppic}
\end{center}
\end{figure}

We
now discuss the delivery of rados to comply with several
DP constraints and their eventual impact on
boosting. We thus adress both levels (i+ii) of
rado delivery in \textsection\ref{srasl}.
Our general model is the standard DP model
\cite{drTA}. 
Intuitively, an algorithm is DP compliant if for any two neighboring datasets, it assigns similar probability to any possible output $O$. In other words, any particular record has only limited influence on the probability of any given output of the algorithm, and therefore the output discloses very little information about any particular record in the input. Formally, a randomized algorithm $\mathcal{A}$ is $(\upepsilon, \updelta)$-differentially-private \cite{dmnsCN} for some $\upepsilon, \updelta >0$ iff:
\begin{eqnarray}
\mathbb{P}_{\mathcal{A}}[O|\mathcal{S}] & \leq &
\exp(\upepsilon)\cdot\mathbb{P}_\mathcal{A}[O|\mathcal{S}'] + \updelta,
\forall \mathcal{S}\approx \mathcal{S}', O,\label{dpreq}
\end{eqnarray}
where the probability is over the coin tosses of $\mathcal{A}$. 
This model is very strong, especially when $\updelta = 0$, and in the
context of ML, maintaining high accuracy in strong DP regimes is
generally a tricky tradeoff \cite{djwPA}.
\begin{algorithm}[t]
\caption{Feature-wise DP-compliant
  rados (\dpfreal)}\label{algodpfeat}
\begin{algorithmic}
\STATE  \textbf{Input} set of examples ${\mathcal{S}}$, sensitive feature
$j_* \in [d]$, number of rados $n$, differential privacy parameter $\upepsilon > 0$;
\STATE  Step 1 : let $\beta \leftarrow 1/(1+\exp(\upepsilon/2)) \in [0,1/2)$;
\STATE  Step 2 : sample $\ve{\sigma}_1, \ve{\sigma}_2, ...,
\ve{\sigma}_n$ i.i.d. (uniform) in $\sbj$;
\STATE \textbf{Return} set of rados
$\{\ve{\rado}_{\ve{\sigma}} : \ve{\sigma} \mbox{ sampled in
  Step 2}\}$;
\end{algorithmic}
\end{algorithm}
Because rados are an intermediate step between training sample
${\mathcal{S}}$ and a rado-based learner, there are two ways to design
rados with respect to the DP framework: crafting DP-compliant rados from unprotected examples,
or crafting rados from DP-compliant examples with the aim to
improve the performance of the rado-based learner (Figure
\ref{sbfdp}). These scenarii can be reduced to the design of
$\Sigma_r$.

\subsection{A feature-wise DP mechanism for rados}\label{sfwdp}

\begin{figure}[t]
\begin{center}
\includegraphics[width=0.9\columnwidth]{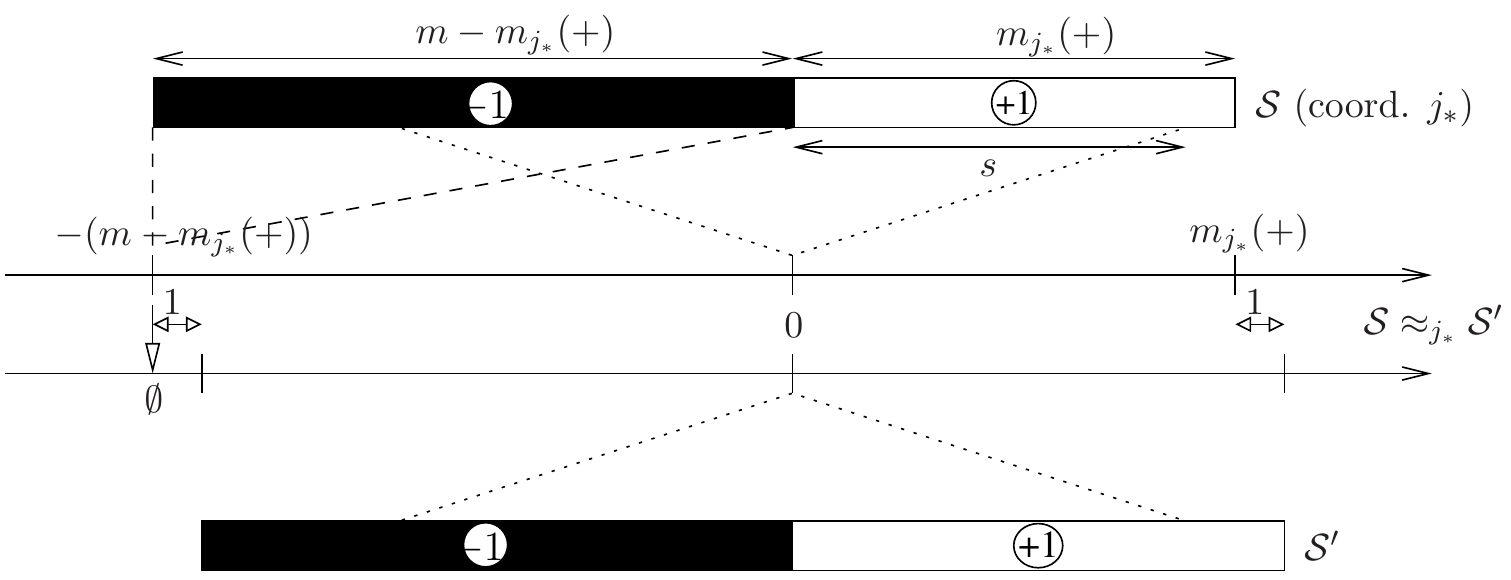}
\caption{How \dpfreal~works: neighbor
  samples ${\mathcal{S}}$ and ${\mathcal{S}}'$ differ by one value for
  feature $j_*$ (\textit{i.e.} one edge coordinate, represented); the
  rado whose support relies only on the ``-1'' in ${\mathcal{S}}$ (dashed
  lines) yields
  infinite ratio $\pr_{{\mathcal{A}}}[O | I]/\pr_{{\mathcal{A}}}[O |
  I']$ in (\ref{dpreq}). This rado would never be sampled by \dpfreal.
  On the other hand, a rado that sums an equal number $s$ of
  ``+1'' and ``-1'' (dotted lines) may yield ratio very close to 1
  (such a rado can
  be sampled by \dpfreal).}
\label{dpexpl}
\vspace{-0.5cm}
\end{center}
\end{figure}

In this Subsection, we consider a relaxation of differential-privacy,
namely \emph{feature-wise} differential privacy, where the
differential privacy requirement applies to $j_*$-\textit{neighboring
  datasets}: we say that two samples $\mathcal{S}, \mathcal{S}'$ are
\textit{$j_*$-neighbors}, noted $\mathcal{S}\approx_{j_*}
\mathcal{S}'$, if they are the same except for the value of the
$j_*^{th} \in [d]$ observation feature of some example. We further assume that
the feature is boolean. For example, we may have
a medical database containing a column 
representing the HIV status of a doctor's patients (1 row = a patient), and we do not wish
that changing a single patient HIV status significantly changes
the density of that feature's values in rados. This setting would also
be very useful in genetic applications to hide in rados gene disorders that
affect one or few genes. Feature-wise DP is analogous to the concept of $\alpha$-label privacy
\cite{chSC}, where differential privacy is guaranteed with respect to
the label. 
Algorithm ${{\mathcal{A}}}$
in ineq. (\ref{dpreq}) is given in Algorithm \ref{algodpfeat}. It relies on the following
subset $\Sigma_r \defeq \sbj \subseteq \Sigma_m$:
\begin{eqnarray}
\sbj
 & \defeq & \left\{\ve{\sigma}  \in  \Sigma_m  :  
\rado_{\ve{\sigma} j_*}  \in  \left[|\{i :  y_i x_{ij_*}
 =  +1\}| -
 \frac{m}{2}  \pm  \Delta_\beta \right]\right\} \label{defSrm}\:\:,
\end{eqnarray}
with $\Delta_\beta \defeq (m/2) - \beta(m+1)$.
The key feature of this mechanism is that it does not alter the examples
in the sense that DP-compliant rados belong to the set of
cardinal $2^m$ that can
be generated from ${\mathcal{S}}$. Usual data-centered DP mechanisms
would rather alter data, \textit{e.g.} via
noise injection \cite{gBP}. Algorithm \ref{algodpfeat} exploits the
fact that
it is the tails of feature $j_*$ that leak sensitive
information about the feature in rados (see Figure \ref{dpexpl}).  The
following Theorem is stated so as we can pick small $\updelta$, typically
 $\updelta \ll 1/m$. Other variants are possible that bring different
 tradeoffs between $\upepsilon$ and $\delta$.

\begin{theorem}\label{thm_dpfreal}
Assume $\upepsilon$ is chosen so that $\upepsilon = o(1)$ but
$\upepsilon = \Omega(1/m)$. In this case, \dpfreal~maintains $(n
\cdot \upepsilon, n \cdot \updelta)$-differential privacy on feature
$j_*$ for some $\updelta > 0$ such that $\upepsilon \cdot \updelta = O(m^{-5/2})$.
\end{theorem}
(Proof in the Appendix, Subsection
\ref{proof_thm_dpfreal}) We have implemented Step 2 in Algorithm \dpfreal~in the
simplest way, using a simple Rademacher rejection sampling
where each $\ve{\sigma}_j$ is picked i.i.d. as $\ve{\sigma}_j \sim
\Sigma_m$ until $\ve{\sigma}_j \in
\sbj$. The
following Theorem shows its algorithmic efficiency.
\begin{theorem}\label{thm_rrs}
For any $\upeta > 0$, let $n^*_\upeta \defeq \upeta
(1-\exp(2\beta-1)) / (4\beta)$,
and let $n_R$ denote the total number of rados sampled in $\Sigma_m$
until $n$ rados are found in $\sbj$. Then
for any  $\upeta > 0$, there is probability $\geq 1 - \upeta$ that
\begin{eqnarray*}
n_R & \leq & n \cdot \left\{
\begin{array}{ccl}
 1 & \mbox{if} & n \leq n^*_\upeta\\
 \left\lceil \frac{1}{m D_{BE}(1-\beta\| 1/2)} \log \frac{n}{n^*_\upeta} \right\rceil &
\multicolumn{2}{l}{\mbox{ otherwise}}
\end{array}
\right. \:\:,\label{boundTRrs}
\end{eqnarray*}
where $D_{BE}$ is the bit-entropy divergence: $D_{BE}(p\|q) = p
\log(p/q) + (1-p) \log((1-p)/(1-q))$, for $p, q \in (0,1)$.
\end{theorem}
(Proof in the Appendix, Subsection
\ref{proof_thm_rrs}) Remark that replacing $\Sigma_m$ by $\Sigma_r = \sbj$
would not necessarily impair the boosting convergence of \radoboost~trained
from rados samples from \dpfreal~(Lemma \ref{lem_radoboost}). The only
systematic change would be in ineq. (\ref{Qbound}) where we would have
to integrate the structural penalty $Q$ from Theorem
\ref{thm_concentration} to further upperbound $\logloss\left({\mathcal{S}}, \ve{\theta}_T\right)$. In this case, the upperbound in (\ref{bsupQ})
reveals that at least when the mean operator in $\sbj$ has small
norm --- which may be the case even when some
examples in ${\mathcal{S}}$ have large norm --- and the gradient
penalty is small, then $Q$ may be small as well.

\begin{table}[t]
\begin{center}
\begin{tabular}{c||c|c}\hline\hline
\hspace{-0.2cm}  Section \ref{sfwdp} \hspace{-0.3cm} & \multicolumn{2}{c}{\hspace{-0.1cm}  Section \ref{sbfdp}}\\
\hspace{-0.1cm} & \hspace{-0.1cm} \radoboost~vs \adaboostSS
\hspace{-0.1cm} & \hspace{-0.1cm} \radoboost:~\textsection \ref{sbfdp} vs \textsection \ref{sfwdp}\\ \hline
\hspace{-0.2cm} \includegraphics[width=0.31\columnwidth]{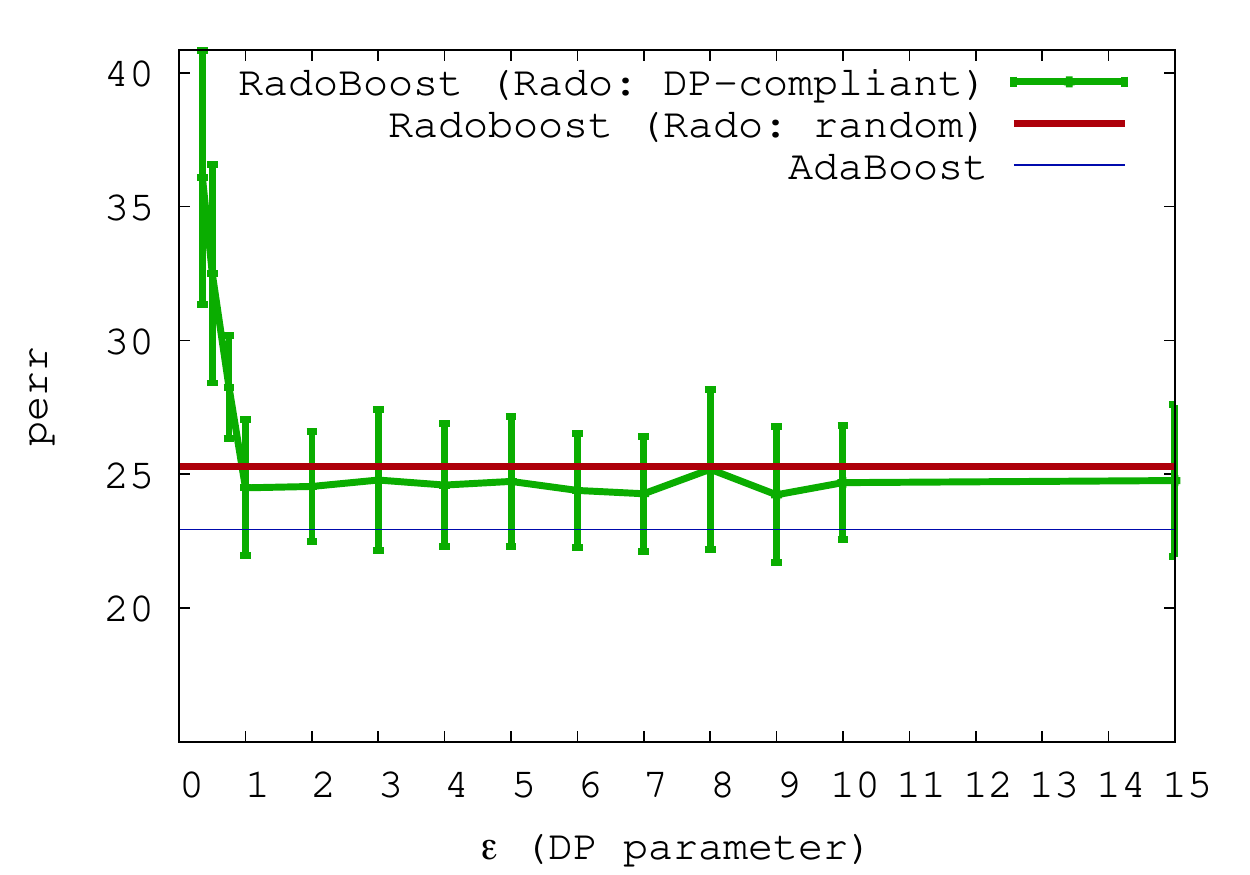}
\hspace{-0.3cm} &
\hspace{-0.1cm} \includegraphics[width=0.31\columnwidth]{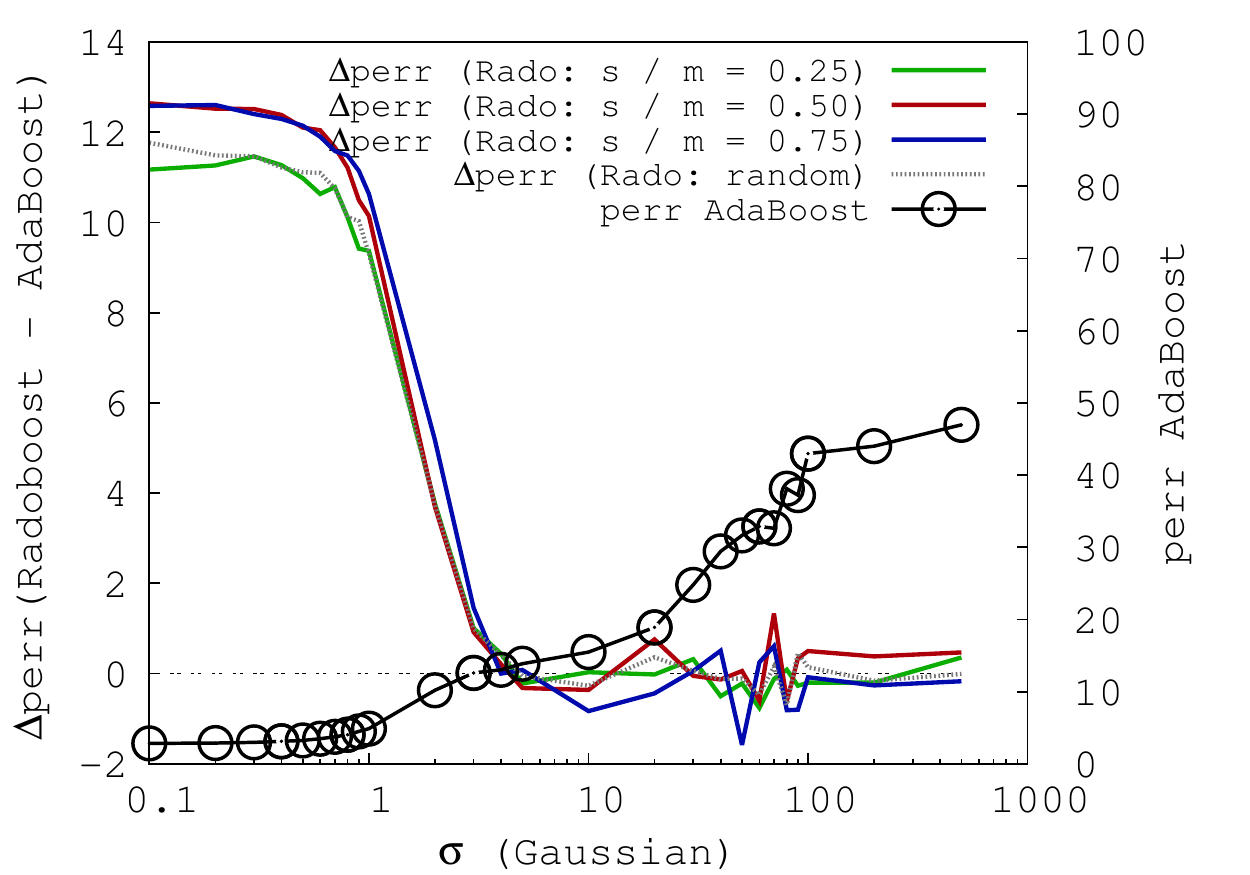} &
\hspace{-0.1cm} \includegraphics[width=0.31\columnwidth]{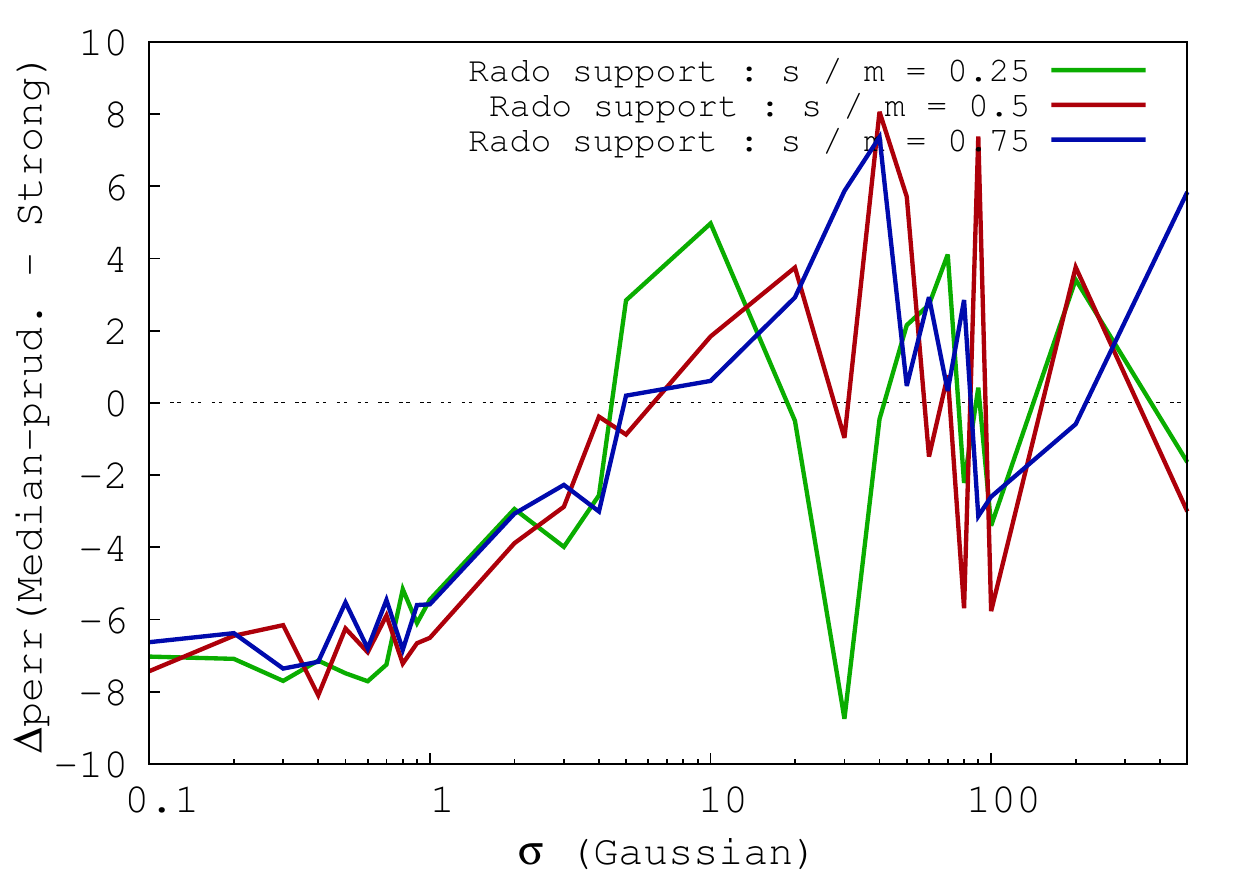}\\
\hspace{-0.2cm}  Abalone \hspace{-0.3cm} &  \hspace{-0.3cm}  Banknote
\hspace{-0.3cm}  & \hspace{-0.3cm} Transfusion\\ \hline
\hspace{-0.2cm} \includegraphics[width=0.31\columnwidth]{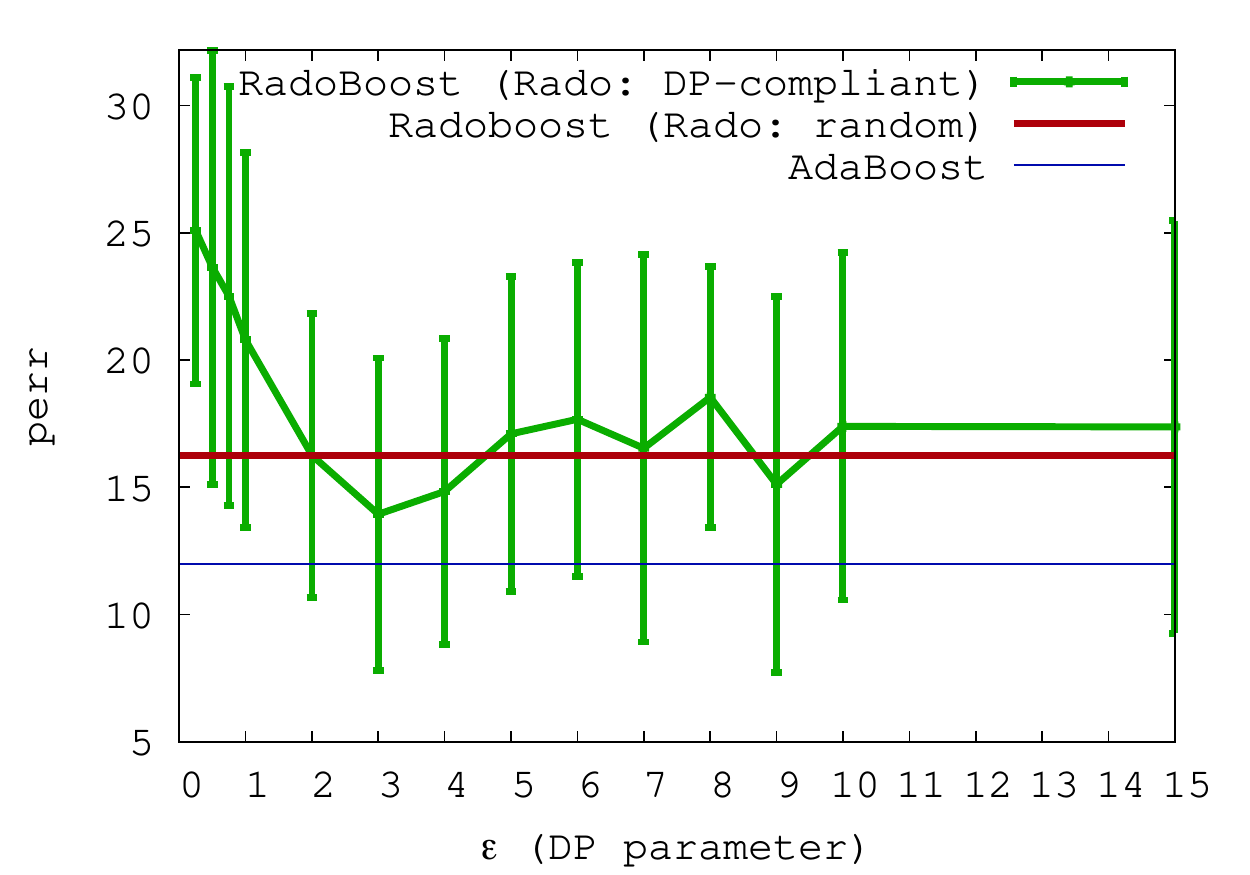}
\hspace{-0.3cm} &
\hspace{-0.1cm} \includegraphics[width=0.31\columnwidth]{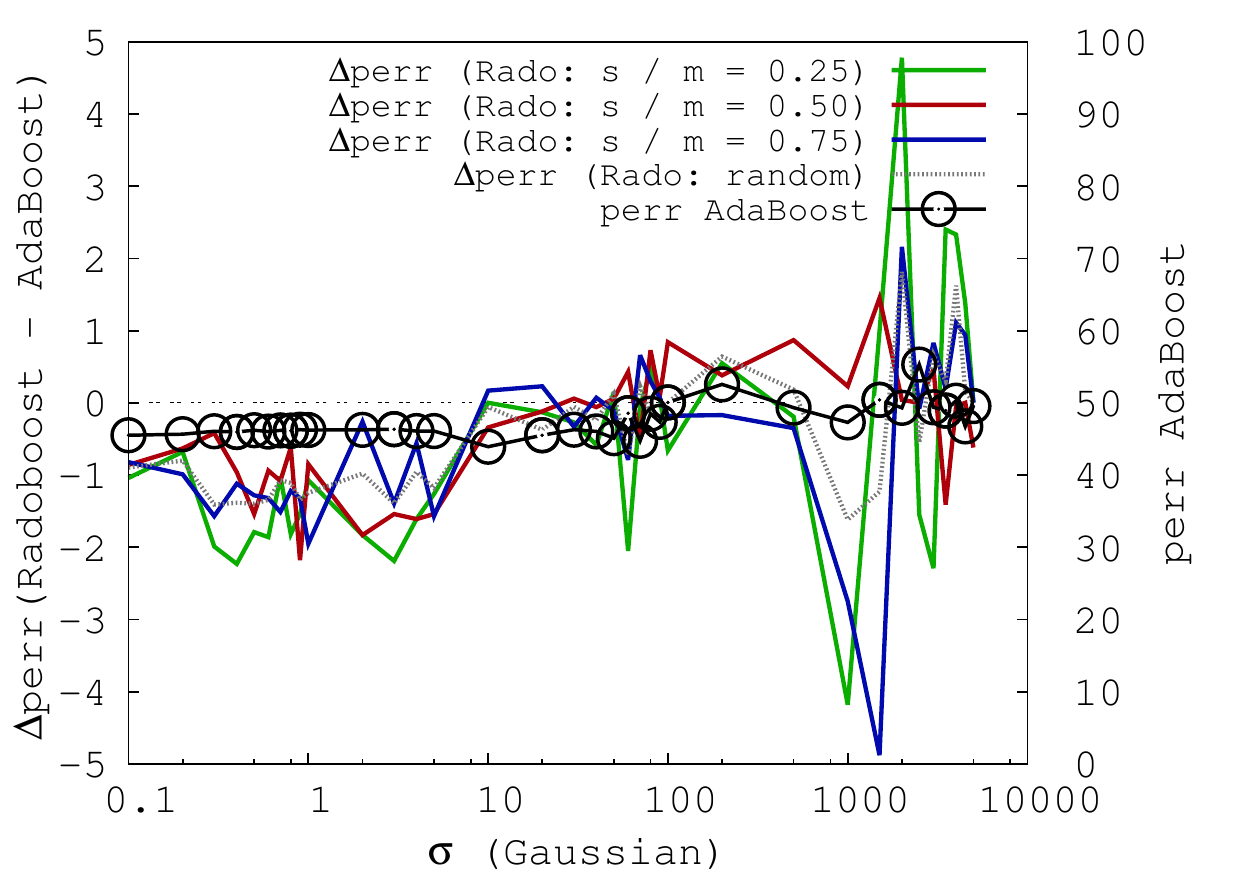}
\hspace{-0.1cm} & \hspace{-0.1cm} \includegraphics[width=0.31\columnwidth]{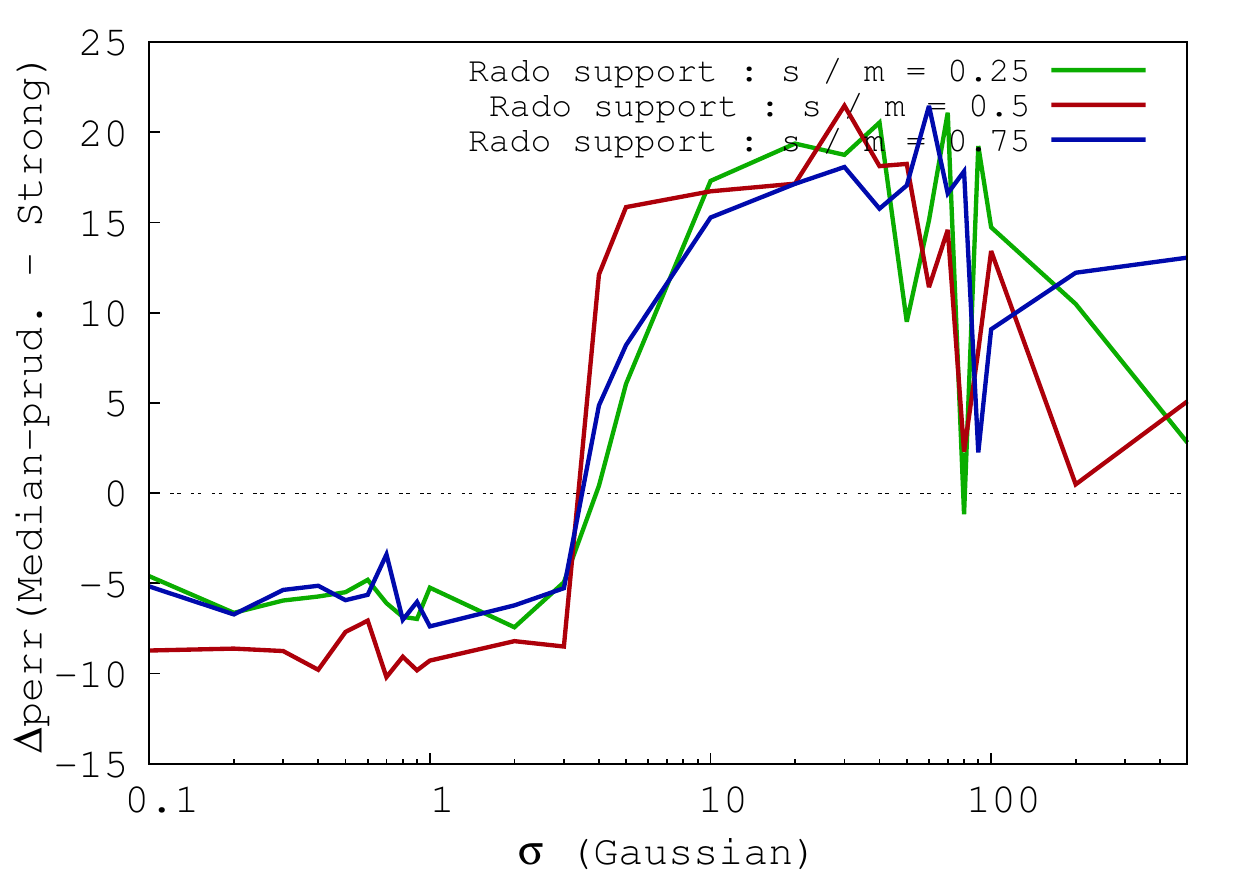}\\
\hspace{-0.2cm}  Ionosphere \hspace{-0.3cm} & \hspace{-0.3cm} Eeg
\hspace{-0.3cm}  & \hspace{-0.3cm} Magic \\ 
\hline\hline
\end{tabular}
\caption{Left table: \radoboost~on feature-wise DP compliant rados (Subsection
  \ref{sfwdp}, showing standard deviations) vs \radoboost~on plain random
  rados baseline and \adaboostSS~baseline (trained with complete
  fold). Center: test error of \radoboost~\textit{minus} \adaboostSS's~(also showing
  \adaboostSS~error on right axis, dotted line), for rados with fixed
  support $s$ ($=m_*$, in green, red, blue) and plain random rados (dotted
  grey). Right: test error of \radoboost~using fixed support $s$ rados and
  a prudential learner, \textit{minus} \radoboost~using plain random rados and
  ``strong'' learner of Section \ref{exp_boost_rado} (See Table \ref{t-s52_1} through Table \ref{t-s52_6}).
\label{t-edpr}}
\end{center}
\end{table}

We end up with several important remarks, whose formal
statements and proofs are left out due to space constraints. First, the tail truncation design
exploited in \dpfreal~can
be fairly simply generalized in two directions, to handle (a) real-valued
features, and/or (b) several sensitive features instead of
one. Second, we can do DP-compliant design of rado delivery beyond feature-wise
privacy, \textit{e.g.} to protect ``rado-wide'' quantities like norms.

\subsection{Boosting from DP-compliant examples via rados}\label{sbfdp}

We now show how to craft rados from DP-compliant examples so as to
approximately keep the convergence rates of \radoboost. More
precisely, since edge vectors are sufficient to learn
(eq. \ref{deflogloss}), we assume that edge vectors are DP-compliant
(neighbor samples, $\mathcal{S}\approx
\mathcal{S}'$, would differ on one edge vector).
A gold standard to protect data in the DP framework
is to convolute data with noise. One popular mechanism is
the Gaussian mechanism \cite{drTA,hpTN}, which convolutes data
with independent Gaussian random variables ${\mathcal{N}}(\ve{0},
\varsigma^2 \mathrm{I})$, whose standard deviation $\varsigma$ depends
on the DP requirement ($\upepsilon, \updelta$). Strong DP regimes are
tricky to handle for learning
algorithms. For
example, the approximation factor $\rho$ of the singular vectors under DP
noise of the noisy power method roughly behaves as
$\rho = \Omega(\varsigma / \Delta)$ \cite{hpTN} (Corollary 1.1) where
$\Delta = O(d)$ is a difference
between two singular values. When $\varsigma$ is small, this is a very good
bound. When the DP requirement blows up, the bound 
remains relevant \textit{if} $d$ increases, which may be hard to
achieve in practice --- it is easier in general to increase $m$ than $d$, which requires to compute
new features for past examples. 

We consider ineq. (\ref{dpreq}) with neighbors $I$ and $I'$ being two sets of
$m$ edge vectors differing by one
edge vector, and $O$ is a noisified set of $m$ edge vectors generated through the
Gaussian mechanism \cite{drTA} (Appendix A). We show the following
non-trivial result: provided we design another particular $\Sigma_r$, the convergence rate of
\radoboost, \textit{as measured over non-noisy rados}, essentially survives noise
injection in the edge vectors through the Gaussian mechanism, even
under strong noise regimes, as long as $m$ is large enough. The intuition is
  straightforward: we build rados summing a large number of edge
  vectors only (this is the design of $\Sigma_r$), so that the i.i.d. noise component gets sufficiently
concentrated for the algorithm to be able to learn almost as fast as
in the noise-free setting. We emphasize the non-trivial fact that convergence rate is
measured over the non-noisy rados, which of course \radoboost~does
\textit{not} see. The result is of independent interest in the
boosting framework, since it makes use of a particular weak learner ($\weak$),
which we call \textit{prudential}, which picks features with $|r_t|$ (\ref{defMu}) upperbounded.

We start by renormalizing
coefficients $\alpha_t$ (eq. (\ref{defalpha})) in \radoboost~by a
parameter $\kappa \geq 1$ given as input, so that we now have
$\alpha_{t} \leftarrow (1/(2 \kappa \rado_{*\iota(t)}))
\log ((1 + r_t)/(1 -
  r_t))$ in Step 2.2.
It is not hard to check that the convergence rate of \radoboost~now becomes, prior to applying the (\textbf{WLA})
\begin{eqnarray}
\loglossrado({\mathcal{S}}, \ve{\theta}_T,
  \mathcal{U}) & \leq & \log(2) - \frac{1}{2\kappa m} \sum_t
  r_t^2\:\:.\label{Qbound2}
\end{eqnarray}
We say that $\weak$ is $\uplambda_p$-\textit{prudential} for $\uplambda_p > 0$ iff it selects at each iteration a feature
such that $|r_t| \leq \uplambda_p$. Edges vectors have been
DP-protected as $y_i (\ve{x}_i + \ve{x}_i^r)$, with $\ve{x}_i^r \sim {\mathcal{N}}(\ve{0},
\varsigma^2 \mathrm{I})$ (for $i\in [m]$). Let $m_{\ve{\sigma}}
\defeq |\{ i : \sigma_{i} = y_i\}|$ denote the \textit{support}
of a rado, and ($m_*>0$ fixed):
\begin{eqnarray}
 \Sigma_r = \sbm
 & \defeq & \left\{\ve{\sigma} \in \Sigma_m : 
m_{\ve{\sigma}} = m_*\right\}\:\:.\label{defSrmm}
\end{eqnarray}
\begin{theorem}\label{thm_random_gau}
$\forall {\mathcal{U}} \subseteq \Sigma_r, \forall \uptau > 0$, if
$\sqrt{m_*} = \Omega \left(\varsigma \ln (1/\uptau)\right)$, then
$\exists \uplambda_p>0$ such that
\radoboost~having access to a $\uplambda_p$-prudential
weak learner returns after $T$ iteration a classifier $\ve{\theta}_T$  which meets with probability $\geq 1
- \uptau$:
\begin{eqnarray}
\loglossrado({\mathcal{S}}, \ve{\theta}_T,
  \mathcal{U}) & \leq & \log(2) - \frac{1}{4\kappa m} \sum_t
  r_t^2\:\:. \label{llleqlast}
\end{eqnarray}
\end{theorem}
The proof, in the Appendix (Subsection
\ref{proof_thm_random_gau}), details parameters and
dependencies hidden in the statement. The use of a prudential weak learner is rather intuitive
in a noisy setting since $\alpha_t$ blows
up when $|r_t|$ is close to 1. Theorem \ref{thm_random_gau} essentially yield
that a sufficiently large support for rados is enough to keep with high
probability the convergence rate of \radoboost~within noise-free regime. Of course, the weak learner is prudential,
which implies bounded $|r_t| < 1$, and furthermore the leveraging
coefficients $\alpha_t$ are normalized, which implies smaller
margins. Still, Theorem \ref{thm_random_gau} is a good theoretical argument to
rely on rados when learning from DP-compliant edge vectors.

\section{Experiments on differential privacy}\label{exp_dp_rado}

Table \ref{t-edpr} presents a subset of the experiments carried out
with \radoboost~and \adaboostSS~in the contexts of Subsections \ref{sfwdp}
and \ref{sbfdp} (see Section
\ref{app_exp_expes} for all additional experiments). Unless otherwise stated, experimental settings (cross
validation, number of
rados for learning, etc.) are the same as in Section
\ref{exp_boost_rado}. 

In a first set of experiments, we have assessed the impact on learning of the
feature-wise DP mechanism: on each tested domain, we have
selected at random a binary feature, and then used Algorithm
\dpfreal~to protect the feature for different values of DP parameter
$\upepsilon$, in a range that covers usual DP experiments
\cite{hghknprDP} (Table 1). The main conclusion that can be drawn from
the experiments is that learning from DP-compliant rados can compete with learning from random
rados, and even learning from examples (\adaboostSS), even for rather small $\upepsilon$. 

We then have assessed the impact on learning
of examples that have been protected using the Gaussian mechanism
\cite{drTA}, with or without rados, with or without a prudential weak
learner for boosting, and with or without using a fixed support for
rado computation. The Appendix provides extensive results for all domains but
the largest ones (Twitter, SuSy, Higgs). In the central column (and
Tables \ref{t-s52_1} through \ref{t-s52_4} in the Appendix),
computing the differences between \radoboost's error and \adaboostSS's
reveals that, on
domains where it is beaten by \adaboostSS~when there is no noise, \radoboost~almost always
rapidly become competitive with \adaboostSS~as noise increases. Hence, \radoboost~is a good
contender from the boosting family to learn from differentially
private (or noisy) data. Second, using a prudential weak learner which
picks the median feature (instead of the more efficient weak learner
that picks the best as in Section
\ref{exp_boost_rado}) can have \radoboost~with fixed support rados
compete or beat \radoboost~with plain random rados, at least for small
noise levels (see Transfusion and Magic in the right column of Table
\ref{t-edpr}). Replacing the median-prudential weak learner by a
strong learner can actually degrade \radoboost's results (see the Appendix,
Tables \ref{t-s52_5} and \ref{t-s52_6}). These two observations
advocate in favor of the theory developed in Subsection \ref{sbfdp}.
Finally, using rados with
fixed support instead of plain random rados (Section \ref{exp_boost_rado}) can significantly improve
the performances of \radoboost~(see the Appendix,
Tables \ref{t-s52_5} and \ref{t-s52_6}).

\section{From rados to examples: hardness results}\label{sec_hardness}

The problem we address here
is how we can recover examples from rados, and when we
\textit{cannot} recover examples from rados. This last setting is
particularly useful from the privacy standpoint, as this may save us costly obfuscation techniques
that impede ML tasks \cite{bptgML}.

\subsection{Algebraic and geometric hardness}

For any $m \in {\mathbb{N}}_*$,
we define matrix $\matrice{G}_m \in \{0,1\}^{m\times 2^m}$ as:
\begin{eqnarray}
\matrice{G}_m & \defeq & 
\left[
\begin{array}{cc}
\ve{0}^\top_{2^{m-1}} & \ve{1}^\top_{2^{m-1}}\\
\matrice{G}_{m-1} & \matrice{G}_{m-1}
\end{array}
\right] 
\end{eqnarray}
if $m>1$, and $\matrice{G}_1 \defeq [0\:\: 1]$ otherwise ($\ve{z}_d$
denotes a vector in ${\mathbb{R}}^d$). Each column of
$\matrice{G}_m$  is the binary indicator vector for the
edge vectors considered in a rado. Hereafter, we let $\matrice{E} \in
{\mathbb{R}}^{d\times m}$ the matrix of columnwise edge vectors from
${\mathcal{S}}$, $\matrice{$\Pi$} \in
{\mathbb{R}}^{d\times n}$ the columnwise rado matrix and $\matrice{U} \in \{0,1\}^{2^m \times n}$ in which each column gives
the index of a rado computed in ${\mathcal{S}}^r$. By construction, we
have:
\begin{eqnarray}
\matrice{$\Pi$} & = & \matrice{E} \matrice{G}_m \matrice{U} \:\:,\label{linkpim}
\end{eqnarray}
and so we have the following elementary results for the (non)
reconstruction of $\matrice{E}$ (proof omitted).
\begin{lemma}\label{lem_reco1}
(a) when recoverable, edge-vectors satisfy: $\matrice{E} = \matrice{$\Pi$}
\matrice{U}^\top \matrice{G}_m^\top (\matrice{G}_m \matrice{U}
\matrice{U}^\top \matrice{G}_m^\top)^{-1}$; (b) when $\matrice{U}$, $\matrice{$\Pi$}$, $m$ are known but $n<m$,
  there is not a single solution to eq. (\ref{linkpim}) in general.
\end{lemma}
Lemma \ref{lem_reco1} states that even when
$\matrice{U}$, $\matrice{$\Pi$}$ and $m$ are known,
elementary constraints on rados can make the recovery of edge vectors
hard --- notice that such constraints are met in our
experiments with \radoboost~in Sections \ref{exp_boost_rado} and \ref{exp_dp_rado}. 

But this represents a lot of \textit{unnecessary} knowledge to learn
from rados: \radoboost~just needs
$\matrice{$\Pi$}$ to learn. We now explore the guarantees
that providing this sole information brings in terms of (not)
reconstructing $\matrice{E}$.
$\forall \matrice{M} \in {\mathbb{R}}^{a \times b}$, we let
${\mathcal{C}}(\matrice{M})$ denote the set of column vectors, and for
any ${\mathcal{C}} \subseteq {\mathbb{R}}^d$, we let ${\mathcal{C}}
\oplus \epsilon \defeq \cup_{\ve{z} \in {\mathcal{C}}}
{\mathcal{B}}(\ve{z} , \epsilon)$. We define the Hausdorff distance, $D_{\mathrm{H}}({\matrice{E}}, {\matrice{E}}')$, between ${\matrice{E}}$ and
${\matrice{E}}'$:
\begin{eqnarray*}
\lefteqn{D_{\mathrm{H}}({\matrice{E}}, {\matrice{E}}')}\nonumber\\
 & \defeq & \inf\{\epsilon:
{\mathcal{C}}({\matrice{E}}) \subseteq {\mathcal{C}}({\matrice{E}}') \oplus \epsilon \wedge {\mathcal{C}}({\matrice{E}}') \subseteq {\mathcal{C}}({\matrice{E}}) \oplus \epsilon\}\:\:.
\end{eqnarray*}
The following Lemma shows that if the only information known is
$\matrice{$\Pi$}$, then there exist samples that bring the same set
of rados ${\mathcal{C}}(\matrice{$\Pi$})$ as the unknown $\matrice{E}$
 \textit{but} who are at
distance proportional to the ``width'' of the domain at hand.
\begin{lemma}\label{lem_algebraic}
For any $\matrice{$\Pi$} \in {\mathbb{R}}^{d\times n}$, suppose eq. (\ref{linkpim}) holds, for some
unknowns $m > 0$, $\matrice{E} \in {\mathbb{R}}^{d\times m}$, $\matrice{U}\in
\{0,1\}^{2^m \times n}$. Suppose
${\mathcal{C}}(\matrice{E}) \subset {\mathcal{B}}(\ve{0} , R)$ for
some $R>0$.
Then there exists
$\matrice{E}' \in {\mathbb{R}}^{d\times {(m+1)}}$, $\matrice{U}' \in
\{0,1\}^{2^{m+1} \times n}$ such that 
\begin{eqnarray}
{\mathcal{C}}(\matrice{E}')
  \subset  {\mathcal{B}}(\ve{0} , R) & \mbox{ and } & 
\matrice{$\Pi$} = \matrice{E}' \matrice{G}_{m+1} \matrice{U}'
\:\:,
\end{eqnarray}
 but
\begin{eqnarray}
D_{\mathrm{H}}(\matrice{E}, \matrice{E}') & = & \Omega\left(
  \frac{R \log d}{\sqrt{d} \log m} \right)\:\:
\end{eqnarray}
if $m\geq 2^d$, and $D_{\mathrm{H}}(\matrice{E}, \matrice{E}') =
\Omega(R/\sqrt{d})$ otherwise.
\end{lemma}
(Proof in the Appendix, Subsection
\ref{proof_lem_algebraic})
Hence, without any more knowledge, leaks, approximations or assumptions on the domain at hand, the recovery of
$\matrice{E}$ pays in the worst case a price proportional to the
radius of the smallest enclosing ${\mathcal{B}}(\ve{0},.)$ ball for
the unknown set of examples. We emphasize that this inapproximability result does not rely on the computational power
at hand.

\subsection{Computational hardness}

In this Subsection, we investigate two important problems in the
recovery of examples. The first problem addresses whether we can \textit{approximately}
recover \textit{sparse} examples from a given set of rados, that is, roughly,
solve (\ref{linkpim}) with a sparsity constraint on examples. The
first Lemma we give is related to the hardness of solving underdetermined
linear systems for sparse solutions \cite{dtSN}. The sparsity
constraint can be embedded in the compressed sensing framework \cite{doCS} to
yield finer hardness \textit{and} approximability results, which is
beyond the scope of our paper. We define
problem ``Sparse-Approximation'' as:
\begin{itemize}
\item [] {\hspace{-0.7cm}(\textbf{Instance})} : set of rados ${\mathcal{S}}^r = \{\ve{\rado}_{1}, \ve{\rado}_{2}, ...,
\ve{\rado}_{n}\}$, $m \in
{\mathbb{N}}_*$, $r, \ell \in {\mathbb{R}}_+$, $\|.\|_p$, $L_p$-norm
for $p \in {\mathbb{R}}_+$;
\item [] {\hspace{-0.7cm}(\textbf{Question})} : Does there exist set
${\mathcal{S}} \defeq \{(\ve{x}_i,y_i), i \in [m]\}$ and set ${\mathcal{U}} \defeq
\{\ve{\sigma}_1, \ve{\sigma}_2, ..., \ve{\sigma}_n \} \in\{-1,1\}^m$ such
that:
\begin{eqnarray*}
\|\ve{x}_i\|_p & \leq & \ell\:\:, \forall i\in [m]\:\:,
\:\:(\mbox{Sparse examples}) \\
\|\ve{\rado}_{j} - \ve{\rado}_{\ve{\sigma}_j}\|_p & \leq &
r\:\:, \forall j \in [n]\:\:.\:\:(\mbox{Rado approximation})
\end{eqnarray*}
\end{itemize}
\begin{lemma}\label{lem_comp1}
Sparse-Approximation is NP-Hard.
\end{lemma}
(Proof in the Appendix, Subsection
\ref{proof_lem_comp1}) In the context of rados, the second problem we
address has very large privacy applications.
Suppose entity \textcircled{{\scriptsize A}} has a huge
database of people (\textit{e.g.} clients), and obtains a set of
rados emitted by another entity \textcircled{{\scriptsize B}}. 
An important question that \textcircled{{\scriptsize A}} may ask is
whether the rados observed \textit{can} be \textit{approximately} constructed by its
database, for example to figure out which of its clients are also
its competitors'. We define this as problem ``Probe-Sample-Subsumption'':
\begin{itemize}
\item [] {\hspace{-0.7cm}(\textbf{Instance})} : set of examples
  ${\mathcal{S}}$, set of rados ${\mathcal{S}}^r  = \{\ve{\rado}_{1}, \ve{\rado}_{2}, ...,
\ve{\rado}_{n}\}$, $m \in
{\mathbb{N}}_*$, $p, r \in {\mathbb{R}}_+$.
\item [] {\hspace{-0.7cm}(\textbf{Question})} : Does there exist
  ${\mathcal{S}}' \defeq \{(\ve{x}_i,y_i), i \in
  [m]\} \subseteq {\mathcal{S}}$ and set ${\mathcal{U}} \defeq \{\ve{\sigma}_1, \ve{\sigma}_2, ..., \ve{\sigma}_n\}\in\{-1,1\}^m$ such
that:
\begin{eqnarray*}
\|\ve{\rado}_{j} - \ve{\rado}_{\ve{\sigma}_j}\|_p & \leq &
r\:\:, \forall j \in [n]\:\:.\:\:(\mbox{Rado approximation})
\end{eqnarray*}
\end{itemize}
\begin{lemma}\label{lem_comp2}
Probe-Sample-Subsumption is NP-Hard.
\end{lemma}
(Proof in the Appendix, Subsection
\ref{proof_lem_comp2}) This worst-case result calls for interesting
domain-specific qualifications, such as in genetics where the
privacy of raw
data, \textit{i.e.} individual genomes, can be compromised by genome-wise statistics
\cite{hsrdtmpsncRI,nslTB}.


\section{Conclusion}

We have
introduced novel quantities that are sufficient for efficient
learning, Rademacher observations. The fact that a subset of these can
replace traditional examples for efficient learning opens interesting problems on
how to craft these subsets to cope with additional constraints. We
have illustrated these constraints in the field of efficient 
learning from privacy-compliant data, from various
standpoints that include differential privacy as well as algebaric,
geometric and computational considerations. In that last case, results rely
on NP-Hardness, and thus go beyond the ``hardness'' of factoring integers
on which rely some popular cryptographic techniques \cite{bptgML}. 
Finally, rados are cryptography-compliant: homomorphic encryption schemes
can be used to compute rados in the encrypted domain from encrypted
edge vectors or examples --- rado computation can thus be easily
distributed in secure multiparty computation applications.

\section{Acknowledgments}

The authors are indebted to Tiberio Ca\'etano for early
discussions that brought the idea of Rademacher observations and their
use in privacy related applications. Thanks are also due to Stephen
Hardy and Hugh Durrant-Whyte for many stimulating
discussions and feedback on the subject. NICTA is funded by the Australian Government through the Department of Communications and the Australian Research Council through the ICT Center of Excellence Program.

\bibliographystyle{plain}
\bibliography{bibgen}

\section{Appendix --- Proofs}\label{app_proof_proofs}

To simplify the proofs, we define the following quantity:
\begin{eqnarray}
\ve{\tilde{\rado}}_{\ve{\sigma}} & \defeq & \sum_{i} \sigma_i
\ve{x}_i\:\:, \forall \ve{\sigma} \in \Sigma_m \:\:.
\end{eqnarray}
so that each rado can be defined as: $\ve{\rado}_{\ve{\sigma}} = (1/2)
\cdot (\ve{\tilde{\rado}}_{\ve{\sigma}} +
\ve{\tilde{\rado}}_{\ve{y}})$. We recall that $\ve{y}$ is the label vector.

\subsection{Proof of Lemma \ref{lem_equivlogexp}}\label{proof_lem_equivlogexp}

We have
\begin{eqnarray}
\logloss\left({\mathcal{S}}, \ve{\theta}\right) & \defeq & \frac{1}{m} \sum_{i}
\log\left(1+\exp\left(-y_i \ve{\theta}^\top \ve{x}_i\right)\right) \nonumber\\
 & = & \frac{1}{m} \sum_{i}
\log\left(\sum_{y \in \{-1,1\}} \exp\left(\frac{1}{2}\cdot y \ve{\theta}^\top
    \ve{x}_i\right)\right) - \frac{1}{m}
\cdot  \frac{1}{2}\cdot \ve{\theta}^\top
\ve{\tilde{\rado}}_{\ve{y}}\label{decomp}\\
 & = & \frac{1}{m}\log \sum_{\ve{\sigma} \in \Sigma_m} \exp\left( \frac{1}{2}
  \cdot \ve{\theta}^\top \ve{\tilde{\rado}}_{\ve{\sigma}}\right) -\frac{1}{m}
\cdot  \frac{1}{2}\cdot \ve{\theta}^\top
\ve{\tilde{\rado}}_{\ve{y}}\nonumber\\
 & = & \frac{1}{m}\log \sum_{\ve{\sigma} \in \Sigma_m} \exp\left( \frac{1}{2}
  \cdot \ve{\theta}^\top \ve{\tilde{\rado}}_{\ve{\sigma}}\right) + \frac{1}{m}
\cdot \log \exp \left(-\frac{1}{2}\cdot \ve{\theta}^\top
\ve{\tilde{\rado}}_{\ve{y}}\right) \nonumber\\
 & = & \frac{1}{m}\log \sum_{\ve{\sigma} \in \Sigma_m} \exp\left( \frac{1}{2}
  \cdot \ve{\theta}^\top (\ve{\tilde{\rado}}_{\ve{\sigma}} - \ve{\tilde{\rado}}_{\ve{y}})\right) \nonumber\\
 & = & \frac{1}{m}\log \sum_{\ve{\sigma} \in \Sigma_m} \exp\left( - \frac{1}{2}
  \cdot \ve{\theta}^\top (\ve{\tilde{\rado}}_{\ve{\sigma}} + \ve{\tilde{\rado}}_{\ve{y}})\right) \label{reason1}\\
 & = & \log(2) + \frac{1}{m}\log \frac{1}{2^m} \sum_{\ve{\sigma} \in \Sigma_m} \exp\left( - \frac{1}{2}
  \cdot \ve{\theta}^\top (\ve{\tilde{\rado}}_{\ve{\sigma}} + \ve{\tilde{\rado}}_{\ve{y}})\right) \nonumber\\
  & = & \log(2) + \frac{1}{m}\log \frac{1}{2^m} \sum_{\ve{\sigma} \in
    \Sigma_m} \exp\left( -  \ve{\theta}^\top \ve{\rado}_{\ve{\sigma}}\right) \nonumber\\
 & = & \log(2) + \frac{1}{m}
\log \explossrado({\mathcal{S}}, \ve{\theta}, \Sigma_m) \:\:. \nonumber
\end{eqnarray}
We refer to (\cite{pnrcAN}) (Lemma 1) for the proof of \ref{decomp}.  Eq. (\ref{reason1}) holds because $\Sigma_m$ is closed by negation.

\subsection{Proof of Theorem \ref{thm_concentration}}\label{proof_thm_concentration}

Let us suppose that our set of rados ${\mathcal{U}}$ satisfies:
\begin{eqnarray}
{\mathcal{U}} \subseteq \Sigma_r \subseteq \Sigma_m\:\:,
\end{eqnarray}
where $\Sigma_r$ is a fixed reference subset of $\Sigma_m$. We shall
use the shorthand $\expect_U[f(U)]$ to denote uniform i.i.d. sampling
of $U$ in $\Sigma_r$. Furthermore, we also let for short
\begin{eqnarray}
\ell \defeq \sup_{\ve{\theta} \in \Theta}
\max_{\ve{\rado}_{\ve{\sigma}} \in \Sigma_r} \exp(-\ve{\theta}^\top
\ve{\rado}_{\ve{\sigma}})\:\:. \label{defell}
\end{eqnarray}
The proof relies on basic knowledge of VC theory and the
``symmetrization trick'', which can be found
\textit{e.g.} in (\cite{bblIT}). 
Plugging eq. (\ref{defell}) into the proof of the symmetrization Lemma (Lemma 2 in
(\cite{bblIT})) yields the following symmetrization Lemma for the
exponential rado-loss. Notice that the assumption is the same as in Lemma 2 in
(\cite{bblIT}).
\begin{lemma}\label{symlem}
For any fixed sample ${\mathcal{S}}$, for any $t$ such that $n t^2
\geq 2$, the following holds over the Rademacher sampling of
$\ve{\sigma}$ in $\Sigma_m$:
\begin{eqnarray}
\lefteqn{\pr\left[\sup_{\ve{\theta} \in \Theta} (\expect_{U} \left[ \explossrado({\mathcal{S}},
    \ve{\theta}, U)\right] - \explossrado({\mathcal{S}}, \ve{\theta},
\mathcal{U})) \geq t\right]}\nonumber\\
  & \leq & 2 \ell^2
\cdot \pr\left[\sup_{\ve{\theta} \in \Theta} (\explossrado({\mathcal{S}}, \ve{\theta},
\mathcal{U}) - \explossrado({\mathcal{S}}, \ve{\theta},
\mathcal{U}')) \geq \frac{t}{2}\right]\:\:,\nonumber
\end{eqnarray}
where $\mathcal{U}, \mathcal{U}'$ are two size-$n$ i.i.d. samples.
\end{lemma}
Consider ${\mathcal{U}}, {\mathcal{U}}' \subseteq \Sigma_r$,
each of cardinal $n$ and differing from one assignment only. Then it
follows, for any $\ve{\theta} \in \Theta$ and
from ineq. (\ref{bsup1}):
\begin{eqnarray}
|\explossrado({\mathcal{S}}, \ve{\theta}, \mathcal{U}) -
\explossrado({\mathcal{S}}, \ve{\theta}, \mathcal{U}')|
& \leq & \frac{2\ell}{n}\:\:.\label{bsup1}
\end{eqnarray}
Applying the independent bounded differences inequality (\cite{mdC}), we
get, for any $\ve{\theta} \in \Theta$ and $t>0$:
\begin{eqnarray}
\pr\left[\expect_{U} \left[ \explossrado({\mathcal{S}},
    \ve{\theta}, U)\right] - \explossrado({\mathcal{S}}, \ve{\theta},
\mathcal{U}) \geq \frac{t}{4}\right] & \leq &
\exp\left(-\frac{n t^2}{16 \ell^2}\right)\:\:.\label{md1}
\end{eqnarray}
Letting $\Pi(n)$ denote the growth function for linear separators
computed over rados, we still have the upperbound
\begin{eqnarray}
\Pi(n) & \leq & \left(\frac{en}{d+1}\right)^{d+1}\:\:.\label{gfls}
\end{eqnarray}
We thus get, for any $\ve{\theta} \in \Theta$:
\begin{eqnarray}
\lefteqn{\pr\left[\sup_{\ve{\theta} \in \Theta} (\expect_{U} \left[ \explossrado({\mathcal{S}},
    \ve{\theta}, U)\right] - \explossrado({\mathcal{S}}, \ve{\theta},
\mathcal{U})) \geq t\right]}\nonumber\\
 & \leq & 2 \ell^2
\cdot \pr\left[\sup_{\ve{\theta} \in \Theta} (\explossrado({\mathcal{S}}, \ve{\theta},
\mathcal{U}) - \explossrado({\mathcal{S}}, \ve{\theta},
\mathcal{U}')) \geq \frac{t}{2}\right]\label{eeq1}\\
 & \leq & 2 \Pi(2 n) \ell^2
\cdot \pr\left[\explossrado({\mathcal{S}}, \ve{\theta},
\mathcal{U}) - \explossrado({\mathcal{S}}, \ve{\theta},
\mathcal{U}') \geq \frac{t}{2}\right]\label{eeq2}\\
 & \leq & 4 \Pi(2 n) \ell^2
\cdot \pr\left[\expect_{U} \left[ \explossrado({\mathcal{S}},
    \ve{\theta}, U)\right] - \explossrado({\mathcal{S}}, \ve{\theta},
\mathcal{U}) \geq \frac{t}{4}\right]\label{eeq3}\\
 & \leq & 4 \Pi(2 n) \ell^2
\cdot \exp\left(-\frac{n t^2}{16 \ell^2}\right) \label{eeq4}\\
 & \leq & 4 \left(\frac{2en}{d+1}\right)^{d+1} \ell^2
\cdot \exp\left(-\frac{n t^2}{16 \ell^2}\right)\:\:. \label{eeq5}
\end{eqnarray}
Ineq. (\ref{eeq1}) follows from Lemma \ref{symlem}, ineq. (\ref{eeq2})
follows from standard VC arguments (see \textit{e.g.} (\cite{bblIT}),
Section 4), ineq. (\ref{eeq3}) follows from the observation that event
$a - b \geq u$ implies $(a - c \geq u/2) \vee (b - c \geq u/2)$,
ineq. (\ref{eeq4}) follows from (\ref{md1}), and finally ineq
(\ref{eeq5}) follows from ineq. (\ref{gfls}). Picking
\begin{eqnarray}
t = t_* & \defeq & 16  \ell \cdot \sqrt{\frac{1}{n}\log \ell +
  \frac{d}{n}\log \frac{2en}{d} + \frac{1}{n} \log \frac{1}{\upeta}}
\end{eqnarray}
yields that the right hand-side of ineq. (\ref{eeq5}) is not more than
$\upeta$, for any $\upeta>0$. So with probability $\geq 1- \upeta$,
any classifier $\ve{\theta} \in \Theta$ will enjoy $\expect_{U} \left[ \explossrado({\mathcal{S}},
    \ve{\theta}, U)\right] \leq \explossrado({\mathcal{S}}, \ve{\theta},
\mathcal{U}) + t_*$, and so we shall have:
\begin{eqnarray}
\lefteqn{\loglossrado({\mathcal{S}}, \ve{\theta},
  \mathcal{U})}\nonumber\\
 & \defeq & \log(2) + \frac{1}{m}\cdot
\log \explossrado({\mathcal{S}}, \ve{\theta}, \mathcal{U})\nonumber\\
 & \geq & \log(2) + \frac{1}{m}\cdot
\log \left(\expect_{U} \left[ \explossrado({\mathcal{S}},
    \ve{\theta}, U)\right] - t_*\right)\nonumber\\
& & = \log(2) + \frac{1}{m}\cdot
\log \left(\expect_{{U}} \left[ \explossrado({\mathcal{S}},
    \ve{\theta}, U)\right]\right) \nonumber\\
 & & + \frac{1}{m}\cdot\log\left(1 - 16 \varrho \cdot \sqrt{\frac{1}{n}\log \ell +
  \frac{d}{n}\log \frac{2en}{d} + \frac{1}{n} \log \frac{1}{\upeta}}\right)\label{pp1}\\
& = & \log(2) + \frac{1}{m}\cdot
\log \explossrado({\mathcal{S}},
    \ve{\theta}, \Sigma_m) + \frac{1}{m}\cdot
\log \frac{\explossrado({\mathcal{S}},
    \ve{\theta}, \Sigma_r)}{\explossrado({\mathcal{S}},
    \ve{\theta}, \Sigma_m)} \nonumber\\
 & & + \frac{1}{m}\cdot \log\left(1 - 16 \varrho\cdot \sqrt{\frac{1}{n}\log \ell +
  \frac{d}{n}\log \frac{2en}{d} + \frac{1}{n} \log \frac{1}{\upeta}}\right)\nonumber\\
& = & \logloss\left({\mathcal{S}}, \ve{\theta}\right) + \frac{1}{m}\cdot
\log \frac{\explossrado({\mathcal{S}},
    \ve{\theta}, \Sigma_r)}{\explossrado({\mathcal{S}},
    \ve{\theta}, \Sigma_m)} \nonumber\\
 & & + \frac{1}{m}\cdot\log\left(1 - 16 \varrho \cdot \sqrt{\frac{1}{n}\log \ell +
  \frac{d}{n}\log \frac{2en}{d} + \frac{1}{n} \log \frac{1}{\upeta}}\right)\label{lasteq}\:\:.
\end{eqnarray}
In eq. (\ref{pp1}), we use the fact that $\varrho = \ell
/ \expect_{{U}} \left[ \explossrado({\mathcal{S}},
    \ve{\theta}, U)\right]$ and $\expect_{{U}} \left[ \explossrado({\mathcal{S}},
    \ve{\theta}, U)\right] = \explossrado({\mathcal{S}},
    \ve{\theta}, \Sigma_r)$. Hence, reordering the expression yields that with probability $\geq 1-\upeta$, the final classifier
$\ve{\theta}$ will satisfy:
\begin{eqnarray}
\logloss\left({\mathcal{S}}, \ve{\theta}\right) & \leq & \loglossrado({\mathcal{S}}, \ve{\theta},
  \mathcal{U}) - \frac{1}{m}\cdot
\log \frac{\explossrado({\mathcal{S}},
    \ve{\theta}, \Sigma_r)}{\explossrado({\mathcal{S}},
    \ve{\theta}, \Sigma_m)} \nonumber\\
 & &  - \log\left(1 - 16 \cdot \frac{\varrho}{\sqrt{n}} \cdot \sqrt{\log \ell+
  d\log \frac{2en}{d} + \log
  \frac{1}{\upeta}}\right)\label{dum11}\:\:.
\end{eqnarray}
There remains to use the fact that 
$\ell \leq \exp(r_\theta \max_{\Sigma_r}
  \left\|\ve{\rado}_{\ve{\sigma}}\right\|_2)$ to complete the proof of ineq. (\ref{thc11}) in Theorem
\ref{thm_concentration}. To prove ineq. (\ref{thc22}), let us call
$1-z$ the quantity inside the log in ineq. (\ref{dum11}). We clearly have to have
$0 \leq
z< 1$, and so for any value of $z$ and for any $0 \leq \alpha < 1$, there exists a value $m_* > 0$
such that 
\begin{eqnarray}
m^{1-\alpha} & \geq & \frac{1}{z}\log \frac{1}{1-z} \:\: (\geq 0)\:\:,\label{binfm}
\end{eqnarray}
for any $m \geq m_*$. In this case, we get after reordering, since
$1-z' \leq \exp z'$,
\begin{eqnarray}
1 - \frac{z}{m^\alpha} & \leq &
\exp\left(-\frac{z}{m^\alpha}\right)\nonumber\\
 & \leq & \exp \left(\frac{1}{m} \log(1-z)\right)\:\:,
\end{eqnarray}
and so, taking logs and using ineq. (\ref{lasteq}), we obtain that for
any $0 \leq \beta < 1/2$, there exists $m_* > 0$ such that for any
$m\geq m_*$:
\begin{eqnarray}
\logloss\left({\mathcal{S}}, \ve{\theta}\right) & \leq & \loglossrado({\mathcal{S}}, \ve{\theta},
  \mathcal{U}) - \frac{1}{m}\cdot
\log \frac{\explossrado({\mathcal{S}},
    \ve{\theta}, \Sigma_r)}{\explossrado({\mathcal{S}},
    \ve{\theta}, \Sigma_m)} \nonumber\\
 & &  - \log\left(1 - 16 \cdot \frac{\varrho}{m^\beta}
    \cdot \sqrt{\frac{r_\theta}{n} \cdot \max_{\Sigma_r}
  \left\|\frac{1}{m} \cdot \ve{\rado}_{\ve{\sigma}}\right\|_2 +
  \frac{d}{n m}\log \frac{2en}{d} + \frac{1}{n m} \log \frac{1}{\upeta}}\right)\:\:.
\end{eqnarray}
Calling $1-z'$ the quantity inside the log, there remains to use
$\log(1-z') \geq - K z'$ for some $K>0$ when $z'$ is sufficiently
close to $0$ (hence, $m$ sufficiently large again). This proves
ineq. (\ref{thc22}) and completes the proof of Theorem
\ref{thm_concentration}. Remark that provided $n$ is sufficiently
large, the right hand-side of ineq (\ref{binfm}) admits the following
equivalent:
\begin{eqnarray}
\frac{1}{z}\log \frac{1}{1-z}  & \sim & 1 + \frac{z}{2}\:\:,
\end{eqnarray}
with $z = \Omega(1/\sqrt{n})$ (omitting the dependences in the other
parameters). Hence, ineq (\ref{binfm})
can be ensured as long as $m$ is large enough with respect to
$r_\theta$, $\max_{\Sigma_r}
  \left\|(1/m) \cdot \ve{\rado}_{\ve{\sigma}}\right\|_2$ (which cannot
  exceed the maximum norm of an observation in ${\mathcal{S}}$), $d$ and
  $\log(1/\upeta)$. 

So, when we apply this last result to \radoboost, it says that for a large enough sample, we can indeed
  pick an $n$ sufficiently large but small compared to $m$ so that we
  shall observe with high probability a decay rate of the
  \textit{expected logistic loss computed over} ${\mathcal{S}}$, $\expect[\logloss\left({\mathcal{S}},
  \ve{\theta}_T\right)]$, of
  order $\Omega(\upgamma^2/m)$ (expectation is measured with respect
  to the sampling of ${\mathcal{U}}$).

We are now left with proving ineq. (\ref{bsupQ}), and so
we study:
\begin{eqnarray}
-Q & = & \frac{1}{m} \cdot \log \left( \frac{\frac{1}{|\Sigma_r|}\sum_{\ve{\sigma}'\in
      \Sigma_r}{\exp(-\ve{\theta}^\top \ve{\rado}_{\ve{\sigma}'})}}{\frac{1}{|\Sigma_m|}\sum_{\ve{\sigma}\in
      \Sigma_m}{\exp(-\ve{\theta}^\top \ve{\rado}_{\ve{\sigma}})}}\right) \nonumber\\
 & = & \frac{1}{m} \cdot \log \left( \frac{|\Sigma_m|\sum_{\ve{\sigma}'\in
      \Sigma_r}{\exp(-\ve{\theta}^\top \ve{\rado}_{\ve{\sigma}'})}}{|\Sigma_r|\sum_{\ve{\sigma}\in
      \Sigma_m}{\exp(-\ve{\theta}^\top \ve{\rado}_{\ve{\sigma}})}}\right) \nonumber\\
 & = & \frac{1}{m} \cdot \log \left( \frac{\sum_{\ve{\sigma}\in
      \Sigma_m}{\sum_{\ve{\sigma}'\in
      \Sigma_r}{\exp(-\ve{\theta}^\top \ve{\rado}_{\ve{\sigma}'})}}}{\sum_{\ve{\sigma}'\in
      \Sigma_r}{\sum_{\ve{\sigma}\in
      \Sigma_m}{\exp(-\ve{\theta}^\top
      \ve{\rado}_{\ve{\sigma}})}}}\right) \nonumber\\
 & = & \frac{1}{m} \cdot \log \left( \frac{\sum_{\ve{\sigma}'\in
      \Sigma_r}{\sum_{\ve{\sigma}\in
      \Sigma_m}{
    \exp(-\ve{\theta}^\top \ve{\rado}_{\ve{\sigma}}) \cdot \exp(-\ve{\theta}^\top (\ve{\rado}_{\ve{\sigma}'} - \ve{\rado}_{\ve{\sigma}}))}}}{\sum_{\ve{\sigma}'\in
      \Sigma_r}{\sum_{\ve{\sigma}\in
      \Sigma_m}{\exp(-\ve{\theta}^\top \ve{\rado}_{\ve{\sigma}})}}}\right) \nonumber\\
 & = & \frac{1}{m} \cdot \log \left( \expect_{(\ve{\sigma}, \ve{\sigma}')\sim D} \left[\exp(-\ve{\theta}^\top
     (\ve{\rado}_{\ve{\sigma}'} - \ve{\rado}_{\ve{\sigma}})) \right]\right) \nonumber\:\:,
\end{eqnarray}
with $D(\ve{\sigma}, \ve{\sigma}') \propto \exp(-\ve{\theta}^\top
\ve{\rado}_{\ve{\sigma}})$. Jensen's inequality yields:
\begin{eqnarray}
-Q & \geq & \frac{1}{m} \cdot \expect_{(\ve{\sigma}, \ve{\sigma}')\sim D} \left[-\ve{\theta}^\top
     (\ve{\rado}_{\ve{\sigma}'} - \ve{\rado}_{\ve{\sigma}}) \right]\nonumber\\
 & & = \frac{1}{m} \cdot \expect_{(\ve{\sigma}, \ve{\sigma}')\sim D} \left[\ve{\theta}^\top\ve{\rado}_{\ve{\sigma}} \right] - \frac{1}{m} \cdot \expect_{(\ve{\sigma}, \ve{\sigma}')\sim D} \left[\ve{\theta}^\top
     \ve{\rado}_{\ve{\sigma}'} \right]\label{eqespq}\:\:.
\end{eqnarray}
We now remark that
\begin{eqnarray}
\expect_{(\ve{\sigma}, \ve{\sigma}')\sim D} \left[\ve{\theta}^\top
     \ve{\rado}_{\ve{\sigma}'} \right] & = & \frac{\sum_{\ve{\sigma}'\in
      \Sigma_r}{\sum_{\ve{\sigma}\in
      \Sigma_m}{
    \exp(-\ve{\theta}^\top \ve{\rado}_{\ve{\sigma}}) \cdot \ve{\theta}^\top
     \ve{\rado}_{\ve{\sigma}'} }}}{\sum_{\ve{\sigma}'\in
      \Sigma_r}{\sum_{\ve{\sigma}\in
      \Sigma_m}{\exp(-\ve{\theta}^\top
      \ve{\rado}_{\ve{\sigma}})}}}\nonumber\\
 & = & \ve{\theta}^\top \left( \frac{\left(\sum_{\ve{\sigma}'\in
      \Sigma_r}{
     \ve{\rado}_{\ve{\sigma}'}}\right)\cdot\left(\sum_{\ve{\sigma}\in
      \Sigma_m}{
    \exp(-\ve{\theta}^\top \ve{\rado}_{\ve{\sigma}})}\right)}{\sum_{\ve{\sigma}'\in
      \Sigma_r}{\sum_{\ve{\sigma}\in
      \Sigma_m}{\exp(-\ve{\theta}^\top
      \ve{\rado}_{\ve{\sigma}})}}} \right) \nonumber\\
 & = & \ve{\theta}^\top \expect_{\ve{\sigma}\sim \Sigma_r}[\ve{\rado}_{\ve{\sigma}}] \label{eqespq2}\:\:,
\end{eqnarray}
and furthermore
\begin{eqnarray}
\frac{1}{m} \cdot \expect_{(\ve{\sigma}, \ve{\sigma}')\sim D} \left[\ve{\theta}^\top
     \ve{\rado}_{\ve{\sigma}} \right] & = & \frac{1}{m} \cdot \frac{\sum_{\ve{\sigma}'\in
      \Sigma_r}{\sum_{\ve{\sigma}\in
      \Sigma_m}{
    \exp(-\ve{\theta}^\top \ve{\rado}_{\ve{\sigma}}) \cdot \ve{\theta}^\top
     \ve{\rado}_{\ve{\sigma}} }}}{\sum_{\ve{\sigma}'\in
      \Sigma_r}{\sum_{\ve{\sigma}\in
      \Sigma_m}{\exp(-\ve{\theta}^\top
      \ve{\rado}_{\ve{\sigma}})}}}\nonumber\\
& = & \frac{1}{m} \cdot \frac{\sum_{\ve{\sigma}\in
      \Sigma_m}{
    \exp(-\ve{\theta}^\top \ve{\rado}_{\ve{\sigma}}) \cdot \ve{\theta}^\top
     \ve{\rado}_{\ve{\sigma}} }}{\sum_{\ve{\sigma}\in
      \Sigma_m}{\exp(-\ve{\theta}^\top
      \ve{\rado}_{\ve{\sigma}})}}\nonumber\\
 & = & \ve{\theta}^\top \left( \frac{1}{m} \cdot \frac{\sum_{\ve{\sigma}\in
      \Sigma_m}{
    \exp(-\ve{\theta}^\top \ve{\rado}_{\ve{\sigma}}) \cdot 
     \ve{\rado}_{\ve{\sigma}} }}{\sum_{\ve{\sigma}\in
      \Sigma_m}{\exp(-\ve{\theta}^\top
      \ve{\rado}_{\ve{\sigma}})}} \right) \nonumber\\
 & = & \ve{\theta}^\top \ve{\nabla}_{\ve{\theta}}
 \frac{1}{m} \cdot \log \explossrado\left({\mathcal{S}}, \ve{\theta}, \Sigma_m\right) \nonumber\\
 & = & \ve{\theta}^\top \ve{\nabla}_{\ve{\theta}}
 \loglossrado\left({\mathcal{S}}, \ve{\theta}, \Sigma_m\right) \label{eqespq3}\:\:.
\end{eqnarray}
Assembling eqs (\ref{eqespq2}) and (\ref{eqespq3}), we get from ineq. (\ref{eqespq}):
\begin{eqnarray}
Q & \leq &  r_\theta \left\|\ve{\nabla}_{\ve{\theta}}
 \loglossrado\left({\mathcal{S}}, \ve{\theta}, \Sigma_m\right) -
 \expect_{\ve{\sigma}\sim \Sigma_r}\left[\frac{1}{m} \cdot
   \ve{\rado}_{\ve{\sigma}}\right]\right\|_2\nonumber\\
 & \leq & r_\theta\left(\|\ve{\nabla}_{\ve{\theta}}
 \loglossrado\left({\mathcal{S}}, \ve{\theta}, \Sigma_m\right)\|_2 + \left\|\expect_{\ve{\sigma}\sim \Sigma_r}\left[\frac{1}{m} \cdot
   \ve{\rado}_{\ve{\sigma}}\right]\right\|_2\right)\:\:,\nonumber
\end{eqnarray}
as claimed.

\subsection{Proof of Lemma \ref{lem_radoboost}}\label{proof_lem_radoboost}

Theorem 1 in (\cite{nnARj}) immediately yields
\begin{eqnarray}
\frac{1}{n} \exp\left( -\ve{\theta}_T^\top\ve{\rado}_j\right) & 
\leq & \prod_{t=1}^{T} {\sqrt{1-r^2_{t}}} \cdot w_{(T+1)j}
\:\:, \forall j\in [n]\:\:.
\end{eqnarray}
Since $\ve{1}^\top \ve{w}_{T+1} = 1$, summing over $j\in [n]$ yields:
\begin{eqnarray}
\explossrado({\mathcal{S}}, \ve{\theta}_T, \mathcal{U}) & \leq &
\prod_{t=1}^{T} {\sqrt{1-r^2_{t}}} \nonumber\\
 & \leq & \exp\left(-\frac{1}{2}\sum_t r_t^2\right)\:\:.\nonumber
\end{eqnarray}
Using the (\textbf{WLA}), this yields ineq. (\ref{explossbound}). 

\subsection{Proof of Lemma \ref{lem_wla}}\label{proof_lem_wla}

Fix for short $k = \iota(t)$. We rewrite $r_t(\ve{w}_t)$ as a function
of the examples:
\begin{eqnarray}
r_t(\ve{w}_t)  & = & \frac{1}{\rado_{*k}}
\sum_{j=1}^{n} {w_{tj} \rado_{j k}} \nonumber\\
 & = & \frac{1}{\rado_{*k}}
\sum_{j=1}^{n}\sum_{i : \sigma_{ji} = y_i} {w_{tj} y_i x_{ik}} \nonumber\\
 & = & \frac{1}{x_{*k}} 
\sum_{i=1}^{m} {\left(\frac{x_{*k}}{\rado_{*k}} \cdot \sum_{j : \sigma_{ji} = y_i} w_{tj}\right) y_i x_{ik}} \:\:.
\end{eqnarray}
Define $\tilde{\ve{w}} \in {\mathbb{P}}^m$ such that
\begin{eqnarray}
\tilde{w}_{i} & \defeq & \frac{1}{\tilde{W}} \cdot \frac{x_{*k}}{\rado_{*k}} \cdot
\sum_{j : \sigma_{ji} = y_i} w_{tj}\:\:, \forall i \in [m]\:\:,
\end{eqnarray}
with 
\begin{eqnarray}
\tilde{W} & \defeq & \frac{x_{*k}}{\rado_{*k}} \cdot \sum_{i=1}^{m}\sum_{j :
   \sigma_{ji} = y_i} {w_{tj}}\nonumber\\
 & & =  \frac{x_{*k}}{\rado_{*k}} \cdot \sum_{j=1}^{n} {w_{tj} |\{i
   : \sigma_{ji} = y_i\}|}
\end{eqnarray}
the normalization coefficient. Because $\ve{w}_t \in {\mathbb{P}}^n$,
$x_{*k} > 0$ and $\rado_{*k} > 0$, it comes that indeed $\tilde{\ve{w}}
\in {\mathbb{P}}^m$, and $\tilde{W} > 0$ (unless ${\mathcal{S}}^r$ is
reduced to the null rado). We thus have $|r_t(\ve{w}_t)| \geq \upgamma$ iff
\begin{eqnarray}
|r^{ex}_t(\tilde{\ve{w}})| & \geq & \frac{\upgamma}{\tilde{W}}\:\:.
\end{eqnarray}
This proves the statement of the Lemma. Remark that 
\begin{eqnarray}
\frac{x_{*k}}{\rado_{*k}} \leq \tilde{W} \leq  
 \frac{x_{*k}}{\left(\frac{\rado_{*k}}{\max_{j} |\{i
   : \sigma_{ji} = y_i\}| }\right)}\:\:,
\end{eqnarray}
so if we assume the weak learning assumption holds for the
examples, $|r^{ex}_t(\tilde{\ve{w}})| \geq \upgamma^{ex} > 0$, then the
weak learning assumption over rados always holds for
\begin{eqnarray}
  \upgamma & = &  \frac{x_{*k}}{\rado_{*k}} \cdot \upgamma^{ex}\:\:,
\end{eqnarray}
and may holds for a value $\upgamma$ which can be as large as 
\begin{eqnarray}
\upgamma & = & \frac{x_{*k}}{\left(\frac{\rado_{*k}}{\max_{j} |\{i
   : \sigma_{ji} = y_i\}| }\right)} \cdot
   \upgamma^{ex} \:\:.
\end{eqnarray}
These two bounds are data dependent (but they depend on data
\textit{only}), and whenever they are significant outlier values for
feature $k$, \textit{i.e.}
$x_{*k}$ is achieved by few examples and all others have feature value
of significantly smaller order, then the available $\upgamma$ 
can be significantly larger than $\upgamma^{ex}$. Compared to the
cases where no such outliers would exist, we thus
may expect significantly better results for \radoboost.

\subsection{Proof of Theorem \ref{thm_dpfreal}}\label{proof_thm_dpfreal}

To ease notations hereafter, we consider wlog that $d=1$ and so $j_*=1$.
We also drop index notation $j_*$ in related notations (so $\sbj$
becomes $\sbnj$).

We let ${\mathcal{S}}$ and ${\mathcal{S}}'$ denote two
\textit{$j$-neighbors}, so that $\mathcal{S}\approx_{j}
\mathcal{S}'$ holds and they differ by the value of one (boolean) feature. 
Algorithm \dpfreal~selects uniformly at random the rados in sets
\begin{eqnarray}
\sbnj({\mathcal{S}})
 & \defeq & \left\{\ve{\sigma} \in \Sigma_m : 
 \rado_{\ve{\sigma}} \in \mathbb{I}({\mathcal{S}})\right\}\:\:,\label{defSrm1}\\
\sbnj({\mathcal{S}}')
 & \defeq & \left\{\ve{\sigma} \in \Sigma_m : 
 \rado_{\ve{\sigma}} \in \mathbb{I}({\mathcal{S}}')\right\}\:\:,\label{defSprimerm1}
\end{eqnarray}
with 
\begin{eqnarray}
\mathbb{I}({\mathcal{S}}) & \defeq & \{-(m-m(+)) + \beta(m+1)\leq z\leq m(+) - \beta(m+1)\}\:\:,\label{defI1}\\
\mathbb{I}({\mathcal{S}}') & \defeq & \{-(m-m(+)) + \beta(m+1) +
   \zeta\leq z\leq m(+) - \beta(m+1)
 + \zeta\}\:\:,\label{defI2}
\end{eqnarray}
since $m'(+) = m(+) + \zeta$ for some $\zeta \in \{-1,0,1\}$. To
relate the sizes of these two sets, we first
compute the size of $\{\ve{\sigma} : \rado_{\ve{\sigma}} = r
| {\mathcal{S}}\}$, for $r \in \mathbb{Z}$. Assuming first $r\geq 0$, we have:
\begin{eqnarray}
\left|\{\ve{\sigma} : \rado_{\ve{\sigma}} = r
| {\mathcal{S}}\}\right| & = & \sum_{i=0}^{\min\{m(+)-r, m-m(+)\}} {{m(+) \choose i+r} {m-m(+) \choose i}}\:\:. \label{ppp2}
\end{eqnarray}
If $r<0$, then similarly:
\begin{eqnarray}
\left|\{\ve{\sigma} : \rado_{\ve{\sigma}} = r
| {\mathcal{S}}\}\right| & = & \sum_{i=0}^{\min\{m(+), m-m(+)+r\}}
{{m(+) \choose i} {m-m(+) \choose i-r}}\label{ppn2}\:\:,
\end{eqnarray}
which is the same expression as (\ref{ppp2}) with the substitutions $r
\mapsto -r$, $m(+)
\mapsto m-m(+)$, $m-m(+)
\mapsto m(+)$, so we have only to analyse the case $r\geq 0$. 
If $m(+)-r > m-m(+)$, we have by Vandermonde identity:
\begin{eqnarray}
\left|\{\ve{\sigma} : \rado_{\ve{\sigma}} = r
| {\mathcal{S}}\}\right| & = & \sum_{i=0}^{m-m(+)} {{m(+) \choose m(+)-i-r} {m-m(+) \choose i}}\nonumber\\
 & = & {m \choose m(+)-r} \label{ppp3}\:\:.
\end{eqnarray}
If $m(+)-r \leq m-m(+)$, then it is not hard to show that Vandermonde identity still
brings (\ref{ppp3}).
We thus have
\begin{eqnarray}
|\sbnj({\mathcal{S}})| & = & \sum_{r = -(m-m(+)) +
  \beta(m+1)}^{m(+) - \beta(m+1)} {{m \choose m(+)-r}}\nonumber\\
 & = & {m \choose \beta(m+1)} + \sum_{r = -(m-m(+)) +
  \beta(m+1) + 1}^{m(+) - \beta(m+1)} {{m \choose m(+)-r}}\nonumber\\
 & = & \frac{m - \beta(m+1)+1}{\beta(m+1)} \cdot {m \choose \beta(m+1) - 1} + \sum_{r = -(m-m(+)) +
  \beta(m+1) + 1}^{m(+) - \beta(m+1)} {{m \choose m(+)-r}}\nonumber\\
 & = & \left(\frac{1}{\beta} - 1\right) \cdot {m \choose \beta(m+1) - 1} + \sum_{r = -(m-m(+)) +
  \beta(m+1) + 1}^{m(+) - \beta(m+1)} {{m \choose m(+)-r}}\nonumber\\
 & \geq & {m \choose \beta(m+1) - 1} + \sum_{r = -(m-m(+)) +
  \beta(m+1) + 1}^{m(+) - \beta(m+1)} {{m \choose m(+)-r}}\nonumber\\
 &  & = \sum_{r = -(m-m(+)) +
  \beta(m+1) + 1}^{m(+) - \beta(m+1)+ 1} {{m \choose m(+)-r}} \nonumber\\
 &  & = \sum_{r = -(m-m(+)) +
  \beta(m+1) + 1}^{m(+) - \beta(m+1)+ 1} {\frac{m(+)+1-r}{m-m(+)+r} \cdot {m \choose (m(+)+1)-r}} \label{chain1}\\
 &   \geq & \sum_{r = -(m-m(+)) +
  \beta(m+1) + 1}^{m(+) - \beta(m+1)+ 1}
{\frac{\beta(m+1)}{m-\beta(m+1)+1} \cdot {m \choose (m(+)+1)-r}}
\label{chain2}\\
 & & = \left(\frac{1}{\beta} - 1\right)^{-1}\sum_{r = -(m-m(+)) +
  \beta(m+1) + 1}^{m(+) - \beta(m+1)+ 1} { {m \choose (m(+)+1)-r}} \label{chain3}\\
 &  & = \left(\frac{1}{\beta} - 1\right)^{-1} \cdot
 |\sbnj({\mathcal{S}}')| \label{chain4}
\end{eqnarray}
if $\zeta = 1$, and
\begin{eqnarray}
|\sbnj({\mathcal{S}})| & = & \sum_{r = -(m-m(+)) +
  \beta(m+1)}^{m(+) - \beta(m+1)} {{m \choose m(+)-r}}\nonumber\\
 & = & {m \choose \beta(m+1)} + \sum_{r = -(m-m(+)) +
  \beta(m+1)}^{m(+) - \beta(m+1)- 1} {{m \choose m(+)-r}}\nonumber\\
 & = & \frac{m - \beta(m+1)+1}{\beta(m+1)} \cdot {m \choose \beta(m+1) - 1} + \sum_{r = -(m-m(+)) +
  \beta(m+1)}^{m(+) - \beta(m+1)-1} {{m \choose m(+)-r}}\nonumber\\
 & = & \left(\frac{1}{\beta} - 1\right) \cdot {m \choose \beta(m+1) - 1} + \sum_{r = -(m-m(+)) +
  \beta(m+1)}^{m(+) - \beta(m+1)-1} {{m \choose m(+)-r}}\nonumber\\
 & \geq & {m \choose \beta(m+1) - 1} + \sum_{r = -(m-m(+)) +
  \beta(m+1)}^{m(+) - \beta(m+1)-1} {{m \choose m(+)-r}}\nonumber\\
 &  & = \sum_{r = -(m-m(+)) +
  \beta(m+1) -1}^{m(+) - \beta(m+1)- 1} {{m \choose m(+)-r}}
\nonumber\\
 & \geq & \left(\frac{1}{\beta} - 1\right)^{-1} \cdot |\sbnj({\mathcal{S}}')|
\end{eqnarray}
if $\zeta = -1$. The last inequality follows from the same chain of
inequalities as in eqs. (\ref{chain1} -- \ref{chain4}).
We now bound the ratio of probabilities for the rado being equal to
$r$, for both sets:
\begin{eqnarray}
\frac{\pr_{\ve{\sigma} \sim
    \sbnj({\mathcal{S}})}\left[\rado_{\ve{\sigma}} = r | {\mathcal{S}}\right]}{\pr_{\ve{\sigma} \sim
    \sbnj({\mathcal{S}}')}\left[\rado_{\ve{\sigma}} = r |
    {\mathcal{S}}'\right]} & = & \frac{|\sbnj({\mathcal{S}}')|}{|\sbnj({\mathcal{S}})|} \cdot \frac{{m \choose
    m(+) -r}}{{m \choose m(+) + \zeta -r}}\nonumber\\
 & \leq & \left(\frac{1}{\beta} - 1\right) \cdot \frac{{m \choose
    m(+) -r}}{{m \choose m(+) + \zeta -r}}\nonumber\\
 & & = \left(\frac{1}{\beta} - 1\right) \cdot  \frac{(m(+)+\zeta-r)! (m-m(+)-\zeta+r)!}{(m(+)-r)! (m-m(+)+r)!}\\
 & &  = \left(\frac{1}{\beta} - 1\right) \cdot \left\{
\begin{array}{rcl}
\frac{m(+)+1-r}{m-m(+)+r} & \mbox{ if } & \zeta = 1\\
1& \mbox{ if } & \zeta = 0\\
\frac{m-m(+)+1+r}{m(+)-r} & \mbox{ if } & \zeta = -1
\end{array}
\right. \nonumber\\
 & \leq & \left(\frac{1}{\beta} - 1\right)^2\:\:.\label{vardelta}
\end{eqnarray}
The last inequality comes from eq. (\ref{defI1}) which guarantees $r \geq -(m-m(+)) +
\beta(m+1)$, and so
\begin{eqnarray}
\frac{m(+)+1-r}{m-m(+)+r} & \leq & \frac{1}{\beta} - 1\:\:,
\end{eqnarray}
and furthermore eq. (\ref{defI1}) also guarantees $r \leq m(+) - \beta(m+1)$, and so 
\begin{eqnarray}
\frac{m-m(+)+1+r}{m(+)-r} & \leq & \frac{1}{\beta} - 1
\end{eqnarray}
as well. We finally get from ineq. (\ref{vardelta}):
\begin{eqnarray}
\frac{\pr_{\ve{\sigma} \sim
    \sbnj({\mathcal{S}})}\left[\rado_{\ve{\sigma}} = r | {\mathcal{S}}\right]}{\pr_{\ve{\sigma} \sim
    \sbnj({\mathcal{S}}')}\left[\rado_{\ve{\sigma}} = r |
    {\mathcal{S}}'\right]} & \leq & \exp(\upepsilon)\:\:,
\end{eqnarray}
which holds for any $r \in \sbnj({\mathcal{S}})\cap
\sbnj({\mathcal{S}}')$. Notice however that the symmetric difference
of these two sets is not empty. To finish the proof, we need to
take into account this symmetric difference. This is the
data-dependent step in \dpfreal~which may leak information about one
feature and disclose its content, through the use of eq. (\ref{defI1}). To
see this, if we assume that one possesses all the data but the unknown
feature value for one person, and knows how rados are computed using
\dpfreal, then by observing the output $\rado_{\ve{\sigma}, j_*}$, he
may guess the unknown value, as depicted by Figure
\ref{dpexpl_2}. Let us denote $A$ this event.
When returning one rado from $\sbnj(.)$, if
we consider without loss of generality a uniform distribution over
examples, then, referring to the notations of Figure \ref{dpexpl_2}, we have:
\begin{eqnarray}
\pr[A] & = & \pr[A | {\mathcal{S}}] \pr[{\mathcal{S}}] +  \pr[A |
{\mathcal{S}}'] \pr[{\mathcal{S}}']\label{defpa}\\
 & < &  \pr[A | {\mathcal{S}}] + \pr[A | {\mathcal{S}}']\:\:.
\end{eqnarray}
If $A$ occurs in ${\mathcal{S}}$, then it is for $r = m(+) - (m-\beta(m+1))$ in Figure
\ref{dpexpl_2}. We get from eq. (\ref{ppp3}):
\begin{eqnarray}
\pr[A | {\mathcal{S}}] & = & \frac{{m \choose
    m-\beta(m+1)}}{\sum_{r=\beta(m+1)}^{m-\beta(m+1)}{m \choose
    r}}\nonumber\\
 & =& \frac{{m \choose
   \beta(m+1)}}{\sum_{r=\beta(m+1)}^{m-\beta(m+1)}{m \choose
    r}}\:\:,
\end{eqnarray}
and we obtain following the same reasoning, using the fact that $m(+)$
increases by one in ${\mathcal{S}}'$,
\begin{eqnarray}
\pr[A | {\mathcal{S}}'] & = & \frac{{m \choose
    \beta(m+1)}}{\sum_{r=\beta(m+1)}^{m-\beta(m+1)}{m \choose
    r}}\:\:.
\end{eqnarray}
The probability of hitting the symmetric difference of $\sbnj({\mathcal{S}})\cap
\sbnj({\mathcal{S}}')$ is taken into account considering
 $\updelta = \pr[A]$ in the $(\upepsilon,
\updelta)$-differentially private release of one rado. We get:
\begin{eqnarray}
\updelta & < & \frac{2{m \choose
    \beta(m+1)}}{\sum_{r=\beta(m+1)}^{m-\beta(m+1)}{m \choose
    r}}\:\:.\label{delta1}
\end{eqnarray}
The interplay between $\upepsilon$ and $\updelta$ can be appreciated
throughout the use of the following properties:
\begin{eqnarray}
\sum_{r = 0}^{\beta(m+1)-1} {m\choose r} & \leq & 2^{m \cdot
  H(u)} \:\:,\label{ddd1}\\
 {m \choose m/2} & < & \frac{1}{\sqrt{m}} \cdot 2^m\label{ddd2}\:\:,
\end{eqnarray}
we have used
\begin{eqnarray*}
H(z) & \defeq & - z \log_2 z - (1-z) \log_2(1-z)\:\:,\\
u & \defeq & \beta - \frac{1-\beta}{m}\:\:.
\end{eqnarray*}
We get
\begin{eqnarray}
\updelta & < & \frac{2}{\sqrt{m}} \cdot \frac{1}{1 - 2^{m \cdot
  (H(u)-1)}}
\end{eqnarray}
Because $H(u)$ is concave, it satisfies (fixing $\upepsilon' \defeq
\upepsilon / 2$ for short):
\begin{eqnarray}
H(u) & \leq & H(\beta) + (u-\beta)H'(\beta)\nonumber\\
 & & = H(\beta) -
 \frac{1-\beta}{m}\log_2\frac{1-\beta}{\beta}\nonumber\\
 & = & H(\beta) - \frac{(1-\beta)\upepsilon'}{m}\nonumber\\
 & = & \frac{1}{\log 2} \cdot \left( \log(1+\exp(\upepsilon')) - \left(
   1 + \frac{1}{m}\right)\cdot \frac{\upepsilon' \exp \upepsilon'}{1 +
   \exp\upepsilon'}\right) \defeq f(\upepsilon')\:\:.
\end{eqnarray}
We have:
\begin{eqnarray}
\frac{1}{1 - 2^{m \cdot
  (f(\upepsilon')-1)}} & \sim_0 & \frac{1}{2m^2\log^2(2)
\upepsilon'} + \left(\frac{1}{2} - \frac{1}{4m^2 \log^3(2)}\right) + O(\upepsilon' )\:\:.
\end{eqnarray}
So, assuming $\upepsilon' = o(1)$, there exists $m' > 0$ and a
constant $K>0$ such that for any $m > m'$,
\begin{eqnarray}
\updelta & < & K \cdot \frac{1}{m^{\frac{5}{2}} \upepsilon}\:\:.
\end{eqnarray}
Finally, we get that when $\upepsilon = \Omega(1/m)$, $(\upepsilon,
\updelta)$-differential privacy can be ensured on the delivery of $n=1$ rado as long as
$\upepsilon \cdot \updelta = O(m^{-5/2})$. 
\begin{figure}[t]
\begin{center}
\includegraphics[width=0.9\columnwidth]{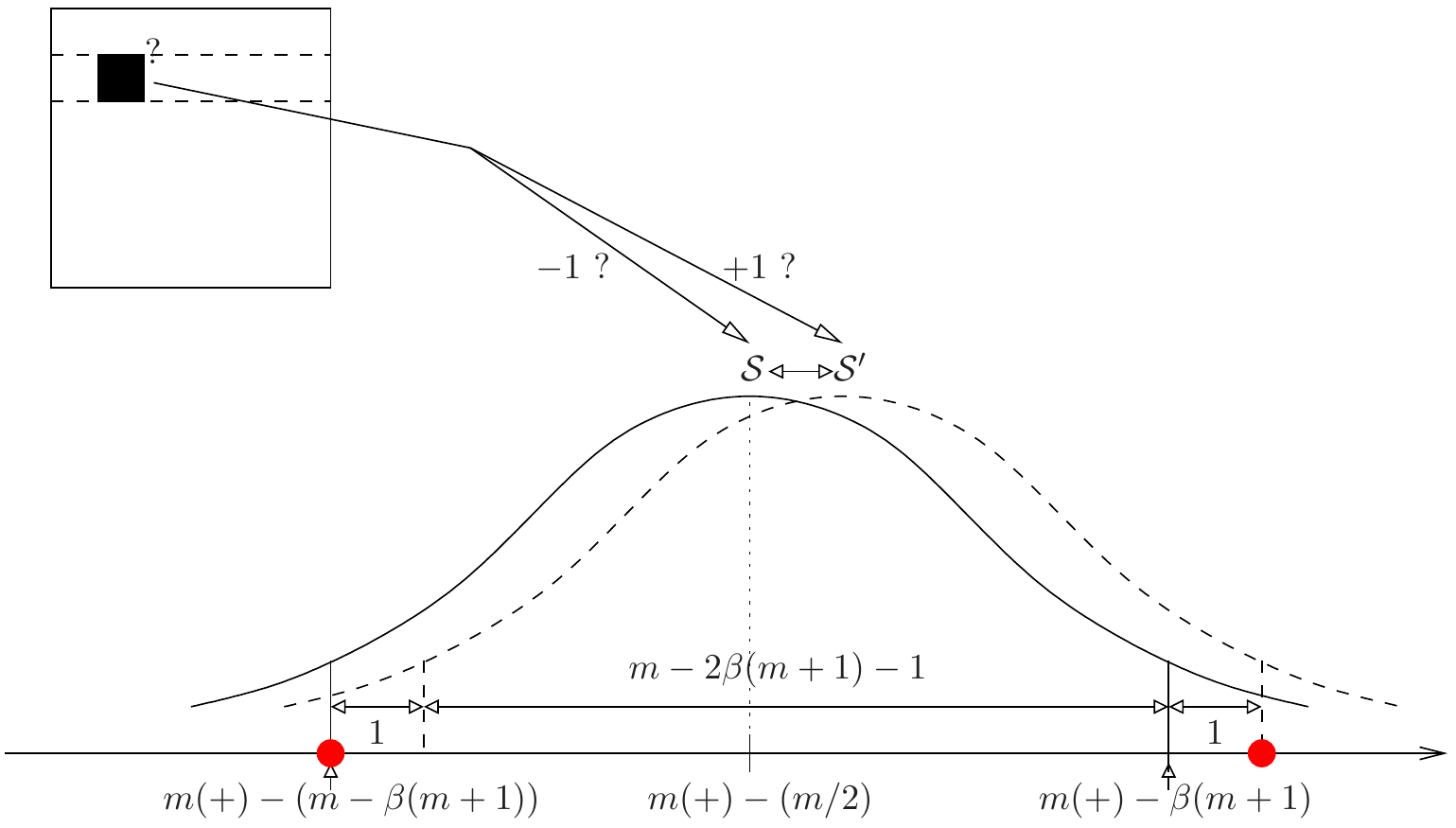}
\caption{Knowing everything (including \dpfreal) but the actual
  feature value for a particular individual (in black), one can hack
  this unknown if he/she is returned by \dpfreal~a rado whose value
  $\rado_{\ve{\sigma}, j_*}$ falls within the two red dots: if it is
  the left one, the value is $-1$, and if it is the right one, the
  value is $+1$. The probability of hitting one of the red dots for one rado is
  $\pr[A]$ in eq. (\ref{defpa}).}
\label{dpexpl_2}
\end{center}
\end{figure}
Taking into account the
fact that rados are generated independently and using Theorem 3.16 in \cite{drTA} concludes the proof of
Theorem \ref{thm_dpfreal} for arbitrary $n$.

To finish the proof, we remark that $\sbnj(.) \neq \emptyset$. Indeed, since $m\geq 1$, $\beta <
m/(m+1)$; furthermore, as long as $m>2$, provided we also have
\begin{eqnarray}
\frac{1+2\beta}{1-2\beta} & = & O(m)\:\:,\nonumber
\end{eqnarray}
we shall have ${\mathbb{I}}({\mathcal{S}}) \cap {\mathbb{Z}} \neq
\emptyset$. This can easily be ensured if
\begin{eqnarray}
\frac{1}{\upepsilon} +\upepsilon & = & O(m)\:\:,
\end{eqnarray}
\textit{i.e.}, provided $\upepsilon = o(1)$,  $\upepsilon = \Omega(1/m)$.

\subsection{Proof of Theorem \ref{thm_rrs}}\label{proof_thm_rrs}

We keep the same notations as in the proof of Theorem \ref{thm_dpfreal}.
The Rademacher rejection sampling of $\ve{\sigma}$ has a
probability to reject a single rado bounded by (a fraction of) the tail of the Binomial,
as indeed 
\begin{eqnarray}
\pr_{\ve{\sigma}\sim \Sigma_m} [\ve{\sigma} \not\in \sbnj|
{\mathcal{S}}] & = & \frac{1}{2^m} \cdot \sum_{r < -(m-m_k(+)) + \beta(m+1)
  \vee r > m_k(+) - \beta(m+1)} {{m \choose m(+)-r}}\nonumber\\
 & = & \frac{1}{2^m} \cdot \sum_{r = -(m-m_k(+))}^{-(m-m_k(+)) + \beta(m+1)
   - 1}  {{m \choose m(+)-r}} + \frac{1}{2^m} \cdot \sum_{r = m_k(+) -
   \beta(m+1) + 1}^{m(+)}  {{m \choose m(+)-r}} \nonumber\\
 & = & \frac{1}{2^m} \cdot \sum_{r = m - \beta(m+1) + 1}^{m}  {{m \choose r}} + \frac{1}{2^m} \cdot \sum_{r = 0}^{\beta(m+1)-1}  {{m \choose r}}  \nonumber\\
 & = & 2\cdot \frac{1}{2^m}\cdot  \sum_{r = m - \beta(m+1) + 1}^{m}  {{m \choose r}}\nonumber\\
 & = & 2\cdot \frac{1}{2^m}\cdot  \sum_{r = (1- \beta)(m+1)}^{m}
 {\frac{m+1-r}{m+1} \cdot {m
     +1 \choose r}}\nonumber\\
 & \leq & 2 \beta \cdot \frac{1}{2^m}\cdot  \sum_{r = (1- \beta)(m+1)}^{m}
 { {m
     +1 \choose r}}\nonumber\\
 & \leq & 2 \beta \cdot \frac{1}{2^m}\cdot  \sum_{r = (1- \beta)(m+1)}^{m+1}
 { {m
     +1 \choose r}}\nonumber\\
 &  & = 4 \beta \cdot  \sum_{r = (1- \beta)(m+1)}^{m+1}
 { {m
     +1 \choose r} \cdot \left(\frac{1}{2}\right)^{m+1-r} \cdot
   \left(\frac{1}{2}\right)^{r} }\nonumber\\
 & \leq & 4\beta \exp\left(-(m+1) \cdot D_{BE}(1-\beta\| 1/2)\right)\:\:,
\end{eqnarray}
where $D_{BE}$ is the bit-entropy divergence (\cite{bnnBV}):
\begin{eqnarray}
D_{BE}(p\|q) & = & p
\log\frac{p}{q} + (1-p) \log \frac{1-p}{1-q}\:\:.
\end{eqnarray}
 The last equation follows \textit{e.g.}
from Theorem 2 in (\cite{agTO}). So the probability $p$ that there exists a
rado, among the $n$ generated, that was rejected at least $T_r$
times for some $T_r\geq 1$ satisfies
\begin{eqnarray}
p & \leq & 4 n \beta \sum_{t=T_r}^{\infty} \exp\left(-(m+1) \cdot t
  \cdot D_{BE}(1-\beta\| 1/2)\right)\nonumber\\
 & & = 4 n \beta \cdot \exp\left(-(m+1) \cdot T_r
  \cdot D_{BE}(1-\beta\| 1/2)\right) \cdot \sum_{t=0}^{\infty} {\exp\left(-(m+1) \cdot t
  \cdot D_{BE}(1-\beta\| 1/2)\right)}
\end{eqnarray}
We now use the facts that (i) $m \geq (1+2\beta)/(1-2\beta)$ (Step 2
in Algorithm \dpfreal), and (ii) function
\begin{eqnarray}
f(z) & \defeq & \frac{2}{1-2z} \cdot \left( \log(2) + (1-z) \log(1-z) +
  z\log z\right)
\end{eqnarray}
is convex over $[0,1/2)$ and has limit tangent $1 - 2z$ in $z = 1/2$, so
\begin{eqnarray}
\exp\left(-(m+1) 
  \cdot D_{BE}(1-\beta\| 1/2)\right) & \leq & \exp\left(
-\frac{2}{1-2\beta} \cdot \left( \log(2) + (1-\beta) \log(1-\beta) +
  \beta\log \beta\right) \right)\nonumber\\
 & \leq & \exp(2\beta - 1) \:\: (<1) \:\:,\nonumber
\end{eqnarray}
and it comes
\begin{eqnarray}
\sum_{t=0}^{\infty} {\exp\left(-(m+1) \cdot t
  \cdot D_{BE}(1-\beta\| 1/2)\right)} & \leq & \frac{1}{1-\exp(2\beta-1)}\:\:,
\end{eqnarray}
and so
\begin{eqnarray}
p & \leq & \frac{4 n \beta}{1-\exp(2\beta-1)} \cdot \exp\left(-(m+1) \cdot T_r
  \cdot D_{BE}(1-\beta\| 1/2)\right)
\end{eqnarray}
So, if $n, \beta, \upeta$ are such that
\begin{eqnarray}
n & \leq & \frac{\upeta (1-\exp(2\beta-1))}{4\beta}\:\:,
\end{eqnarray}
then there is probability $\geq 1 - \upeta$ that no rado was
rejected. Otherwise, with probability $\geq 1 - \upeta$, each rado among the
$n$ was
rejected no more than 
\begin{eqnarray}
T_r^* & = & \left\lceil \frac{1}{m D_{BE}(1-\beta\| 1/2)} \log \frac{4\beta
  n}{\upeta(1-\exp(2\beta-1))} \right\rceil \label{deftstar}
\end{eqnarray}
times. There remains to multiply this bound by the number of rados to
get an upperbound on the number of iterations of Rademacher rejection
sampling, and we obtain eq. (\ref{boundTRrs}). This finishes the proof
of Theorem \ref{thm_rrs}.\\

\noindent \textbf{Remarks}: the actual dependence of eq. (\ref{deftstar}) on $\beta$ is such that
unless $\upepsilon$ is extremely close to 0\footnote{Recall that
  $\beta = 1/(1+\exp(\upepsilon/2))$ in Step 1 of Algorithm \dpfreal.}, in which case the
requirement on differential privacy is the strongest, $T_r^*$ does not
actually blow up. To see this, let us define 
\begin{eqnarray}
f(\beta) & \defeq & \frac{1}{D_{BE}(1-\beta\| 1/2)} \log
\frac{4\beta}{1-\exp(2\beta-1)} \:\:. \label{deffbeta}
\end{eqnarray}
\begin{figure}[t]
\vskip 0.2in
\begin{center}
\begin{tabular}{cc}
\includegraphics[width=0.45\columnwidth]{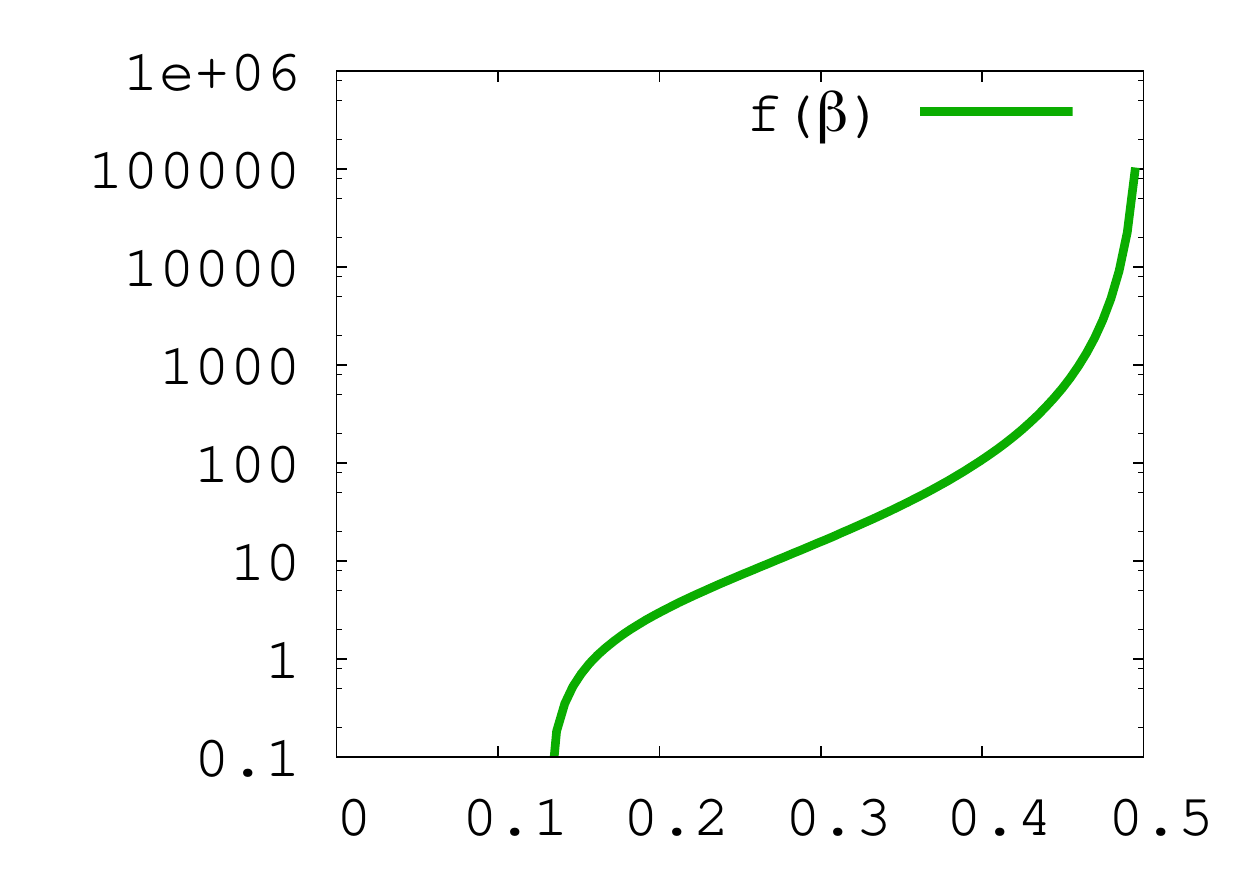}
& \includegraphics[width=0.45\columnwidth]{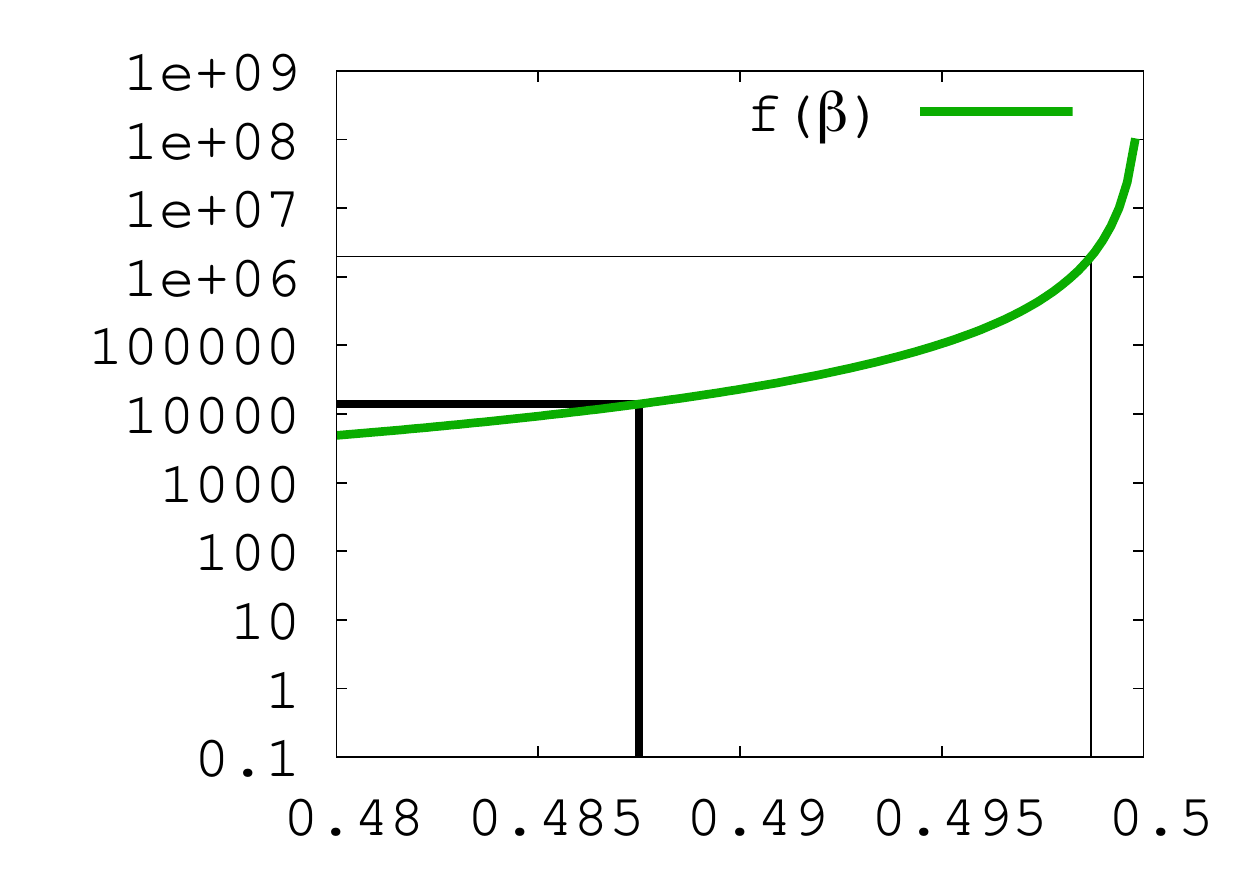}
\end{tabular}
\caption{Left: function $f(\beta)$ as depicted in
  eq. (\ref{deffbeta}). Right: same function over smaller range,
  depicting the value of $f$ for $\upepsilon = 0.1$ (thick dark line)
  and $\upepsilon = 0.01$ (slim dark line).}
\label{f-beta}
\end{center}
\end{figure} 
Figure \ref{f-beta} displays $f(\beta)$ over different ranges. One
sees that when $\upepsilon = 0.1$, provided $m/\log n$ is in the order
of thousands and $n \gg e$, then $T_r^*$ is in fact of the order
$\log(1/\upeta)$, which may be quite small indeed.

\subsection{Proof of Theorem \ref{thm_random_gau}}\label{proof_thm_random_gau}

Let us first remark that the DP-protection of vector edges by computing
noisified example set 
\begin{eqnarray}
{\mathcal{S}}^{+} & \defeq &
\{(\bm{x}^+_i, y_i) \defeq (\bm{x}_i + \bm{x}_i^r, y_i), i\in
[m]\}\:\:,
\end{eqnarray} 
where
$\bm{x}_i^r \sim {\mathcal{N}}(\ve{0},
\varsigma^2 \mathrm{I})$, is equivalent to noisifying edges because label $y \in
\{-1,1\}$ and the pdf of the Gaussian mechanism is invariant by
multiplication by $y$.\\

The key quantity to prove the Theorem is, for any noisified rado $\ve{\rado}^+_{j}
\defeq (1/2) \cdot \sum_i {(\sigma_{ji} + y_i)\ve{x}^+_i}$, the
support $m_j
\defeq |\{ i : \sigma_{ji} = y_i\}|$ of the rado. We also
renormalize the leveraging coefficient in \radoboost, replacing
eq. (\ref{defalpha}) in \radoboost~pseudocode by:
\begin{eqnarray}
\alpha_{t} & \leftarrow & \frac{1}{2 \kappa \rado_{*\iota(t)}}
\log \frac{1 + r_t}{1 -
  r_t}\:\:,\label{defalphakappa}
\end{eqnarray}
for some fixed $\kappa \geq 1$. \\

\noindent We now embark in the proof of Theorem \ref{thm_random_gau}. Lemma 2 in (\cite{nnARj}) yields
\begin{eqnarray}
\exp\left( -\ve{\theta}_T^\top\ve{\rado}_j\right)
 & = & \exp\left( -\ve{\theta}_T^\top\ve{\rado}^+_j\right)
 \cdot \exp\left(\frac{1}{2} \cdot \ve{\theta}_T^\top \sum_i {(\sigma_{ji} + y_i) \ve{x}^r_i}\right)\nonumber\\
& 
\leq & \left(\prod_{t=1}^{T} {\sqrt{1-r^2_{t}}} \cdot n w_{(T+1)j}\right)^{\frac{1}{\kappa}} \cdot \exp\left(\frac{1}{2} \cdot \ve{\theta}_T^\top \sum_i {(\sigma_{ji} + y_i) \ve{x}^r_i}\right)
\:\:, \forall j\in [n]\:\:.
\end{eqnarray}
Averaging over $j\in [n]$ yields:
\begin{eqnarray}
\explossrado({\mathcal{S}}, \ve{\theta}_T, \mathcal{U}) & \leq &
\left(\prod_{t=1}^{T} {\sqrt{1-r^2_{t}}}\right)^{\frac{1}{\kappa}}
\cdot \sum_{j=1}^{n} n^{\frac{1}{\kappa}-1} w^{\frac{1}{\kappa}}_{(T+1)j} \cdot \exp\left(\frac{1}{2} \cdot \ve{\theta}_T^\top \sum_i {(\sigma_{ji} + y_i) \ve{x}^r_i}\right) \nonumber\\
 & \leq & \underbrace{\exp\left(-\frac{1}{2\kappa}\sum_t r_t^2\right)}_{A}
 \cdot \underbrace{\sum_{j=1}^{n} \tilde{w}_{(T+1)j} \cdot
   \exp\left(\frac{1}{2} \cdot \ve{\theta}_T^\top \sum_i {(\sigma_{ji} + y_i)
       \ve{x}^r_i}\right)}_{B}\label{sepb}\:\:,
\end{eqnarray}
with $\tilde{w}_{(T+1)j} \defeq n^{\frac{1}{\kappa}-1}
w^{\frac{1}{\kappa}}_{(T+1)j}$. The right-hand side of ineq. (\ref{sepb}) multiplies two separate
quantities, $A$ which quantifies the performances of $\ve{\theta}_T$ in
\radoboost~on the set of noisy rados on which it was trained, and $B$ which is an expectation, computed over $\ve{w}_T$,
of the agreements between $\ve{\theta}_T$ and the noisy part of the
rados. When rados are noise-free and $\kappa \geq 1$, we have $\ve{x}_i^r = \ve{0}$,
$\forall i$ and 
\begin{eqnarray}
\sum_{j=1}^{n} \tilde{w}_{(T+1)j} & = & n^{\frac{1}{\kappa}}\cdot \frac{1}{n}
\sum_{j=1}^{n} w^{\frac{1}{\kappa}}_{(T+1)j}\nonumber\\
 & \leq & n^{\frac{1}{\kappa}}\cdot\left(\frac{1}{n}
\sum_{j=1}^{n} w_{(T+1)j}\right)^{\frac{1}{\kappa}}\nonumber\\ 
 & & = n^{\frac{1}{\kappa}}\cdot n^{-\frac{1}{k}} = 1\label{peqsigma}
\end{eqnarray} 
because of the
concavity of $x^{1/\kappa}$, and so we return to the noise-free rado boosting
bound with ``penalty $1/\kappa$'' for renormalizing the leveraging
coefficients in \radoboost~(this proves ineq. (\ref{Qbound2})). Assuming $\ve{\theta}_T$ output by
\radoboost, we obtain, $\forall {\mathcal{S}},
  \mathcal{U}$ such that support of all $n$ rados is of the same size,
\textit{i.e.} $m_j = m_*, \forall j \in [n]$,
\begin{eqnarray}
\lefteqn{\loglossrado({\mathcal{S}}, \ve{\theta}_T,
  \mathcal{U})}\nonumber\\
 & = & \log(2) + \frac{1}{m}
\log \explossrado({\mathcal{S}}, \ve{\theta}_T, \mathcal{U}) \nonumber\\
 &  \leq & \log(2) - \frac{1}{2\kappa m}\sum_t r_t^2 + \frac{1}{m} \cdot \log \sum_{j=1}^{n} \tilde{w}_{(T+1)j} \cdot
   \exp\left(\frac{1}{2} \cdot \ve{\theta}_T^\top \sum_i {(\sigma_{ji} + y_i)
       \ve{x}^r_i}\right) \nonumber\\
 &  \leq & \log(2) - \frac{1}{2\kappa m}\sum_t r_t^2 + \frac{m_*}{m} \cdot \log \sum_{j=1}^{n} \tilde{w}_{(T+1)j} \cdot
   \exp\left(\frac{1}{2m_*}\cdot \ve{\theta}_T^\top \sum_i {(\sigma_{ji} + y_i)
       \ve{x}^r_i}\right) \nonumber\\
 & & =  \log(2) - \underbrace{\frac{1}{2\kappa m}\sum_t r_t^2}_{\defeq
 C} + \underbrace{\frac{m_*}{m} \cdot \log \sum_{j=1}^{n} \tilde{w}_{(T+1)j} \cdot
   \exp\left(\frac{\varsigma}{\sqrt{m_*}} \cdot \ve{\theta}_T^\top
     \sum_i {\frac{\sigma_{ji} + y_i}{2\varsigma\sqrt{m_*}}
       \ve{x}^r_i}\right)}_{\defeq D}
 \:\:.\label{llleq00}
\end{eqnarray}
We now study a sufficient condition for $C-D$ to be $\Omega ((1/m)
\sum_t r_t^2)$ with high probability over the noise mechanism, thereby ensuring a convergence
rate over \textit{non-noisy} rados that shall comply with the
noise-free bounds of ineq. (\ref{Qbound}), up to the
hidden factors. This shall be achieved through several Lemmata.
\begin{lemma}\label{partial3}
With probability $\geq 1-\uptau$ over the noise
mechanism we shall have:
\begin{eqnarray}
\left\|
     \sum_i {\frac{\sigma_{ji} + y_i}{2\varsigma\sqrt{m_*}}
       \ve{x}^r_i} \right\|_2 & \leq & 
   \sqrt{2\log\left(\frac{n}{\uptau}\right)} \:\:, \forall j \in [n]\:\:.\label{uptau1}
\end{eqnarray}
\end{lemma}
\begin{proof}
The Sudakov-Tsirelson
     inequality (\cite{blmCI}, Theorem 5.6) states that if $\ve{x} \sim {\mathcal{N}}(\ve{0},\mathrm{I}_d)$
     and $f(\bm{x}) : {\mathbb{R}}^d \rightarrow {\mathbb{R}}$ is
     $L$-Lipschitz, then
\begin{eqnarray}
\pr\left[f(\bm{x}) - \expect[f(\bm{x})]\geq t\right] & \leq &
\exp\left(-\frac{t^2}{2L^2}\right)\:\:.\label{stineq}
\end{eqnarray}
Since function
$f(\bm{x}) \defeq \|\bm{x}\|_2$ is 1-Lipschitz by the triangle
inequality and $\sum_i {\frac{\sigma_{ji} + y_i}{2\varsigma\sqrt{m_*}}
       \ve{x}^r_i}$ is a standard Gaussian random because the
     $\ve{x}^r_i$ are sampled independently, ineq. (\ref{stineq})
     yields that we shall have simultaneously over the randomized part
     of the rados, with probability $\geq
     1 - \uptau$,
\begin{eqnarray}
\left\|
     \sum_i {\frac{\sigma_{ji} + y_i}{2\varsigma\sqrt{m_*}}
       \ve{x}^r_i} \right\|_2 & \leq & 
   \sqrt{2\log\left(\frac{n}{\uptau}\right)} \:\:, \forall j \in [n],\nonumber
\end{eqnarray}
which proves the Lemma.
\end{proof}
\begin{lemma}\label{partial2}
Assume $\ve{\theta}_T \in {\mathcal{B}}(0,r_\theta)$ for some
$r_\theta>0$. Then with probability $\geq 1-\uptau$ over the noise
mechanism we shall have
\begin{eqnarray}
D & \leq & \frac{\varsigma
     r_\theta}{m}
   \sqrt{2m_* \log\left(\frac{n}{\uptau}\right)} \:\:.\label{boundD}
\end{eqnarray}
\end{lemma}
\begin{proof}
We use Lemma \ref{partial3}. Cauchy-Schwartz inequality implies 
\begin{eqnarray}
\ve{\theta}_T^\top
     \sum_i {\frac{\sigma_{ji} + y_i}{2\varsigma\sqrt{m_*}}
       \ve{x}^r_i} & \leq & \|\ve{\theta}_T\|_2 \cdot \left\| \sum_i {\frac{\sigma_{ji} + y_i}{2\varsigma\sqrt{m_*}}
       \ve{x}^r_i} \right\|_2\nonumber\\
 & \leq & 
     r_\theta
   \sqrt{2\log\left(\frac{n}{\uptau}\right)} \:\:, \forall j \in [n]\:\:.
\end{eqnarray}
We thus get in this case
\begin{eqnarray}
D & \leq & \frac{\varsigma
     r_\theta}{m}
   \sqrt{2m_* \log\left(\frac{n}{\uptau}\right)} + \frac{m_*}{m} \cdot \log
   \sum_{j=1}^{n} \tilde{w}_{(T+1)j}\nonumber\\
 & \leq & \frac{\varsigma
     r_\theta}{m}
   \sqrt{2m_* \log\left(\frac{n}{\uptau}\right)} \:\:.\label{llD}
\end{eqnarray}
because of ineq. (\ref{peqsigma}).
\end{proof}
We now prove a specific $r_\theta > 0$ which makes use of the
concentration of the randomized part of rados in Lemma \ref{partial3}.
\begin{lemma}\label{partial1}
Suppose there exists $\mu > \mu' > 0$ such that it simultaneously holds:
\begin{eqnarray}
\mu & \leq & \frac{\min_{k} \max_j |\rado_{jk}|}{m_*}\:\:, \label{pineq}\\
\mu' & \leq & \mu - \varsigma
   \sqrt{\frac{1}{m_*} \log\left(\frac{n}{\uptau}\right)}\:\:, \label{secineq}
\end{eqnarray}
where $\rado_{jk} = (1/2) \sum_i (\sigma_{ji} + y_i) x_{ik}$ is the
non-noisy part of rado $\ve{\rado}^+_j$. Assume the existence of $\uprho>0$ such that the weak learner \weak~in \radoboost~is
$\uplambda_p$-prudential for
\begin{eqnarray}
\uplambda_p & = & 1 - \frac{2}{\sqrt{1-\uprho} \kappa \mu' m_*} \:\:.\label{defRT}
\end{eqnarray}
Then probability $\geq 1-\uptau$ over the noise
mechanism we shall have
\begin{eqnarray}
\|\ve{\theta}_T\|_2 & \leq & (1 - \uprho) \sum_t r^2_t\:\:.
\end{eqnarray}
\end{lemma}
\textbf{Remarks}: notice that ineq. (\ref{pineq}) is equivalent to
saying that each coordinate $k$ has at east one non-zero entry in the
noise-free part of the rados. Unless coordinate $k$ is zero for all
examples --- in which case we can just discard this feature ---, this
assumption is easy to satisfy. 
\begin{proof}
We have
\begin{eqnarray}
\|\ve{\theta}_T\|_2 - \sum_t r^2_t& = & \sum_t \frac{1}{4 \kappa^2 \rado^2_{*\iota(t)}} \log^2
  \frac{1 + r_t}{1 - r_t} - r_t^2 \:\:.
\end{eqnarray}
Assuming the existence of $z>0$ such that $2 \kappa \rado_{*\iota(t)} \geq z,
\forall t$, and using the
fact that 
\begin{eqnarray}
\log^2\frac{1+x}{1-x} \leq \frac{4x^2}{(1-|x|)^2}\:\:, \forall x \in (0,1)\:\:,
\end{eqnarray} 
we shall have
\begin{eqnarray}
\sum_t \frac{1}{4 \kappa^2\rado^2_{*\iota(t)}} \log^2
  \frac{1 + r_t}{1 - r_t} - r_t^2 & \leq & \sum_t \frac{1}{z^2} \log^2
  \frac{1 + r_t}{1 - r_t} - r_t^2 \nonumber\\
 & \leq & \sum_t \frac{1}{z^2} \cdot \frac{4r^2_t}{(1 - |r_t|)^2} - r_t^2 \nonumber\\
 &  & = \sum_t r_t^2 \left(\frac{4 - z^2(1 - |r_t|)^2}{z^2(1 - |r_t|)^2}\right) \nonumber\\
 & \leq & - \uprho \sum_t r^2_t\:\:,
\end{eqnarray}
as long as 
\begin{eqnarray}
|r_t| & \leq & 1 - \frac{2}{\sqrt{1-\uprho} z} \:\:, \forall t\:\:,
\end{eqnarray}
where $\uprho \in (0,1)$. Since $\rado_{*\iota(t)} \geq \min_k
\max_j |\rado^+_{j k}|$, we can fix $z_* = 2 \kappa \min_k \max_k
|\rado^+_{j k}|$, but recall that $\rado^+_{j k}$ sums a random Gaussian
part and a non random part. Ineq. (\ref{uptau1}) tells us that with
high probability, the magnitude of the random part will satisfy 
\begin{eqnarray}
\sum_i {(\sigma_{ji} + y_i)
       \ve{x}^r_i} & \leq & \varsigma
   \sqrt{2m_* \log\left(\frac{n}{\uptau}\right)} \:\:, \forall j\in [n]\:\:.
\end{eqnarray}
Thus, we shall have in this case, using ineqs. (\ref{pineq},
\ref{secineq}) and given Lemma \ref{partial3}:
\begin{eqnarray*}
\min_k \max
|\rado^+_{j k}| & \geq & \left(\mu - \varsigma
   \sqrt{\frac{1}{m_*} \log\left(\frac{n}{\uptau}\right)} \right)
 \cdot m_* \nonumber\\
 & \geq & \mu' m_* \:\:,\nonumber
\end{eqnarray*}
and we get the statement of the Lemma.
\end{proof}
We now return to ineq. (\ref{llleq00}), and use Lemmata
\ref{partial3}, \ref{partial2}
and \ref{partial1}, and obtain that with probability $\geq 1 - \uptau$, a sufficiently
prudential weak learner shall imply:
\begin{eqnarray}
\lefteqn{\loglossrado({\mathcal{S}}, \ve{\theta}_T,
  \mathcal{U})}\nonumber\\
 & \leq & \log(2) - \frac{1}{2\kappa m}\sum_t r_t^2 + \frac{m_*}{m} \cdot \log \sum_{j=1}^{n} \tilde{w}_{(T+1)j} \cdot
   \exp\left(\frac{\varsigma}{\sqrt{m_*}} \cdot \ve{\theta}_T^\top
     \sum_i {\frac{\sigma_{ji} + y_i}{2\varsigma\sqrt{m_*}}
       \ve{x}^r_i}\right)\nonumber\\
 & \leq & \log(2) - \frac{1}{m}\cdot \underbrace{\left(\frac{1}{2\kappa} -
  (1-\uprho) 
     \cdot \varsigma
   \sqrt{2m_* \log\left(\frac{n}{\uptau}\right)}\right)}_{\defeq E} \sum_t r_t^2
 \:\:.\label{llleq1}
\end{eqnarray}
We want $E \geq 1/(4\kappa)$. Equivalently, we want
\begin{eqnarray}
1-\uprho & \leq & \frac{1}{4\kappa \varsigma
   \sqrt{2m_* \log\left(\frac{n}{\uptau}\right)}}\:\:,
\end{eqnarray}
and for the prudential weak learner to exist, we also need
\begin{eqnarray}
1-\uprho & > & \frac{4}{\kappa^2 \mu'^2 m_*^2} \:\:.
\end{eqnarray}
Assuming ineqs (\ref{pineq}) and (\ref{secineq}), we thus get that if
\begin{eqnarray}
\kappa & \geq & \frac{4 \varsigma}{\mu'^2 m_*^\frac{3}{2}}
   \sqrt{2\log\left(\frac{n}{\uptau}\right)}\:\:,
\end{eqnarray}
then there exists a prudential weak learner for which, with
probability $\geq 1 - \uptau$ over the noise mechanism, we shall have
after $T$ rounds of boosting of \radoboost, using the prudential weak
learner and renormalizing the leveraging coefficients by $\kappa$ as
in (\ref{defalphakappa}),
\begin{eqnarray}
\loglossrado({\mathcal{S}}, \ve{\theta}_T,
  \mathcal{U}) & \leq & \log(2) - \frac{1}{4\kappa m} \sum_t
  r_t^2\:\:, \label{llleqlast22}
\end{eqnarray}
which proves Theorem \ref{thm_random_gau}.
Notice that the constraint $\kappa \geq 1$ can easily be enforced by
picking $\mu'$ sufficiently small.\\

\noindent \textbf{Remarks}: we finish by emphasizing the fact that ineq. (\ref{llleqlast}) is
computed over \textit{non-noisy} rados. It is not hard to see that
ineqs (\ref{pineq}) and (\ref{secineq}) shall be all the easier to
meet as $m_*$ is large compared to $\log n$, $\log(1/\uptau)$ and $\varsigma$. So,
provided rados have a sufficiently large support, the convergence rate
of
the logistic rado-risk of
\radoboost~over the non noisy rados may compete, up to a small constant factor, with the one that
would be achieved by training \radoboost~over \textit{non-noisy} rados. 

\subsection{Proof of Lemma \ref{lem_algebraic}}\label{proof_lem_algebraic}

\begin{figure}[t]
\vskip 0.2in
\begin{center}
\centerline{\includegraphics[width=0.7\columnwidth]{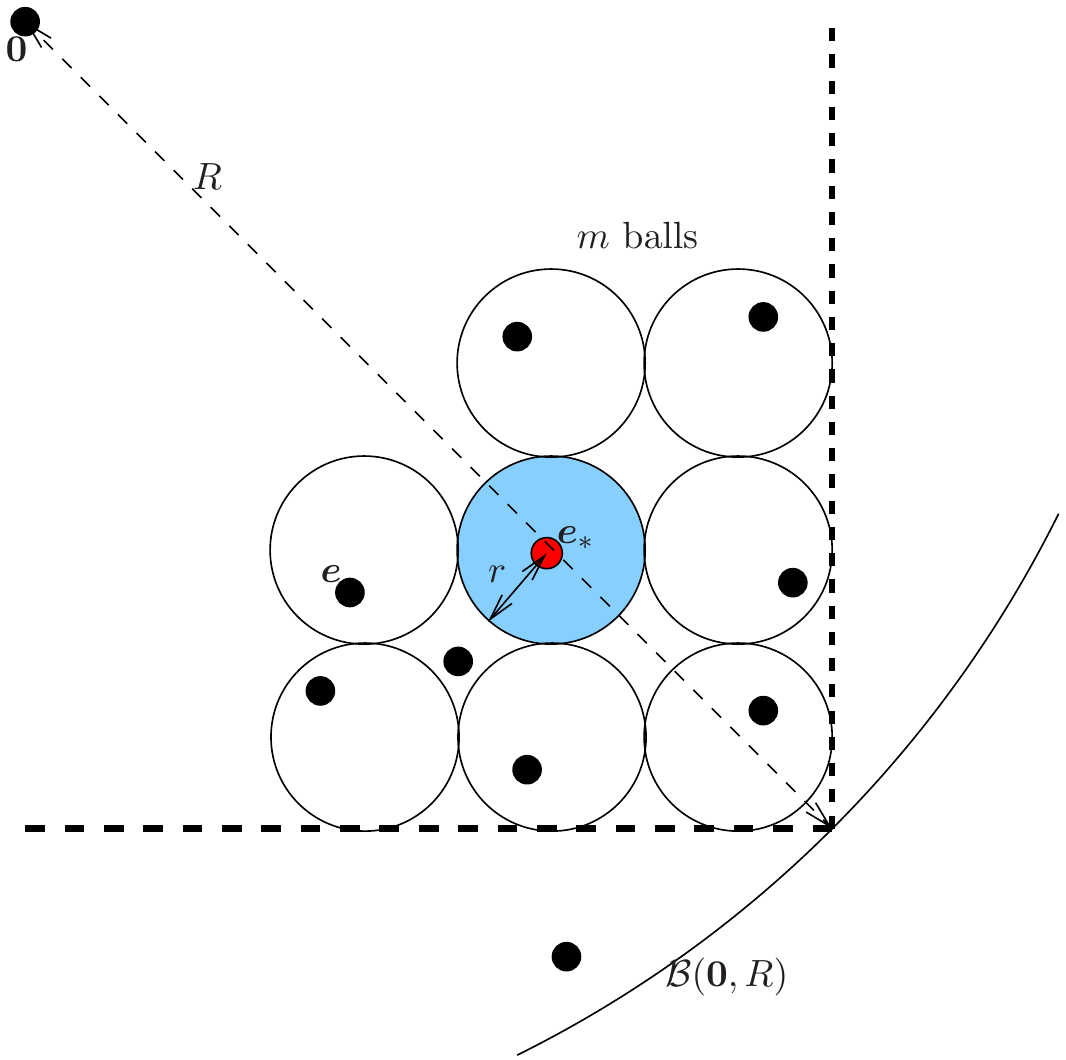}}
\caption{Construction for the proof of Lemma
  \ref{lem_algebraic}. Black dots denote edge vectors from
  ${\mathcal{S}}$; at least one ball, in blue, contains no such edge vector.}
\label{f-const}
\end{center}
\end{figure} 

Consider first that $m\geq 2^d$. A simple proof of the Lemma consists in considering the largest
$d$-dim square, of edge length $\ell = 2R/\sqrt{d}$, shown with thick dashed
line in Figure \ref{f-const}. We then pack this square with $m+1$
spheres, as shown. Since the edge length is covered by $\lceil \log(m)
/ \log(d)\rceil$ diameters of these spheres, we obtain that the radius $r$ of each such sphere satisfies:
\begin{eqnarray}
r & = & \frac{2R}{\sqrt{d} \cdot \lceil \frac{\log(m+1)}{\log d} \rceil} \nonumber\\
 & \geq & \frac{R \log d}{2\sqrt{d}\log (m+1)}\:\:,
\end{eqnarray}
because $m\geq 2^d > d$. Because of the construction, at least one of these spheres does not
contain an edge vector from ${\mathcal{C}}(\matrice{E})$ and is thus
empty. Consider one such empty sphere whose center $\ve{e}_* $ is the closest to
$\ve{0}$, as shown in Figure \ref{f-const}, and consider one adjacent
sphere, located no farther\footnote{If no such sphere exists, we can pick
$\ve{e}_* = \ve{0}$, the center of a sphere ${\mathcal{B}}(\ve{0},r)$
which contains no example from ${\mathcal{S}}$. In this case, there is
no need to remove any example from ${\mathcal{S}}$: the proof still holds by adding example
$(\ve{0}, y)$ to ${\mathcal{S}}$, to create ${\mathcal{S}}'$.}, with one edge vector $\ve{e}= y \ve{x}$ from
${\mathcal{C}}(\matrice{E})$ inside, with $(\ve{x}, y)\in
{\mathcal{S}}$, where ${\mathcal{S}}$ generates $\matrice{$\Pi$}$. We
create ${\mathcal{S}}'$ out of ${\mathcal{S}}$ by replacing $(\ve{x},
y)$ by two examples, $(y \ve{e}_*, y)$ and $(\ve{e}- y \ve{e}_*, y)$. It is
worthwhile remarking that
\begin{eqnarray}
{\mathcal{C}}(\matrice{E}') & \subset & {\mathcal{B}}(\ve{0}, R)
\end{eqnarray}
by construction, and furthermore any rado that
can be created from ${\mathcal{S}}$ can also be created from
${\mathcal{S}}'$. Hence, any $\matrice{$\Pi$}$ defined over
  ${\mathcal{S}}$ can also be obtained from ${\mathcal{S}}'$. There
  remains to remark that, by construction, $\ve{e}_*$ is distant from every edge vector of ${\mathcal{S}}$ from
  at least $r$, and so:
\begin{eqnarray}
D_{\mathrm{H}}(\matrice{E}, \matrice{E}') & = & \Omega\left(
  \frac{R \log d}{\sqrt{d} \log m} \right)\:\:;
\end{eqnarray}
this proves Lemma \ref{lem_algebraic} when $m\geq 2^d$. When $m < 2^d$, the
construction of Figure \ref{f-const} can still be done but with 
larger balls, for which
\begin{eqnarray}
r & = & \frac{R}{2\sqrt{d}}\:\:.
\end{eqnarray}
Picking as $\ve{e}_* $ the center of any of these empty balls, we
obtain
\begin{eqnarray}
D_{\mathrm{H}}(\matrice{E}, \matrice{E}') & \geq & 
  \frac{R}{2\sqrt{d}} \:\:,
\end{eqnarray}
as claimed.

\subsection{Proof of Lemma \ref{lem_comp1}}\label{proof_lem_comp1}

We make a reduction from
the X3C3 (\cite{prOT}) problem whose instance is a set $S \defeq \{s_1, s_2,
..., s_n\}$ and a set of $3$-subsets of $S$, $C \defeq \{c_1, c_2, ...,
c_{d}\}$, and an integer $m$. Each element of $S$ belongs to exactly
three subsets of $C$. The question is whether there exists a
cover of $S$ using at most $m$ elements from $C$. The reduction is the
following:
\begin{itemize}
\item to each feature corresponds an element of $C$;
\item to each element $s_j$ of $S$ we associate a boolean rado
  $\ve{\rado}_{j}$ which is $1$ in coordinate $k$
  iff $s_j \in c_k$, and zero otherwise:
\begin{eqnarray}
\ve{\rado}_{j} & = & \ve{1}_{\{k : s_j \in c_k\}}\:\:.
\end{eqnarray}
($\ve{1}_{\mathcal{I}}$ is ``1'' in coordinate $i_k$
for $k\in {\mathcal{I}}$, and zero everywhere
else)
\item The number of examples is $m$;
\item Parameters $r$ and $\ell$ are fixed as follows:
\begin{itemize}
\item if $p\neq 0$, the value of $r$ is
  $2^{1/p}$. We also fix $\ell = \epsilon$-machine, where
  $\epsilon$-machine is the smallest $\epsilon$ such that $1-\epsilon
  < 1$ in machine encoding;
\item else if $p=0$, then $r=2$ and $\ell = 1$;
\end{itemize}
\end{itemize}
Let us number the constraints of Sparse-Approximation, so that we
want:
\begin{eqnarray}
\|\ve{x}_i\|_p & \leq & \ell\:\:, \forall i\in [m]\:\:,
\:\:(\mbox{Sparse examples}) \label{psparse1}\\
\|\ve{\rado}_{j} - \ve{\rado}_{\ve{\sigma}_j}\|_p & \leq &
r\:\:, \forall j \in [n]\:\:. \label{pqual1}\:\:(\mbox{Rado approximation})
\end{eqnarray}
Suppose there exists a solution to X3C3 with $m$ subsets of $C$, $C^* \defeq \{c^*_{k_1},
c^*_{k_2}, ..., c^*_{k_m}\}$. Create $m$
positive examples ($y_i = 1$) whose observation is $\ve{x}_{i}
\defeq \ve{1}_{\{k_i\}}$ (the all-0 vector with only one ``1'' in
coordinate $k_i$). Clearly, the sparsity constraint on examples (\ref{psparse1})
is satisfied. We craft the rados following $n$ Rademacher assignations, where $\ve{\sigma}_i$ is $+1$
only for $\ve{x}_{k_i}$, and $-1$ otherwise. Notice that 
\begin{eqnarray}
\ve{\rado}_{j} -\ve{\rado}_{\ve{\sigma}_j} & = & \ve{1}_{\{k : s_j \in c_k\}} - \ve{1}_{\{k_i, s_j \in
   c^*_{k_i}\}}\\
 & = & \ve{1}_{\{k : s_j \in c_k \wedge c_k \not\in C^*\}}\:\:.
\end{eqnarray}
It comes 
\begin{eqnarray}
\|\ve{\rado}_{j} -\ve{\rado}_{\ve{\sigma}_j}\|_p & \leq &
2^{1/p} \defeq r\:\:, \forall j \in [n]\:\:,
\end{eqnarray}
if $p\neq 0$, and
\begin{eqnarray}
\|\ve{\rado}_{j} -\ve{\rado}_{\ve{\sigma}_j}\|_0 & \leq &
2 \defeq r\:\:, \forall j \in [n]
\end{eqnarray}
otherwise, since each element of $S$ belongs to three sets in $C$. Therefore,
there exists a solution to Sparse-Approximation.\\

Now, suppose
there exists a solution to Sparse-Approximation. Remark that we can remove wlog any
example having
null observation as this does not change the feasibility of the solution.
Consider the case where $p\neq 0$. The Rado approximation constraint (\ref{pqual1})
of Sparse-Approximation makes that the following property (P) is satisfied:
\begin{itemize}
\item [(P)] for each $j\in [n]$, there exists $i\in [m]$ and feature $k\in
[d]$ such that $\ve{\rado}_{\ve{\sigma}_j}$ and example
$\ve{x}_i$ have their coordinate $k$ non-zero, and furthermore the
coordinate in $\ve{x}_i$ has magnitude exactly $\epsilon$: it cannot
be less otherwise (\ref{pqual1}) is violated, and it cannot be more
otherwise (\ref{psparse1}) is violated. Hence, each of these $\ve{x}_i$
have exactly one non-zero coordinate.
\end{itemize}
Because property (P) holds for all rados, we see that the
corresponding indexes in the $\ve{x}_i$ (the corresponding non-zero
coordinates for features for which (P) holds; there cannot be more than
$m$) define a solution to X3C3. The case $p = 0$ is easier as
(\ref{psparse1}) enforces the number of non-zero coordinates in each
observation to be at most one, and therefore exactly one since there
is no null observation.

We finally note that Sparse-Approximation trivially belongs to NP, so
it is actually NP-Complete. 

\subsection{Proof of Lemma \ref{lem_comp2}}\label{proof_lem_comp2}

We make the same reduction as for Sparse-Approximation. The set of
examples ${\mathcal{S}}$ consists of
all canonical basis vectors, associated to positive class.

\section{Appendix --- Experiments}\label{app_exp_expes}

\subsection{Supplementary experiments to Table \ref{tc1_full}}\label{exp_tc1}

\begin{table}[t]
\begin{center}
{\scriptsize
\begin{tabular}{|crrc|rcrcrc|c|c|}
\hline \hline
Domain & \multicolumn{1}{c}{$m$} & \multicolumn{1}{c}{$d$}  &
100$\sigma$ &
\multicolumn{6}{|c|}{err$\pm\sigma$} & $p$ & $p'$ \\
 & & & & \adaboostSS$_*$ & $\searrow$ & \adaboostSSS$_*$ & $\searrow$ &
 \radoboost$_*$ & $\searrow$ & & \\ \hline 
Fertility & 100 & 9 & -- & 44.00$\pm$18.38 & \bY &
57.00$\pm$17.03 & \rN & 53.00$\pm$14.18 & --- & 0.28 & 0.42
\\  
Haberman & 306 & 3 & -- & 25.78$\pm$4.78 & \rN & 41.88$\pm$12.38 & \rN
& 25.77$\pm$6.04 & \bY & 0.98 
& $\varepsilon$\\
Transfusion & 748 & 4 & -- & 39.19$\pm$6.66 & \bY & 36.78$\pm$5.76 & \bY
& 36.65$\pm$5.74 & \bY & 0.04 
& 0.95\\
Banknote & 1 372 & 4 & -- & 2.70$\pm$1.38 & \bY &
2.70$\pm$1.38 & \rN & 13.93$\pm$3.68 & \bY
& $\varepsilon$ & $\varepsilon$\\
Breast wisc & 699 & 9 & -- & 2.86$\pm$1.90 & \bY &
4.43$\pm$2.07 & \rN & 3.58$\pm$1.69 & \bY &
0.24 & 0.14\\
Ionosphere & 351 & 33 & -- & 11.92$\pm$7.03 & \rN & 11.37$\pm$4.94 &
\bY &
17.07$\pm$9.26 & \rN & 0.05 & 0.03\\
Sonar & 208 & 60 & -- & 25.60$\pm$11.41 & \bY & 30.36$\pm$10.46 & \rN &
27.02$\pm$12.77 & \bY & 0.51 & 0.43\\
Wine-red$^*$ & 1 599 & 11 & 1 & 26.33$\pm$4.00 & \rN & 25.95$\pm$4.01
& \bY &
27.70$\pm$3.39 & \bY & 0.05 & 
0.03\\
Abalone$^*$  & 4 177 & 8 & -- & 25.59$\pm$2.59 & \rN &
25.45$\pm$2.74   & \rN & 24.80$\pm$2.59 & \bY & 0.18 & 0.07\\
Wine-white$^*$ & 4 898 & 11 & 1 & 31.07$\pm$2.10 & \rN &
30.54$\pm$2.06 & \rN &
33.42$\pm$2.38 & \rN &$\varepsilon$ & 
$\varepsilon$\\
Magic$^*$ & 19 020 & 10 & -- & 21.18$\pm$1.16 & \rN & 21.23$\pm$1.34 & \rN &
22.90$\pm$2.19 & \rN & $\varepsilon$ & 
$\varepsilon$\\
EEG & 14 980 & 14 & 14 & 43.54$\pm$1.67 & \bY& 43.06$\pm$2.35 & \bY&
43.73$\pm$1.89 & \bY& 0.67 & 
0.09\\
Hardware$^*$ & 28 179 & 95 & -- & 3.01$\pm$0.27 & \bY& 
2.70$\pm$0.39 & \bY& 7.35$\pm$3.31 & \bY&
$\varepsilon$ & $\varepsilon$  \\
Twitter$^*$ & 583 250 & 77 & 44 & 6.08$\pm$0.15 & \bY& 
6.72$\pm$0.64 & \bY& 5.71$\pm$0.64 & \bY& 
0.07 & $\varepsilon$  \\
SuSy & 5 000 000 & 17 & -- & 28.17$\pm$0.03 & \rN & 27.92$\pm$1.40 
& \rN & 27.14$\pm$0.39 & \bY &
$\varepsilon$ & 0.13  \\
Higgs & 11 000 000 & 28 & -- & 46.20$\pm$0.05 & \rN & 47.68$\pm$0.55 &
\rN &
47.86$\pm$0.06 & --- & $\varepsilon$ & 0.34\\
\hline\hline
\end{tabular}
}
\end{center}
\caption{Comparison of \radoboost~to \adaboostSS~(\cite{ssIBj}) and
  \adaboostSS~trained with a random subset of training of the same size
  as ${\mathcal{S}}_*$ (\adaboostSSS). The symbol ``$_*$'' indicates algorithms
  are ran with the
  replacement of eq. (\ref{defrt}) for the normalized
  edge $r_t$. Conventions
  are the same as in Table \ref{tc1_full}. The symbols \bY,
  \rN, ---, respectively indicate whether the new version performs better than
  (resp. worse than, similarly to) the non-modified version.}
  \label{tccorr}
\end{table}

Table \ref{tccorr} is obtained under the same experimental setting as
that of Table \ref{tc1_full}, with an important modification in how the
normalized edge is computed. More specifically, the computation of
$r_t$ in Step 2.2 of \radoboost~(see (\ref{defMu})) is
completed by the following step:
\begin{eqnarray}
r_t & \leftarrow & \mathrm{sign}(r_t) \cdot \max\{0.1,
|r_t|\}\label{defrt}
\end{eqnarray} 
The same modification is also carried out in \adaboostSS~(\cite{ssIBj})
(Corollary 1). This aims to prevent the fact that domains with outlier
feature values could trick \adaboostSS~in picking the wrong sign for
$\alpha_t$ for a large number of iterations, due to values of $r_t$
with a very small magnitude (but with the wrong sign). Experiments
display that this corrects \adaboostSS's bad results on Twitter, but
on other domains like Fertility, Haberman,
Sonar, Abalone, the change happens to give worse results for
\adaboostSS~and/or \adaboostSSS. \radoboost's results, on the other hand, tend to improve
with sparse exceptions.

\subsection{Supplementary experiments to Section \ref{sradp} --- I / III}\label{exp_sradp}

Tables \ref{t-s52_1}, \ref{t-s52_2}, \ref{t-s52_3}, \ref{t-s52_4}
present results comparing \adaboostSS, \radoboost~with random rados
and \radoboost~with fixed support size rados ($m_*$). Unless otherwise stated in Tables, the following experimental
setup holds:
\begin{itemize}
\item \radoboost~is trained with $n = \min\{1000, \mbox{train fold
  size}/2\}$ rados;
\item \adaboostSS~is trained using the complete training fold;
\item for each standard deviation $\upsigma$, we generate 10 noisy
  domains; each is then processed following 10 folds stratified
cross-validation. Thus, each dot on the colored curves is the average
of ten experiments;
\item \radoboost~is trained with two types of rados: random rados as
  in Section \ref{exp_boost_rado} --- this gives the
  grey dashed curves ---, or rados with
  fixed support $m_*$ (noted $s$ on the plots) as in Subsection
  \ref{sbfdp} --- this gives the
  colored curves ---;
\end{itemize}

\begin{table}[t]
\begin{center}
\begin{tabular}{|r|c||c|}
\hline \hline
 & \weak~= Strong & \weak~= Median-prudential\\ \hline
\begin{turn}{90}Fertility\end{turn} & \includegraphics[width=0.40\columnwidth]{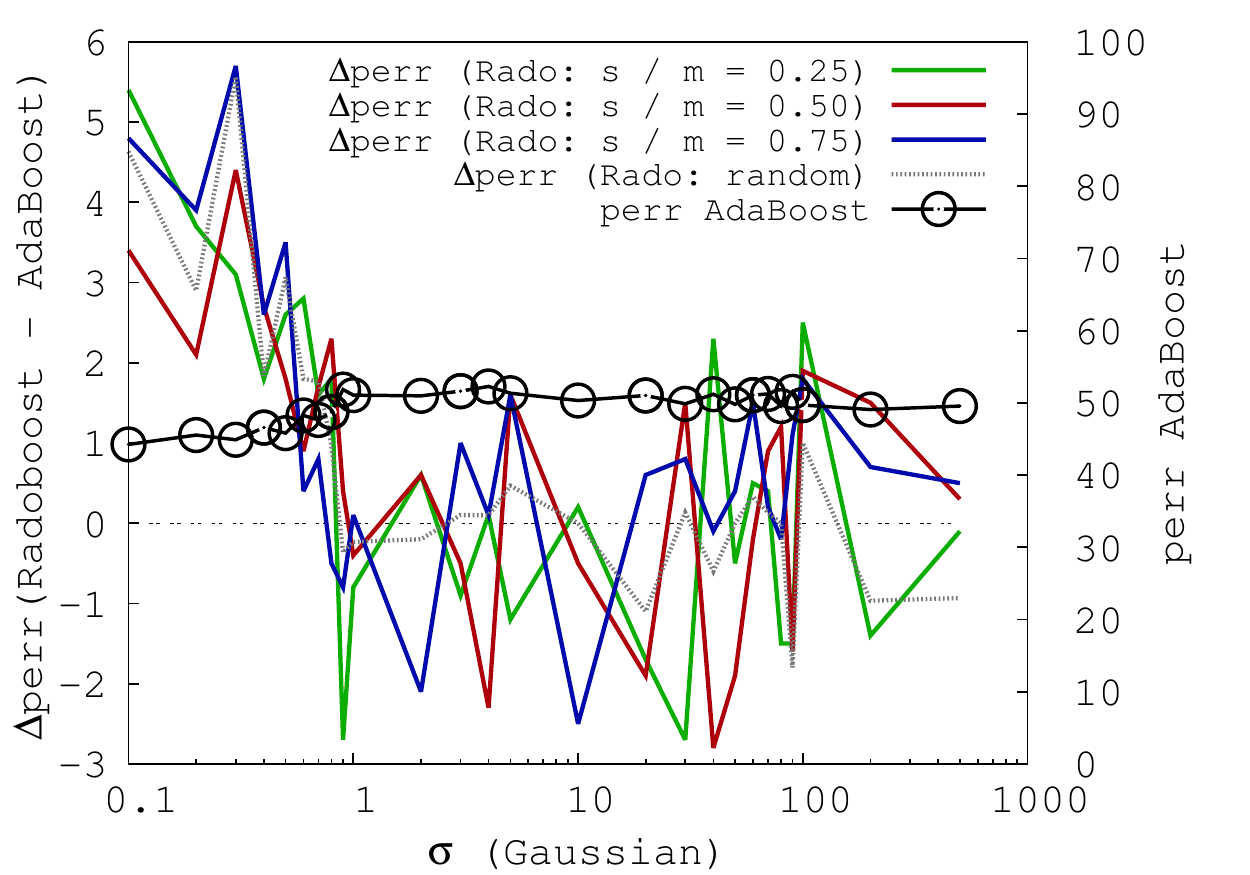}
& \includegraphics[width=0.40\columnwidth]{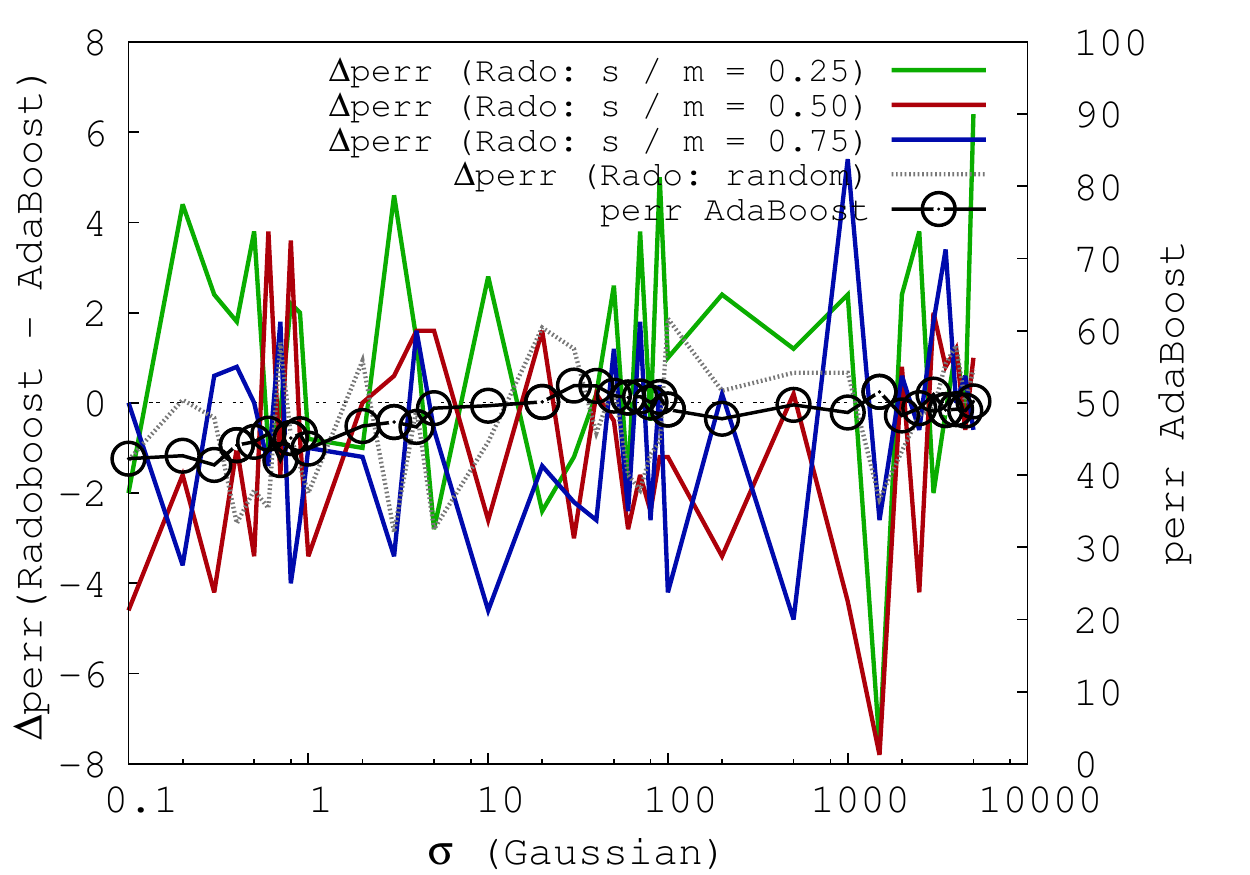}\\ \hline
\begin{turn}{90}Haberman\end{turn} & \includegraphics[width=0.40\columnwidth]{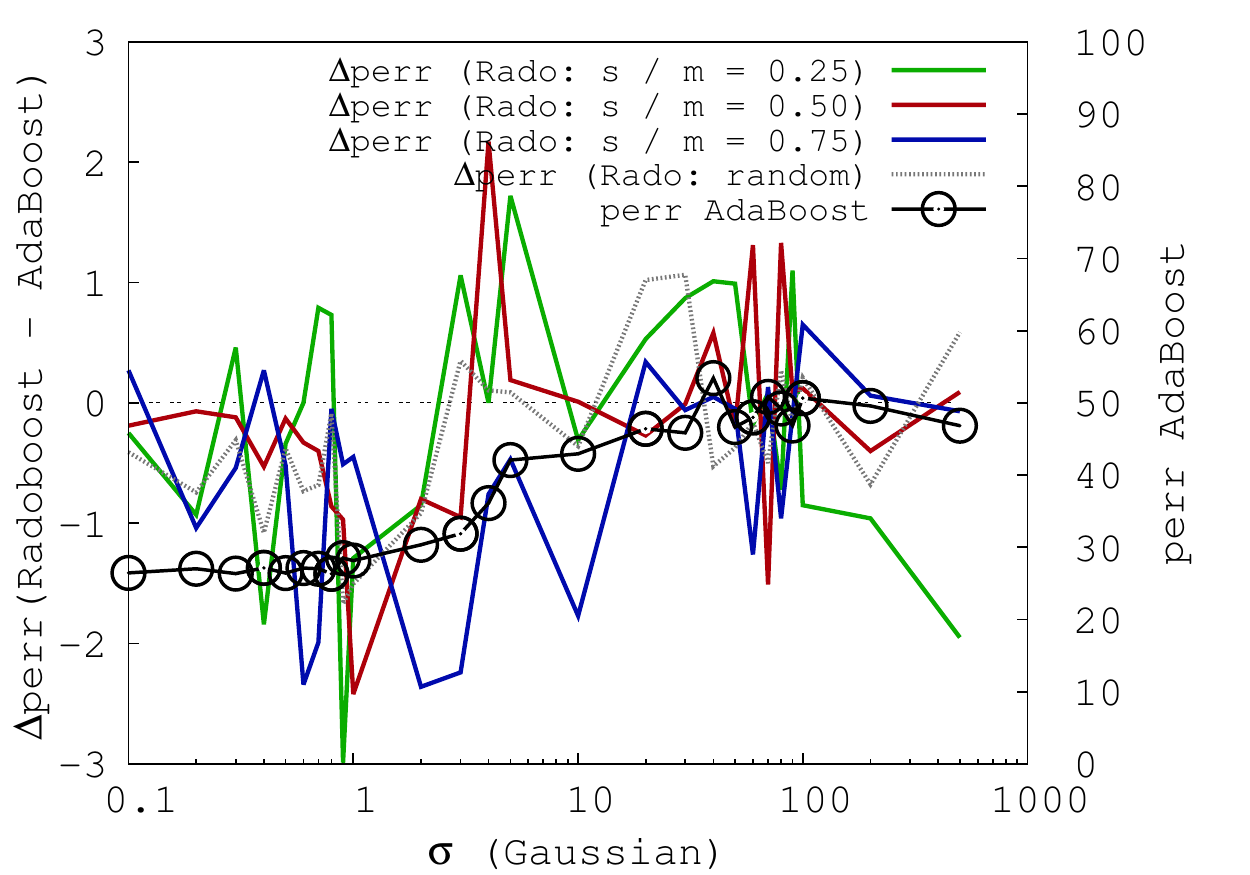}
& \includegraphics[width=0.40\columnwidth]{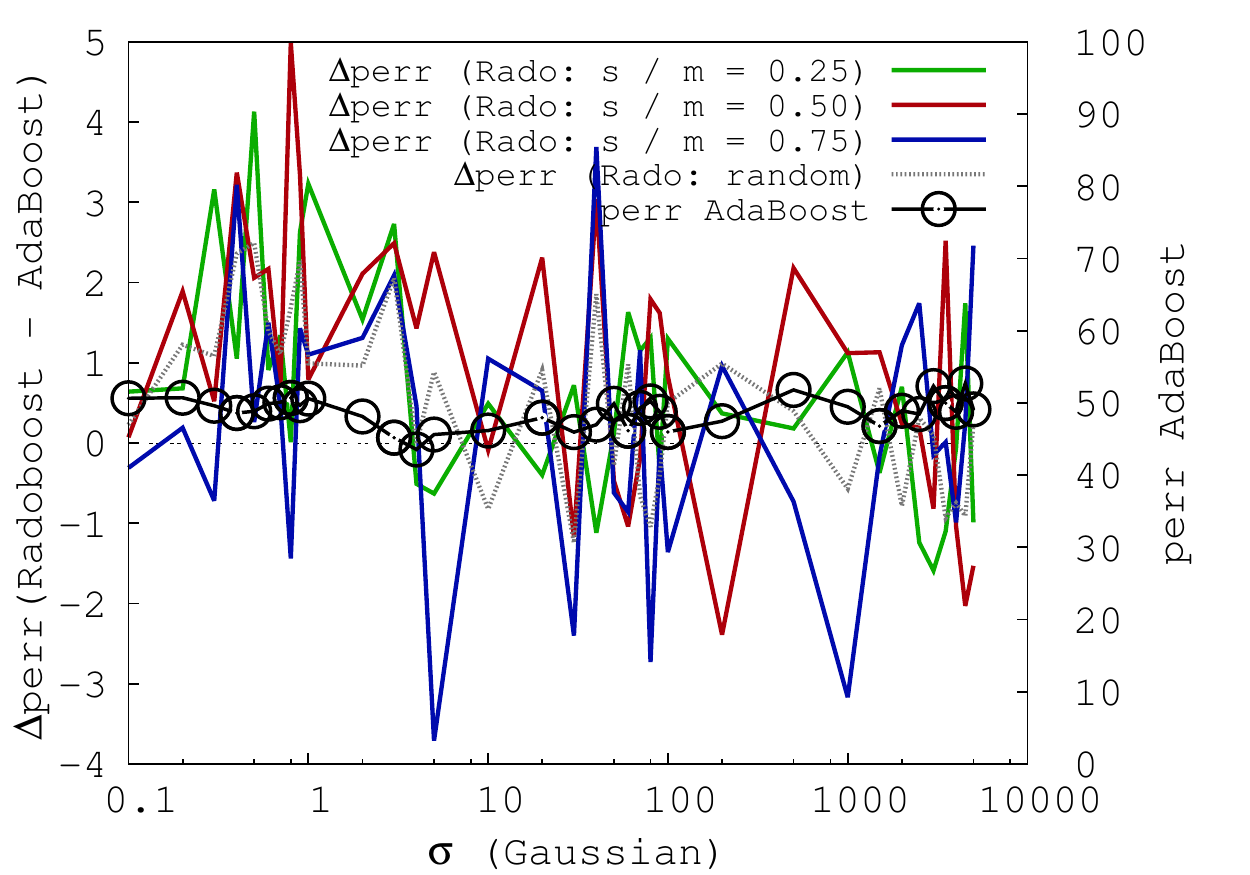}\\ \hline
\begin{turn}{90}Transfusion\end{turn} & \includegraphics[width=0.40\columnwidth]{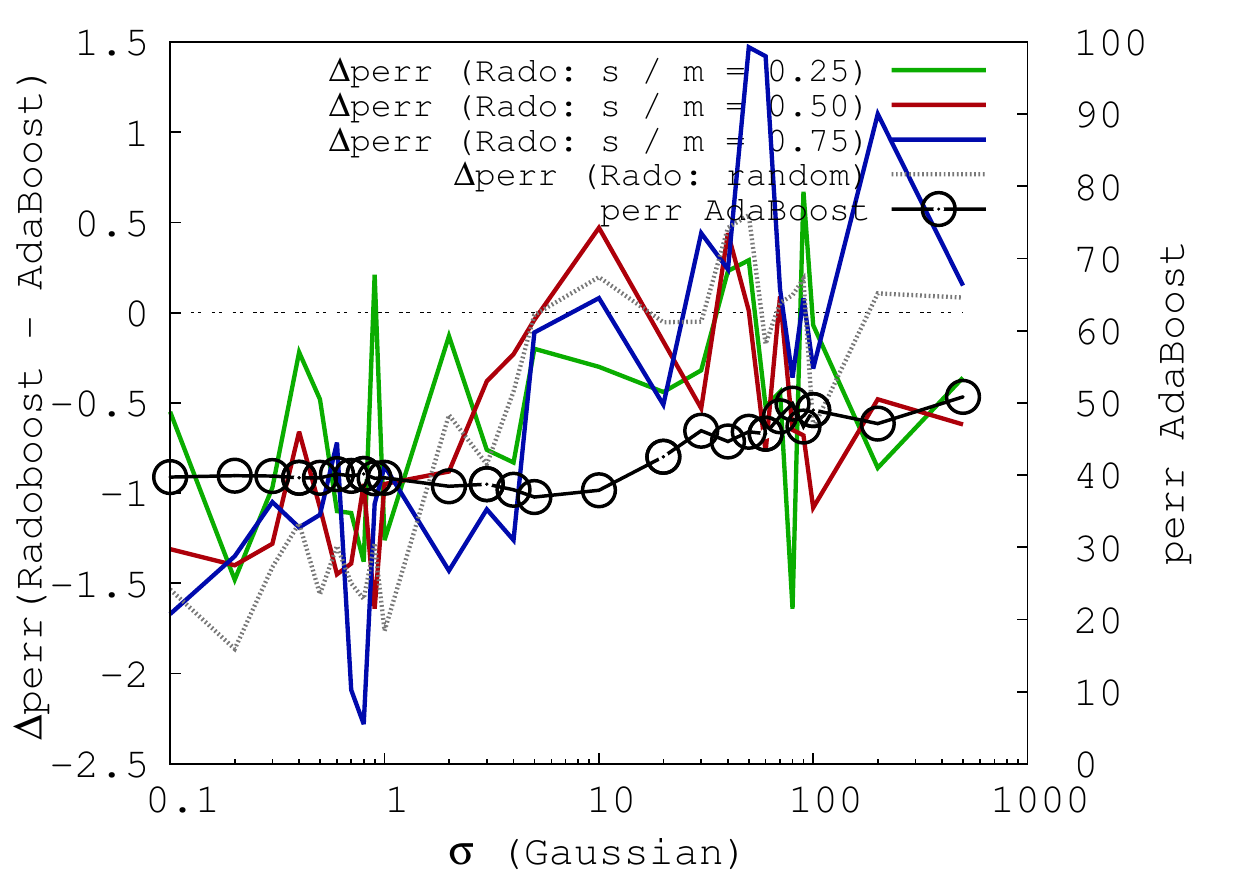}
& \includegraphics[width=0.40\columnwidth]{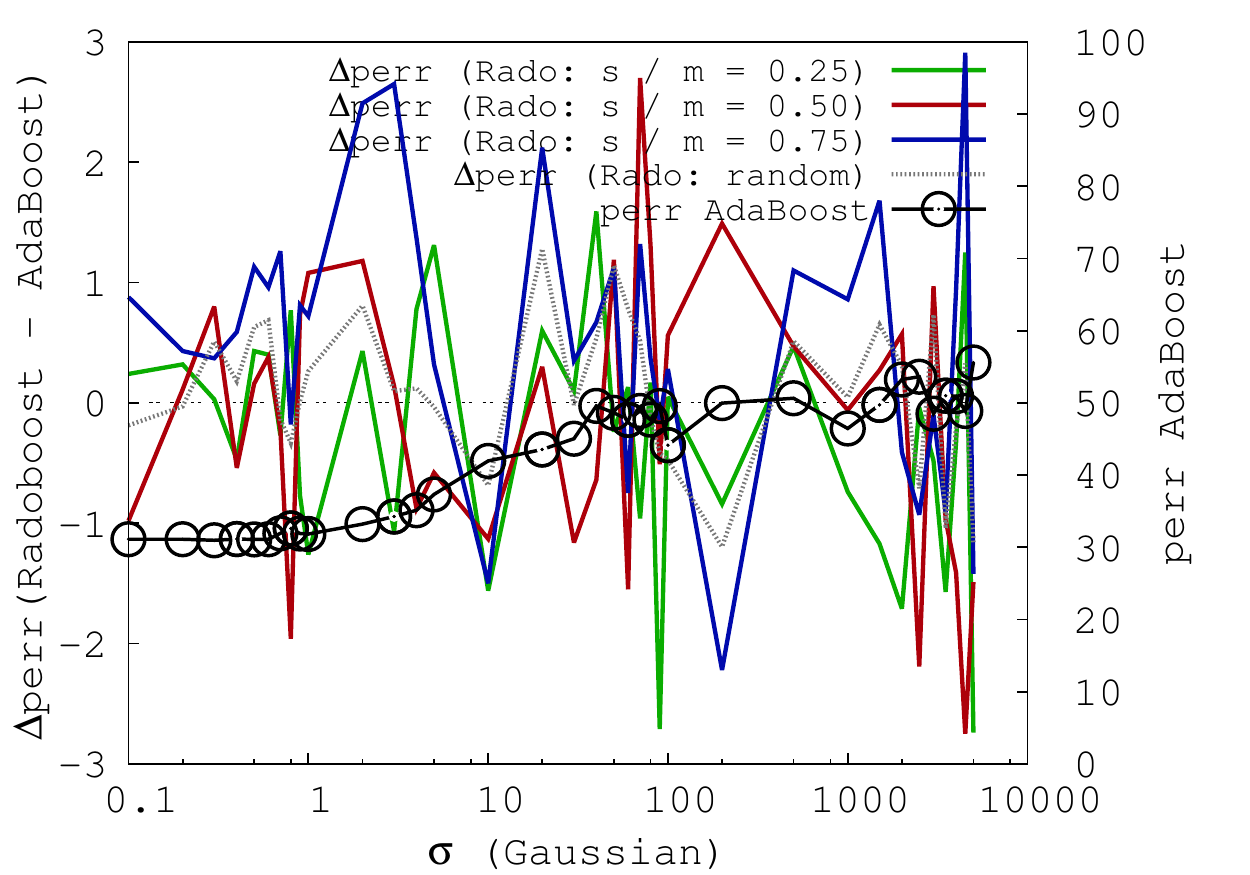}\\ \hline \hline
\end{tabular}
\end{center}
  \vspace{-0.5cm}
\caption{Learning from examples that have been noisified using the
  Gaussian mechanism ${\mathcal{N}}(\ve{0},
\upsigma^2 \mathrm{I})$ (See Section \ref{proof_thm_random_gau}), as a function of $\upsigma$. In each
plot, the \textbf{right}
axis gives \adaboostSS's (\cite{ssIBj}) test error, related to the big dotted
curve. All other curves are related to the \textbf{left} axis, which gives the
difference of test errors ($\Delta$perr) between \radoboost~and
\adaboostSS. The grey dashed
curve is for rados picked uniformly at random in $\Sigma_m$,
following Section \ref{sbur}.  The colored curves
(green, red, blue) correspond to rados with fixed support $s$ ($=m_*$) such that $s/m \in \{0.25, 0.5, 0.75\}$, generated with the mechanism of 
Section \ref{sbfdp}. $m$ refers to the size of a
training fold. Range of $\upsigma$ is not the same on the left and
right plots. The horizontal dashed black line indicates
$\Delta$perr = 0: colored lines below this line indicate runs of
\radoboost~that are better
than \adaboostSS's.}
  \label{t-s52_1}
\end{table}

\begin{table}[t]
\begin{center}
\begin{tabular}{|r|c||c|}
\hline \hline
 & \weak~= Strong & \weak~= Median-prudential\\ \hline
\begin{turn}{90}Banknote\end{turn} & \includegraphics[width=0.40\columnwidth]{Plots/Strong_Weak_Oracle/banknote_RadoBoostSupport-AdaBoost}
& \includegraphics[width=0.40\columnwidth]{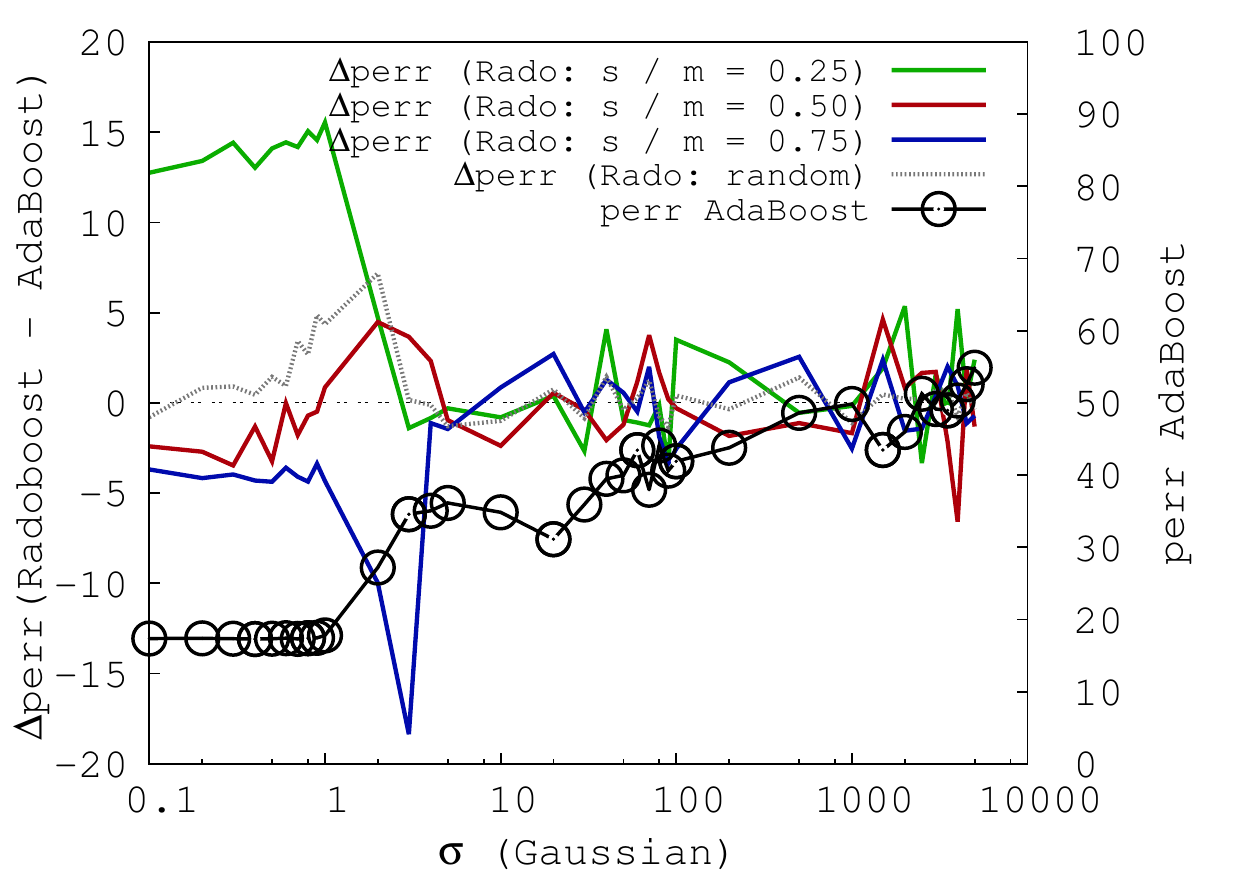}\\ \hline
\begin{turn}{90}Breastwisc\end{turn} & \includegraphics[width=0.40\columnwidth]{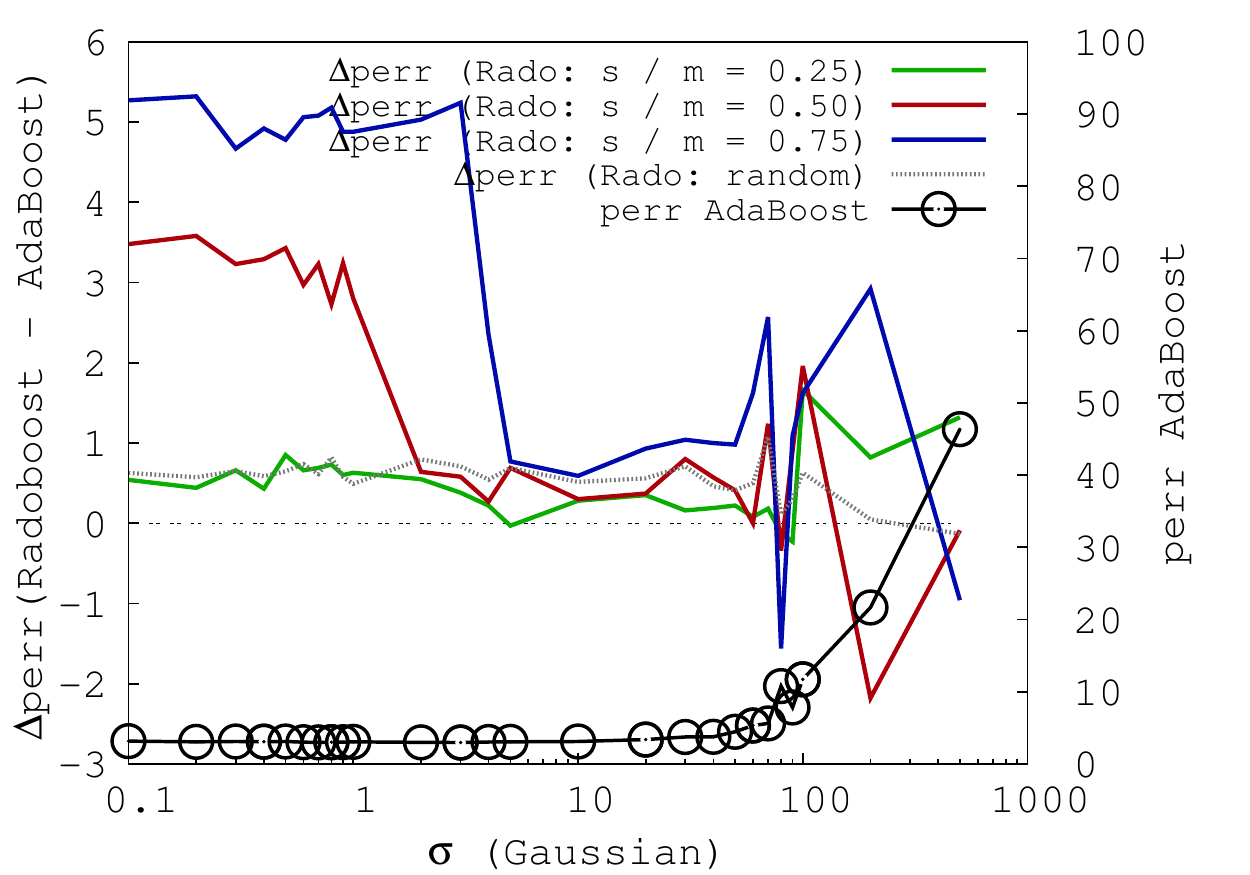}
& \includegraphics[width=0.40\columnwidth]{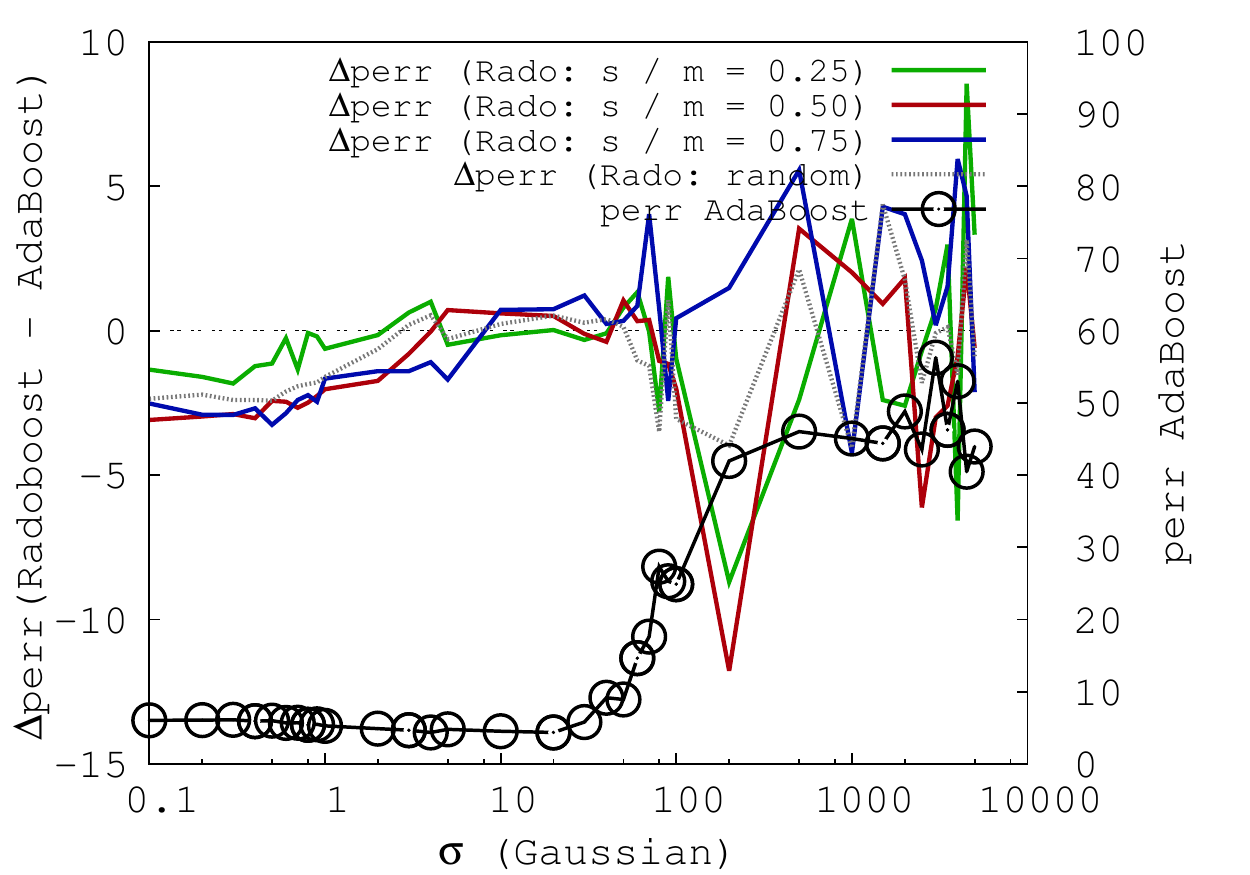}\\ \hline
\begin{turn}{90}Ionosphere\end{turn} & \includegraphics[width=0.40\columnwidth]{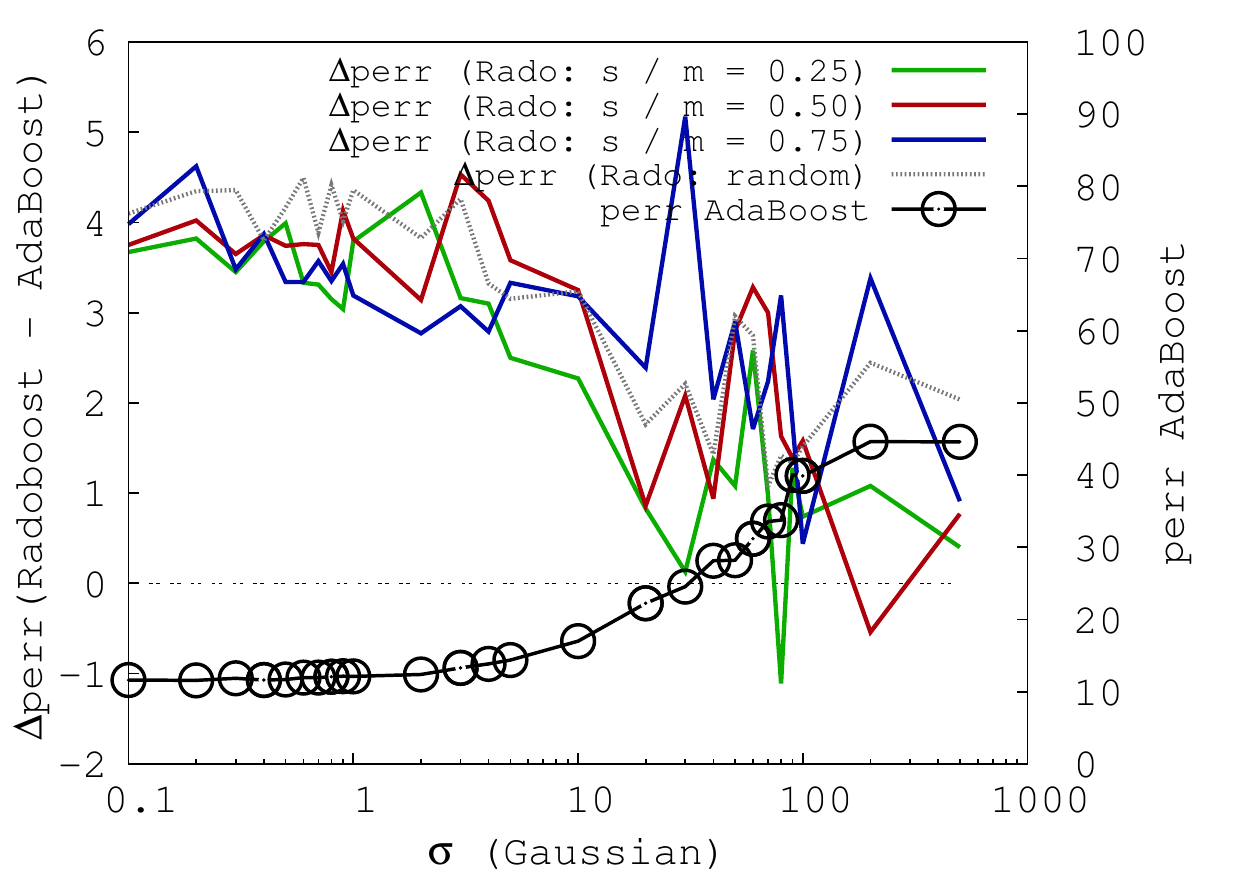}
& \includegraphics[width=0.40\columnwidth]{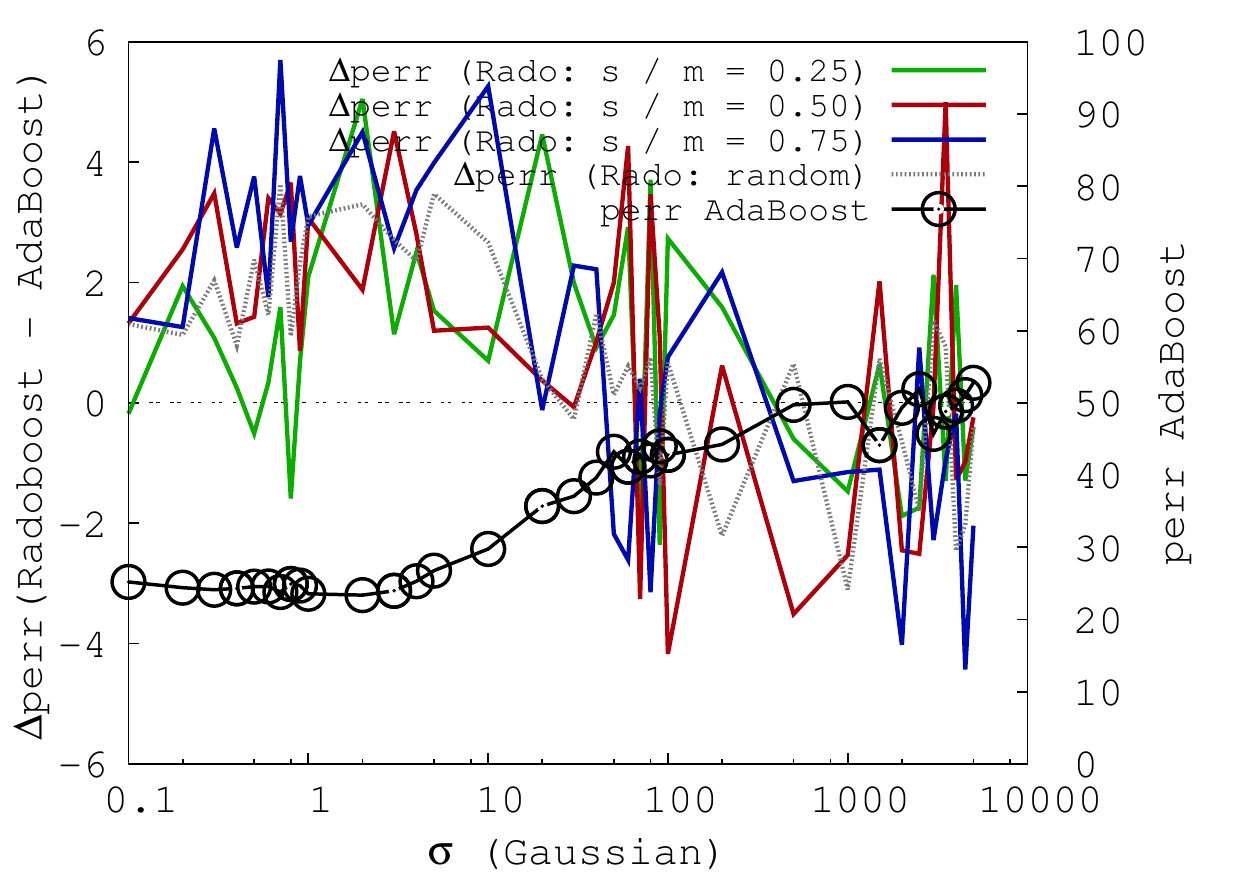}\\ 
\hline \hline
\end{tabular}
\end{center}
\caption{Learning from examples that have been noisified using the
  Gaussian mechanism ${\mathcal{N}}(\ve{0},
\upsigma^2 \mathrm{I})$ (See Section \ref{proof_thm_random_gau}), as a function of $\upsigma$. Conventions follow Table \ref{t-s52_1}.}
  \label{t-s52_2}
\end{table}

\begin{table}[t]
\begin{center}
\begin{tabular}{|r|c||c|}
\hline \hline
 & \weak~= Strong & \weak~= Median-prudential\\ \hline
\begin{turn}{90}Sonar\end{turn} & \includegraphics[width=0.40\columnwidth]{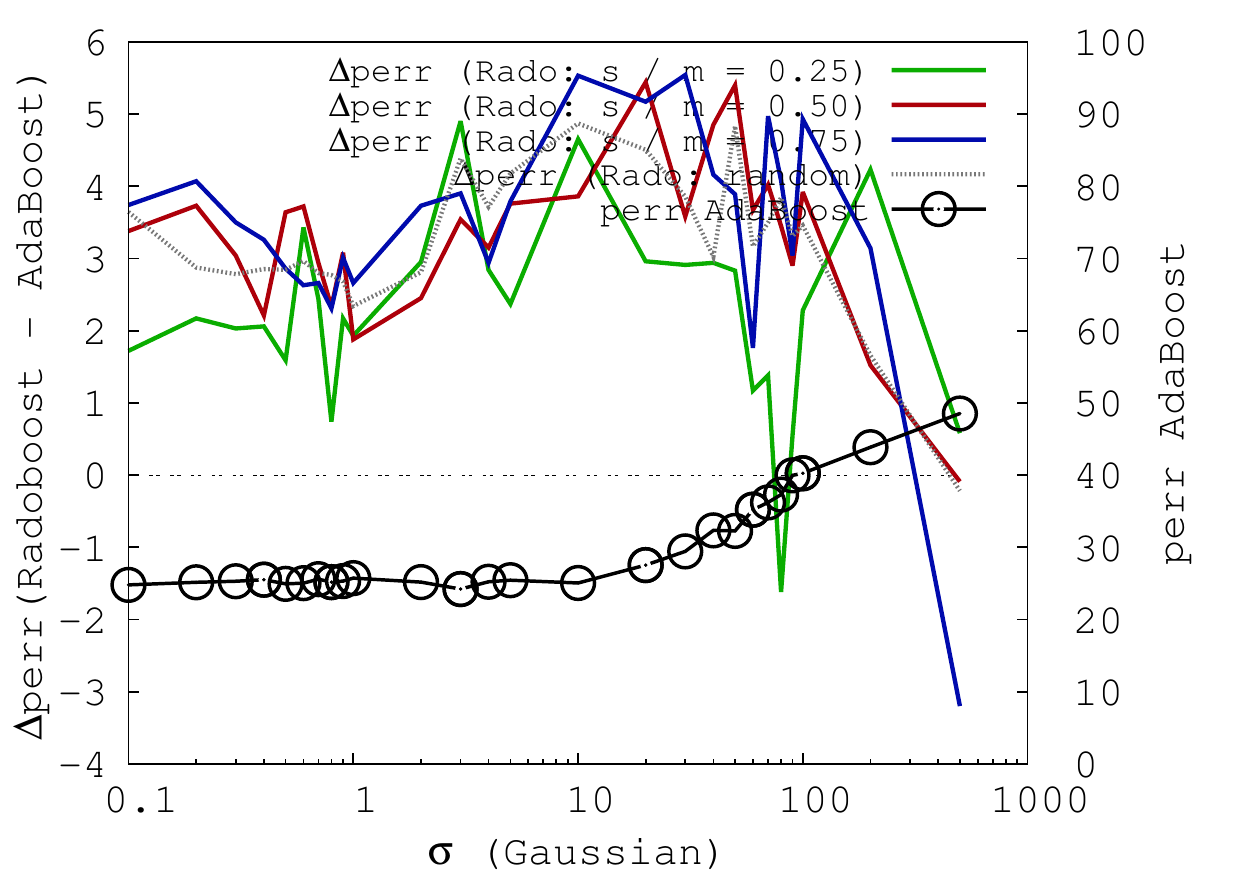}
& \includegraphics[width=0.40\columnwidth]{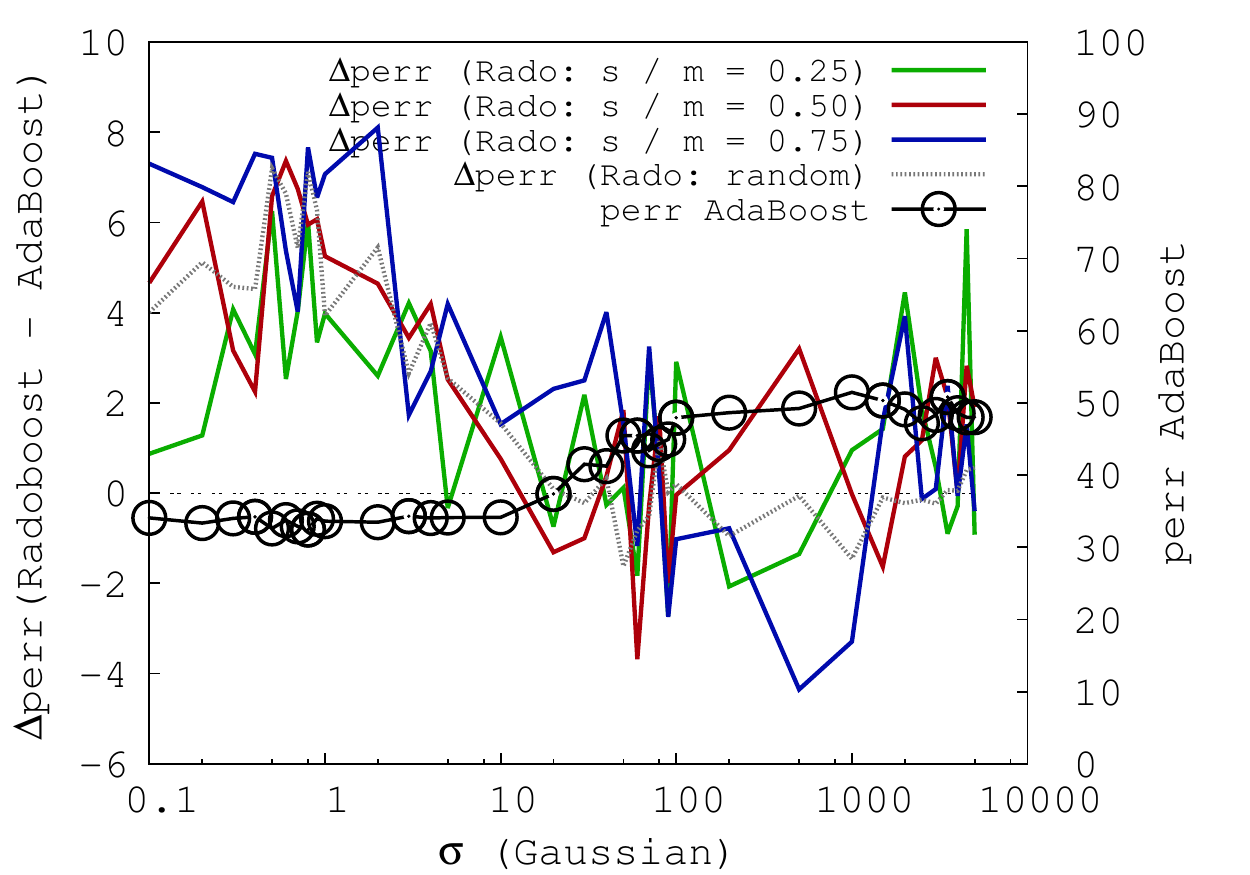}\\ \hline
\begin{turn}{90}Winered\end{turn} & \includegraphics[width=0.40\columnwidth]{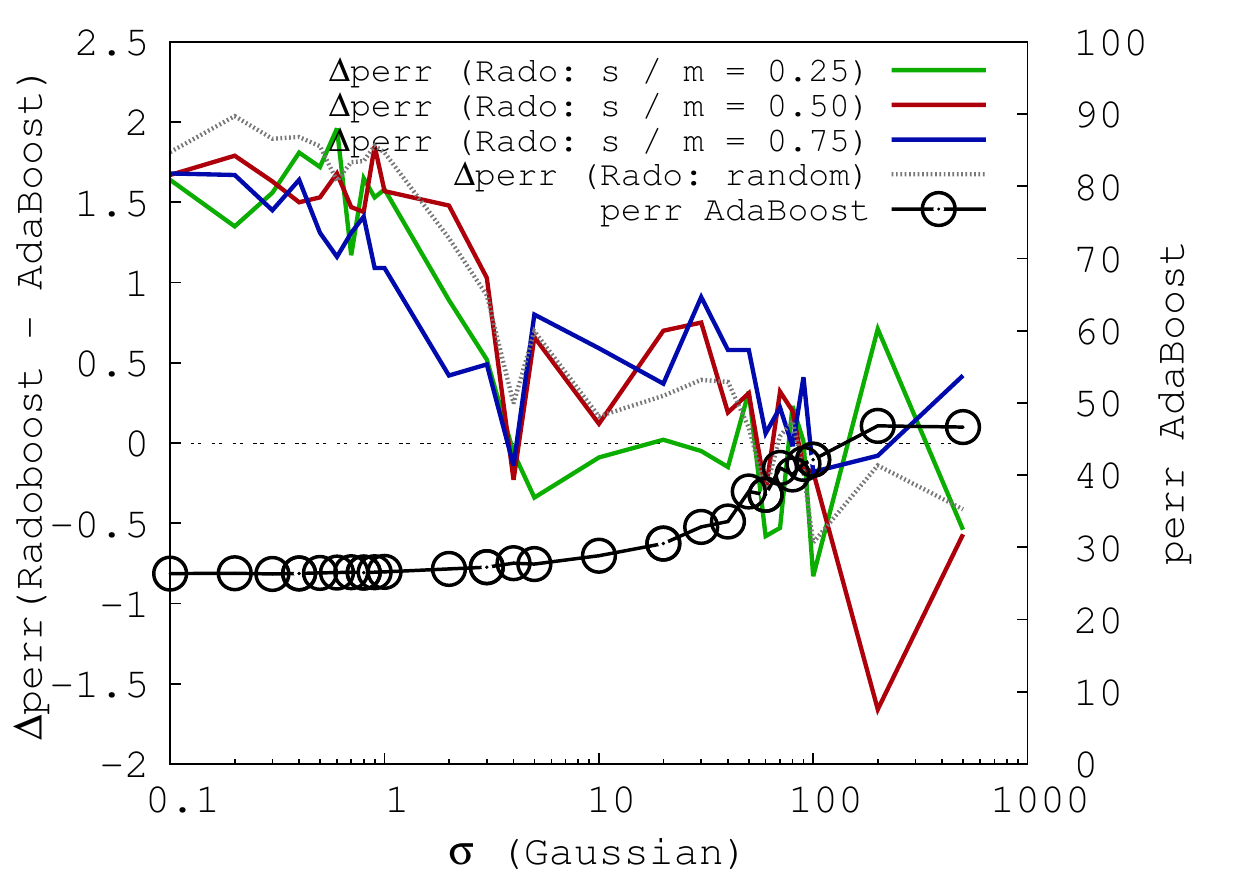}
& \includegraphics[width=0.40\columnwidth]{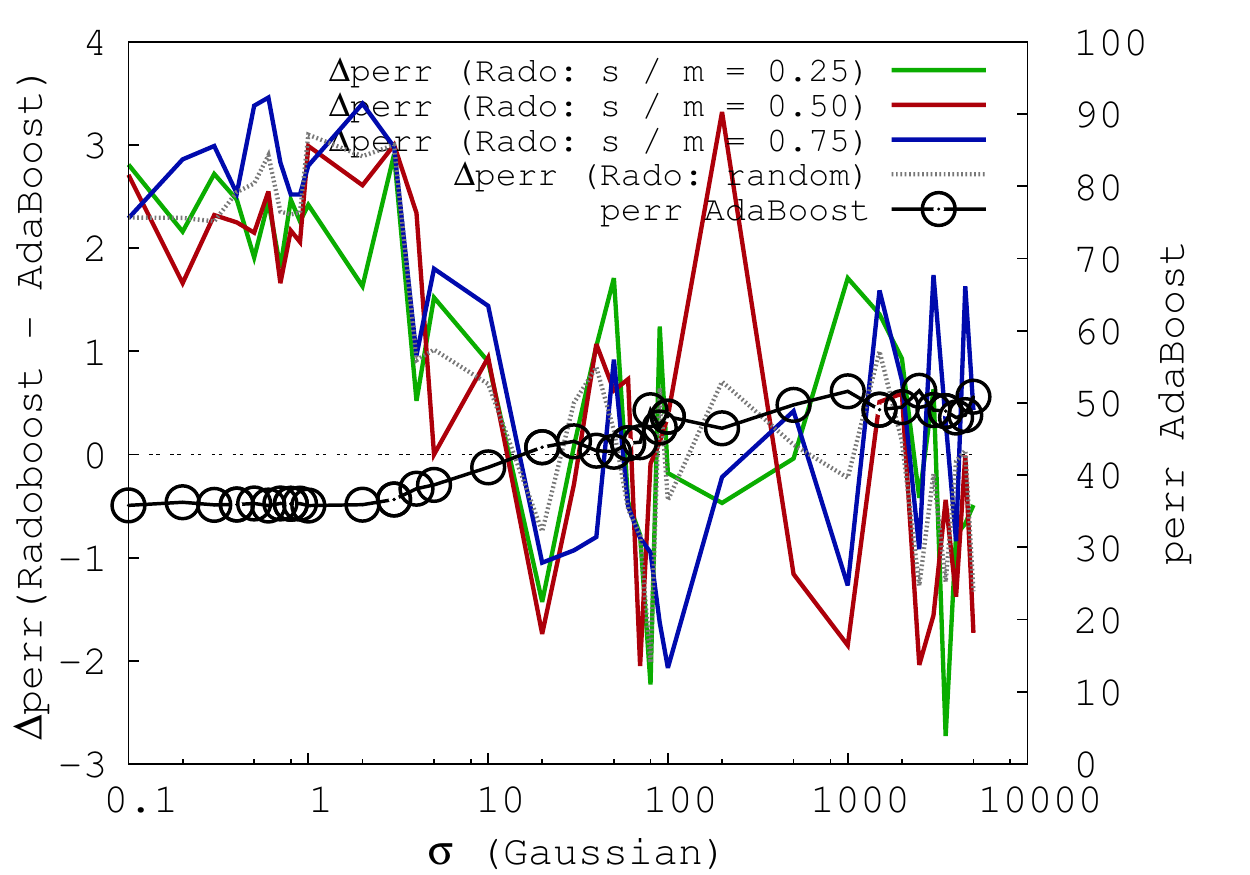}\\ \hline
\begin{turn}{90}Abalone\end{turn} & \includegraphics[width=0.40\columnwidth]{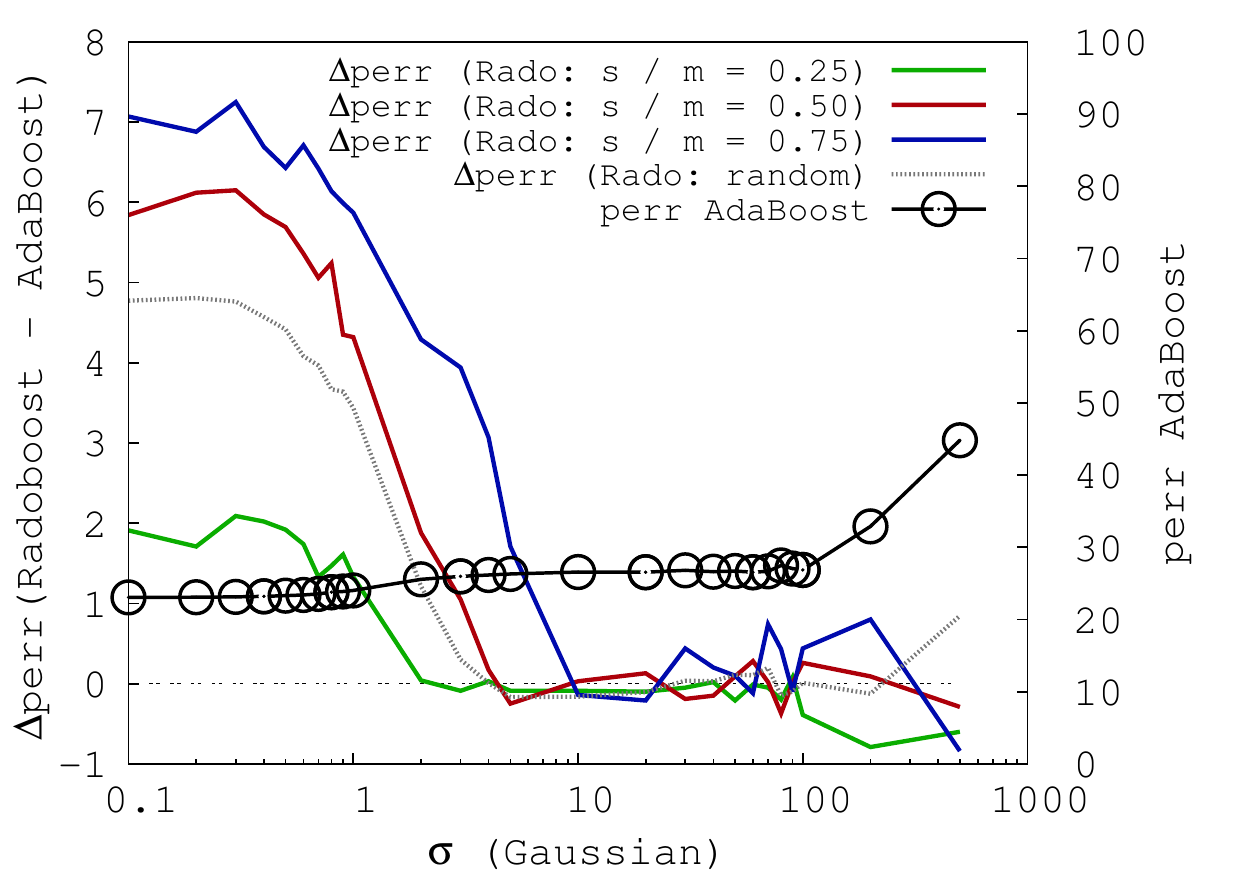}
& \includegraphics[width=0.40\columnwidth]{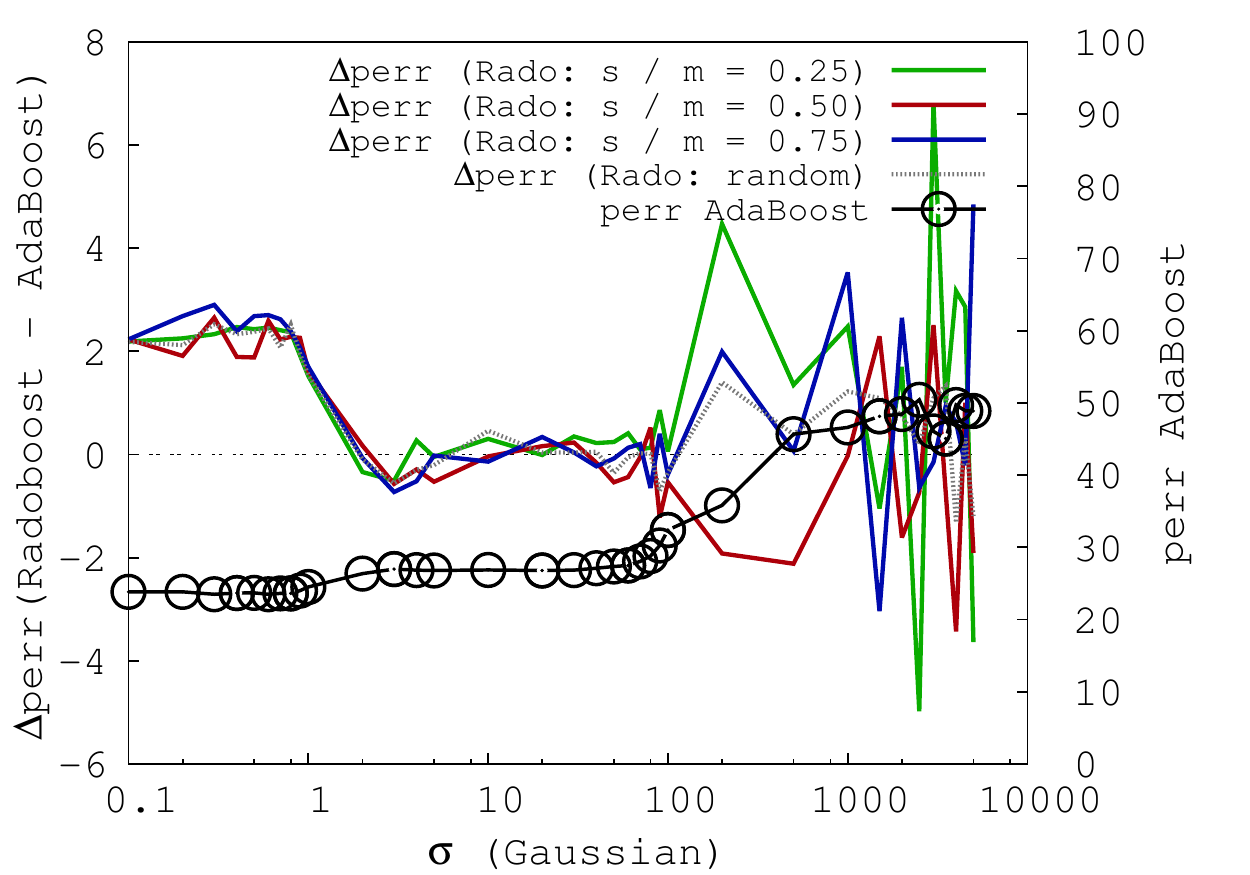}\\ 
\hline \hline
\end{tabular}
\end{center}
\caption{Learning from examples that have been noisified using the
  Gaussian mechanism ${\mathcal{N}}(\ve{0},
\upsigma^2 \mathrm{I})$ (See Section \ref{proof_thm_random_gau}), as a function of $\upsigma$. Conventions follow Table \ref{t-s52_1}.}
  \label{t-s52_3}
\end{table}

\begin{table}[t]
\begin{center}
\begin{tabular}{|r|c||c|}
\hline \hline
 & \weak~= Strong & \weak~= Median-prudential\\ \hline
\begin{turn}{90}Wine-white\end{turn} & \includegraphics[width=0.40\columnwidth]{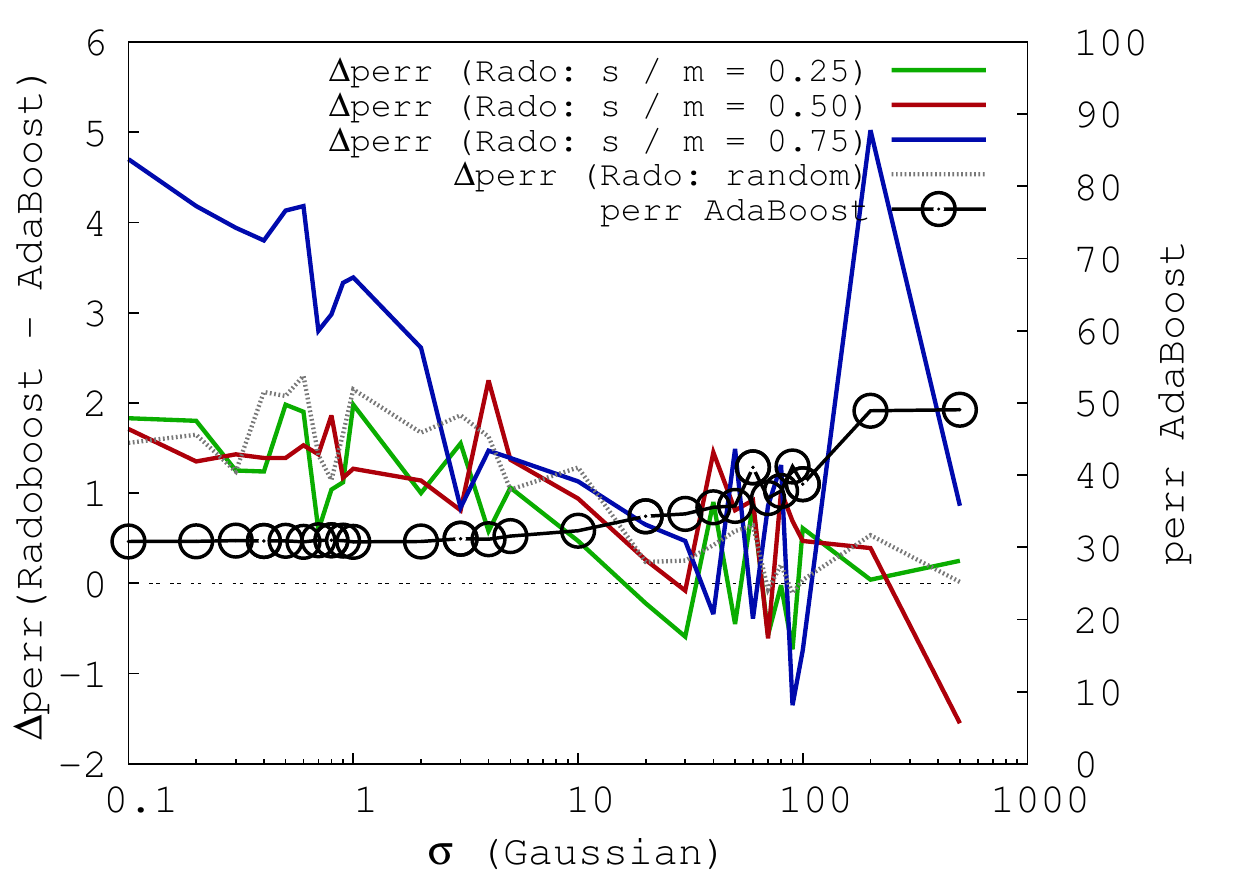}
& \includegraphics[width=0.40\columnwidth]{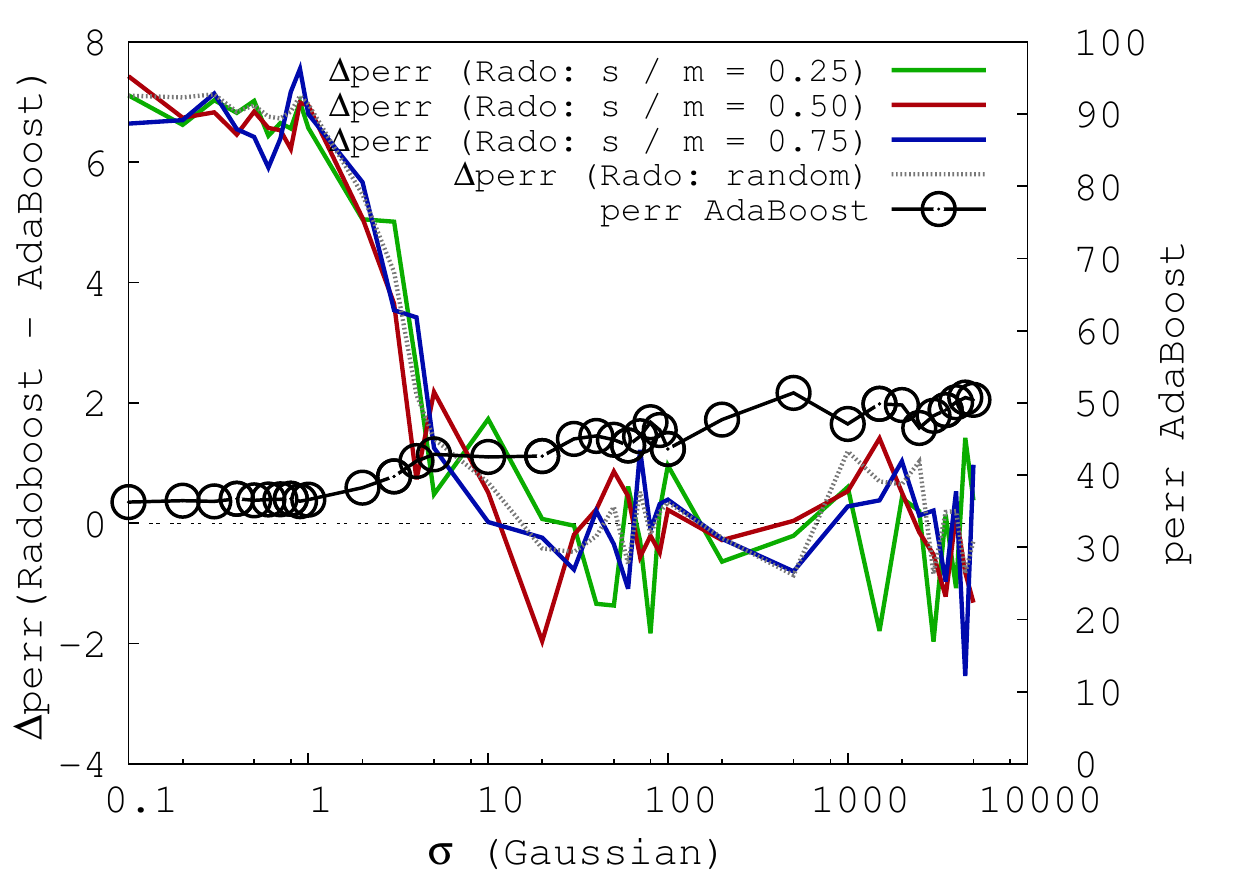}\\ \hline
\begin{turn}{90}Magic\end{turn} & \includegraphics[width=0.40\columnwidth]{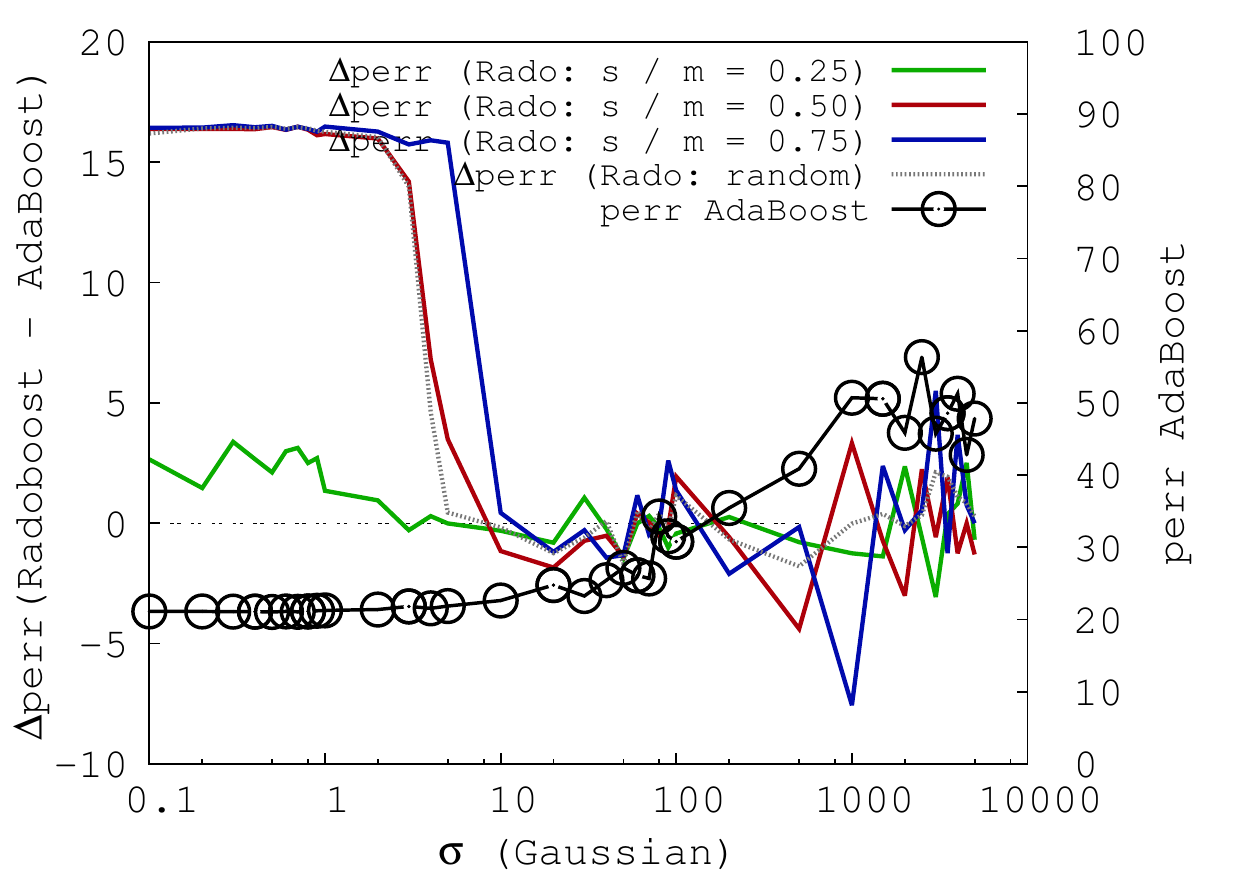}
& \includegraphics[width=0.40\columnwidth]{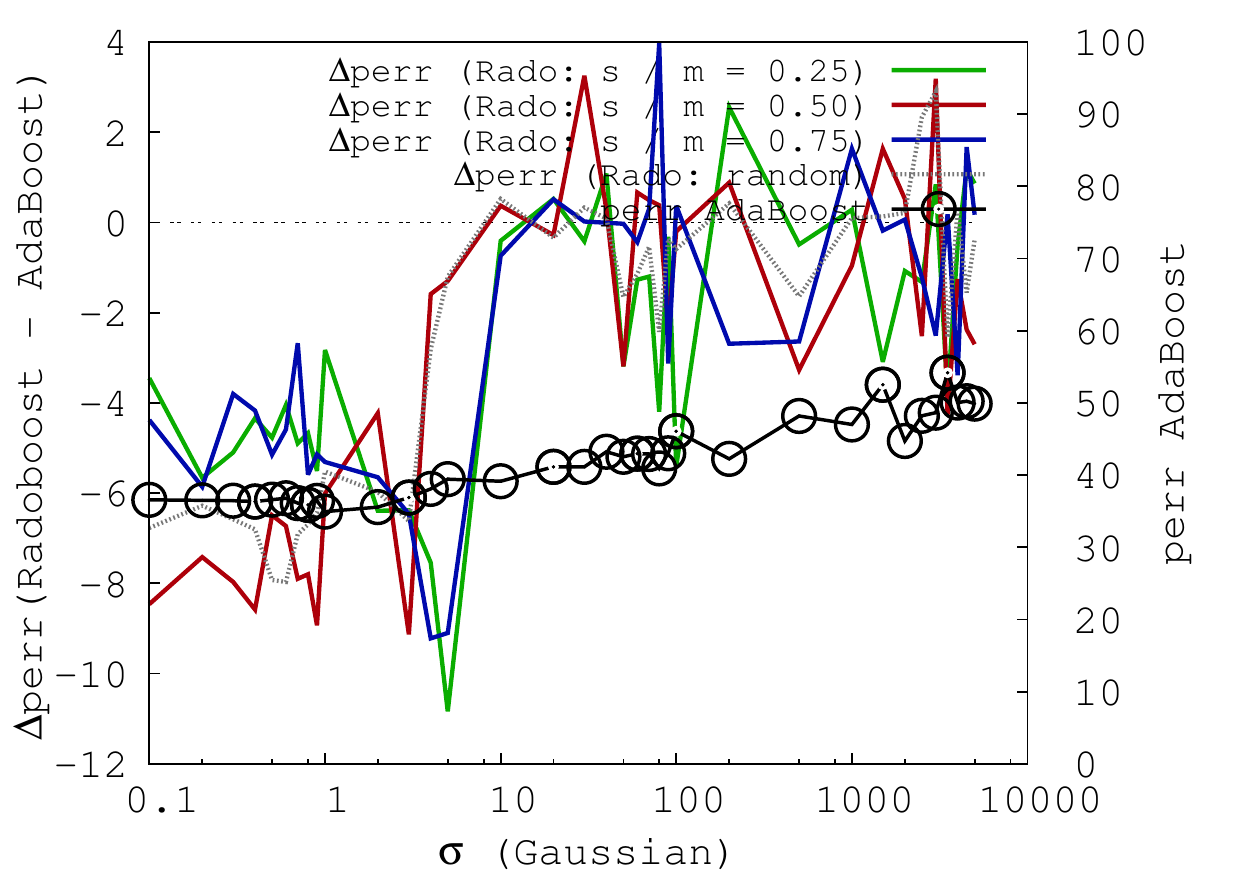}\\ \hline
\begin{turn}{90}Eeg\end{turn} & \includegraphics[width=0.40\columnwidth]{Plots/Strong_Weak_Oracle/eeg_RadoBoostSupport-AdaBoost}
&  \includegraphics[width=0.40\columnwidth]{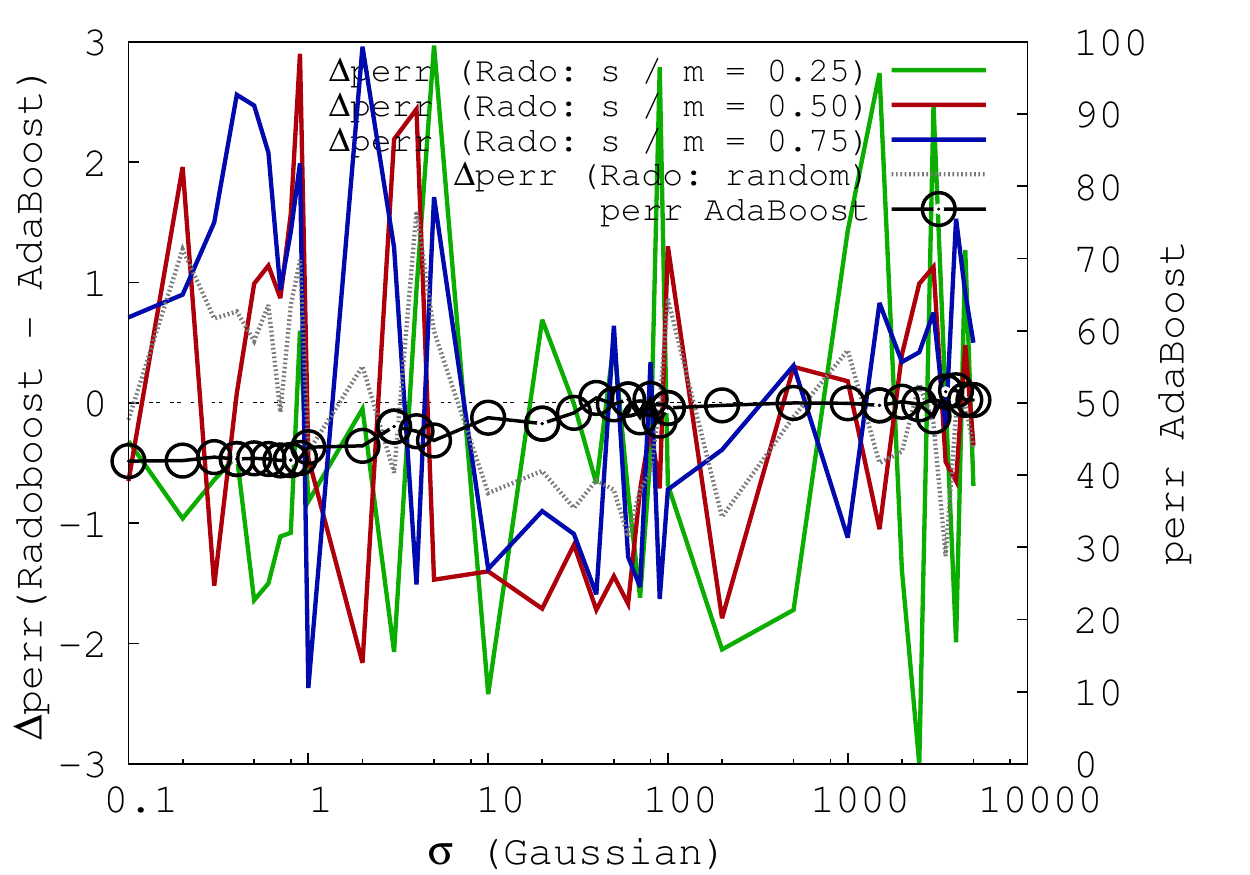}\\ 
\hline \hline
\end{tabular} 
\end{center}
\caption{Learning from examples that have been noisified using the
  Gaussian mechanism ${\mathcal{N}}(\ve{0},
\upsigma^2 \mathrm{I})$ (See Section \ref{proof_thm_random_gau}), as a function of $\upsigma$. Conventions follow Table \ref{t-s52_1}.}
  \label{t-s52_4} 
\end{table}

\subsection{Supplementary experiments to Section \ref{sradp} --- II / III}\label{exp_sradp2}

Tables \ref{t-s52_7} and \ref{t-s52_8} compare \radoboost~trained with
rados of fixed support and using a ``prudential'' weak learner (which
picks the median feature according to $|r_t|$), to \radoboost~trained
with plain random rados and using the ``strongest'' possible weak
learner which picks the best feature according to $|r_t|$.

\begin{table}[t]
\begin{center}
\begin{tabular}{|c||c|}
\hline \hline
 \includegraphics[width=0.40\columnwidth]{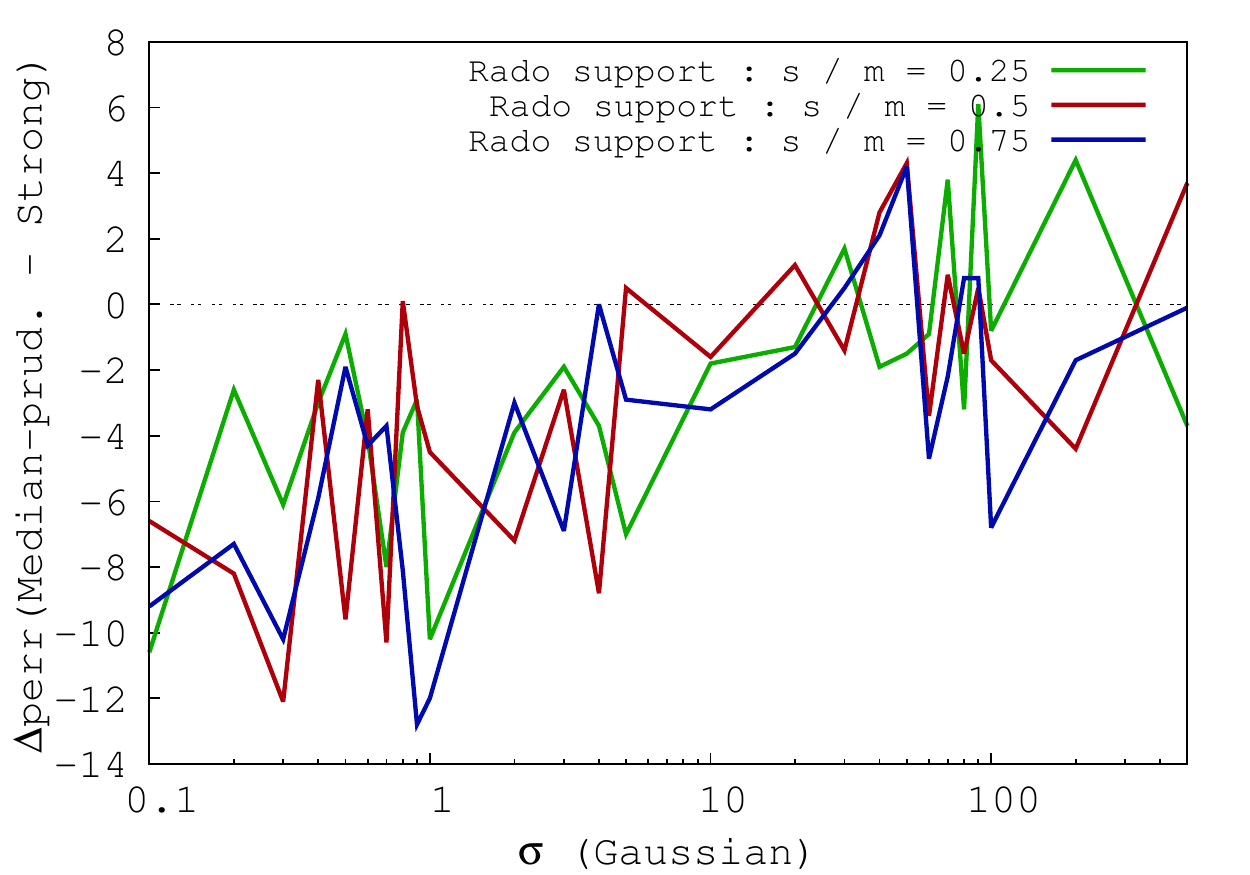}
& \includegraphics[width=0.40\columnwidth]{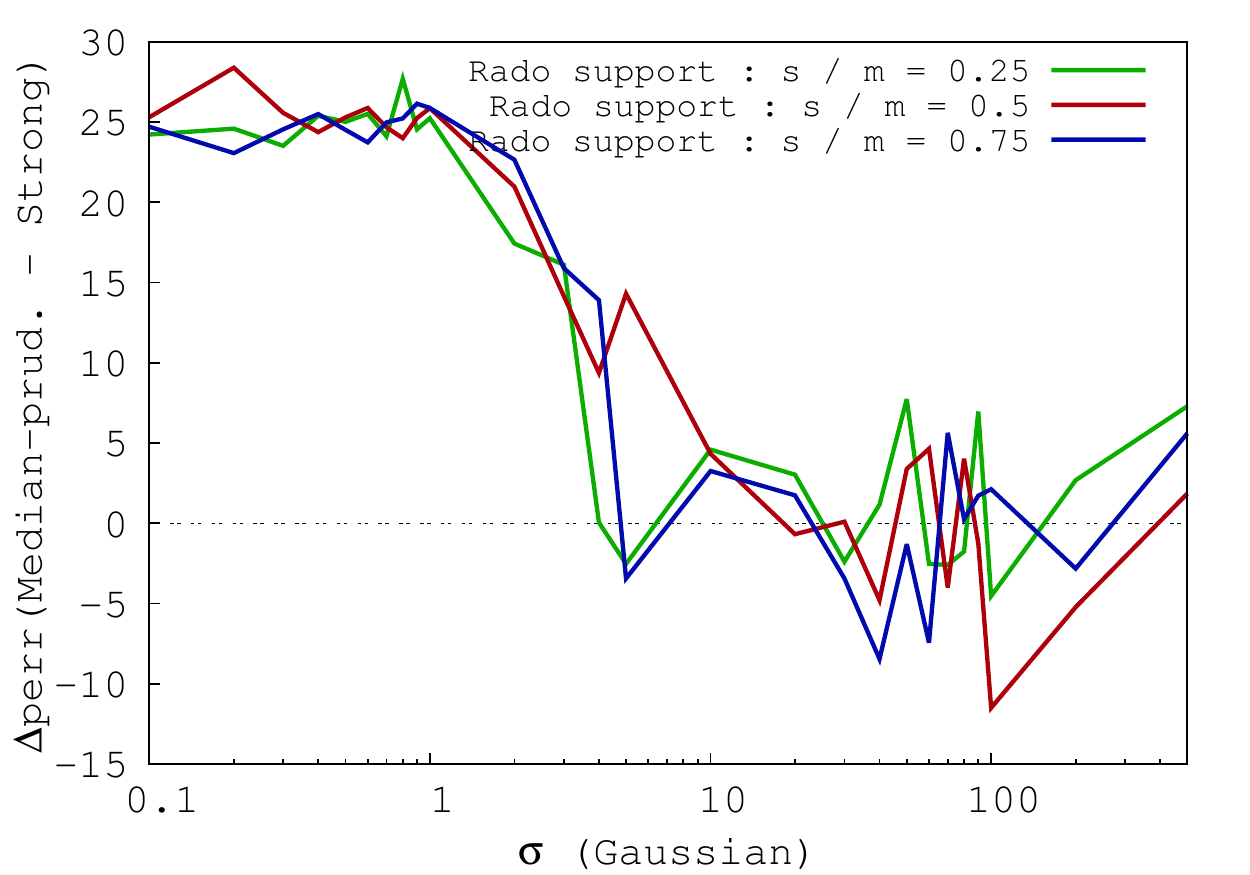}\\ 
Fertility & Haberman \\ \hline
 \includegraphics[width=0.40\columnwidth]{Plots/Median+SupportMinusStrong+Random/transfusion_Median+SupportMinusStrong+Random}
& \includegraphics[width=0.40\columnwidth]{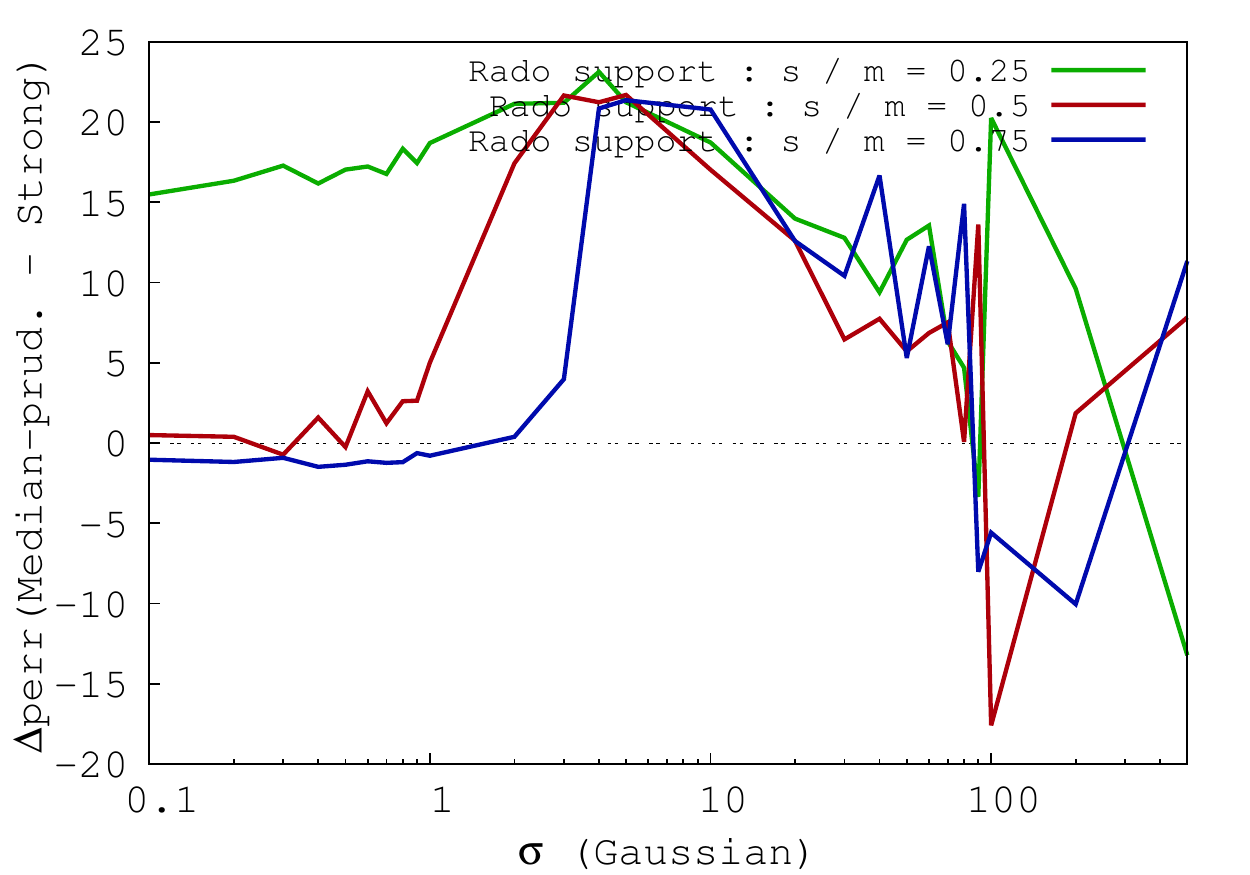}\\
Transfusion & Banknote\\ \hline
\includegraphics[width=0.40\columnwidth]{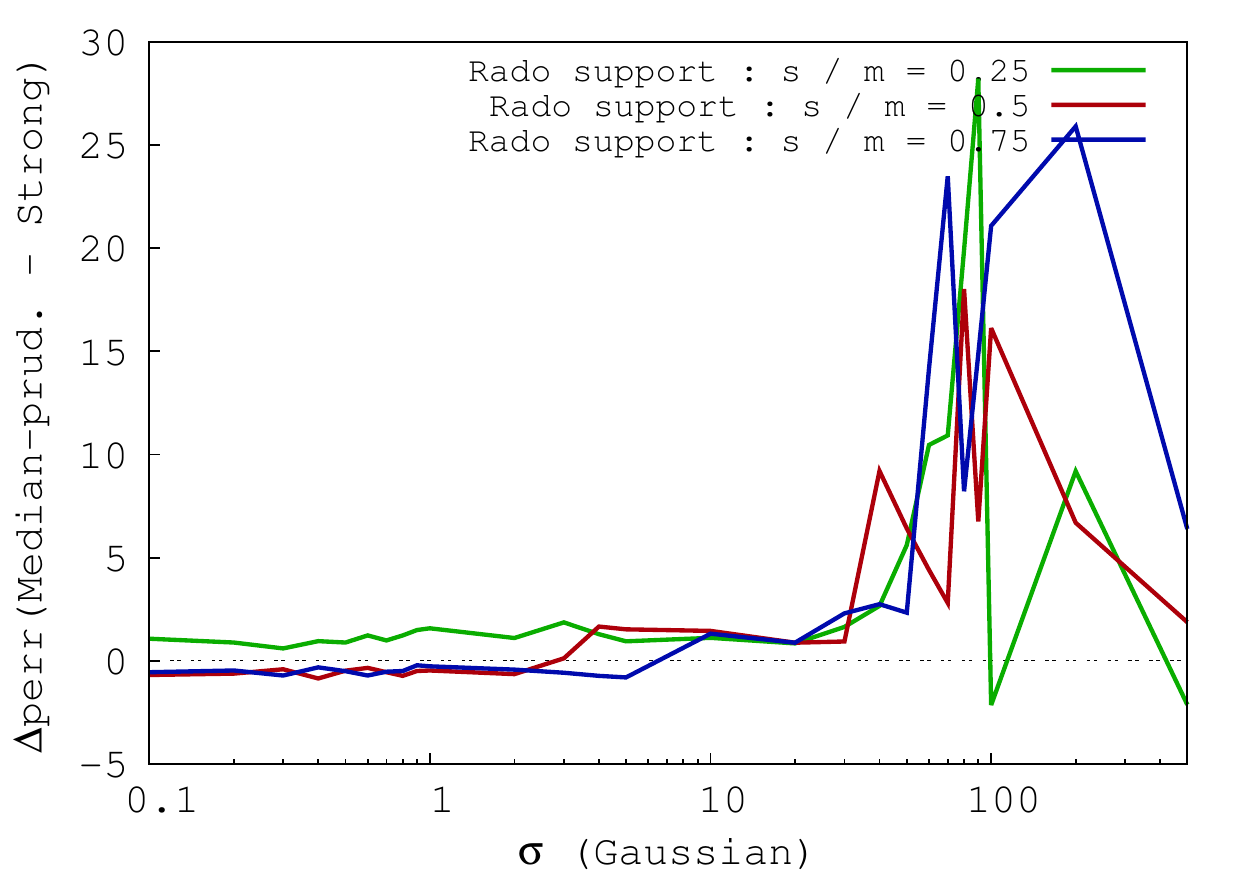}
& \includegraphics[width=0.40\columnwidth]{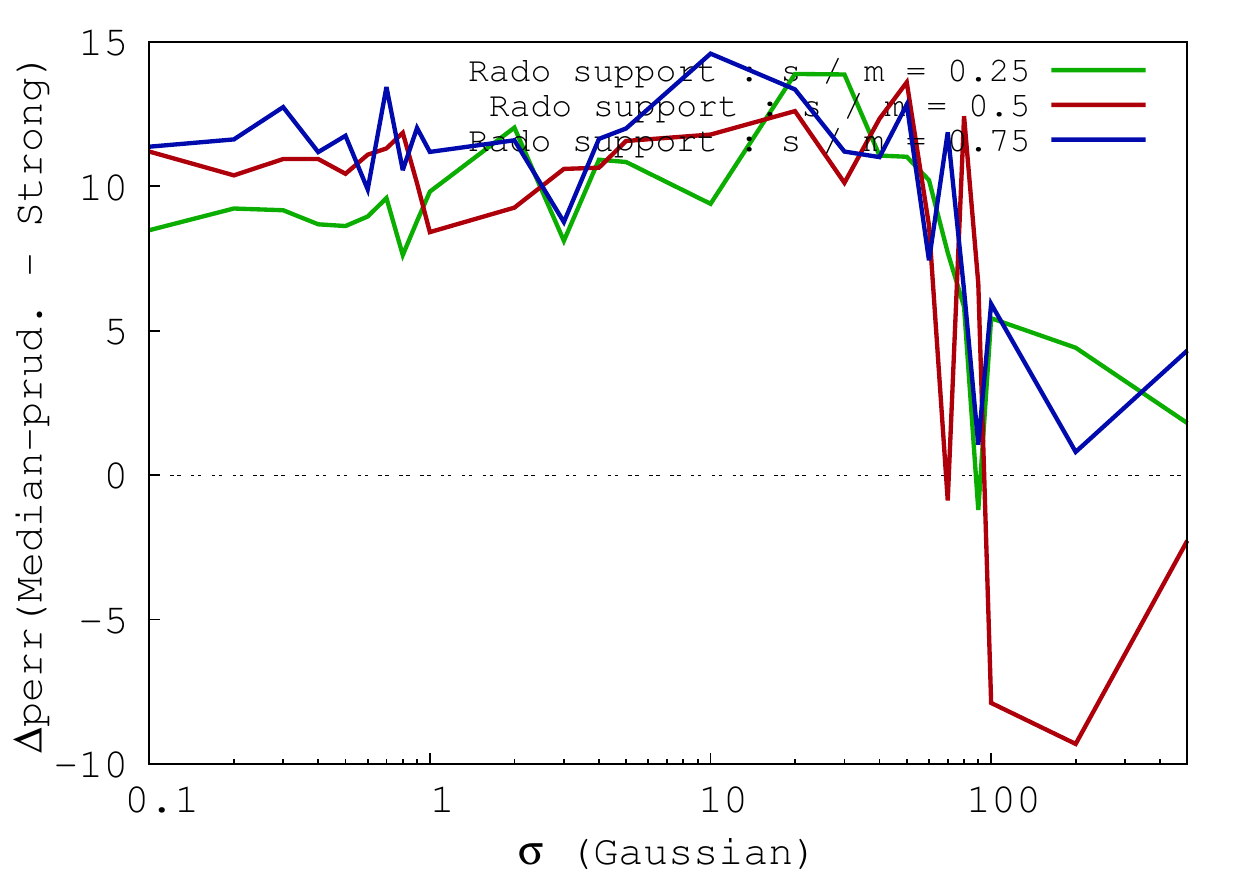}\\
Breastwisc & Ionosphere\\ 
\hline \hline
\end{tabular} 
\end{center}
\caption{Test error of \radoboost~trained with rados with fixed
  support and Median-prudential weak learner (Subsection \ref{sbfdp}), \textit{minus} test error of
  \radoboost~trained with random rados and the ``Strong'' weak learner
  of Section \ref{exp_boost_rado} (\textit{i.e.} the one that picks
  the best feature at each iteration), as a function of the Gaussian mechanism's standard deviation $\upsigma$. Horizontal dashed line
  correspond to $\Delta$perr = 0. Points below this line denote better
  performances over the rados with fixed support and with the
  prudential weak learner. 
$s$ is the support size ($m$ relates to the size of the training
fold), for three values, $s / m = 0.25$ (green), $s / m = 0.5$ (red)
and $s / m = 0.75$ (blue).}
  \label{t-s52_7} 
\end{table}

\begin{table}[t]
\begin{center}
\begin{tabular}{|c||c|}
\hline \hline
 \includegraphics[width=0.40\columnwidth]{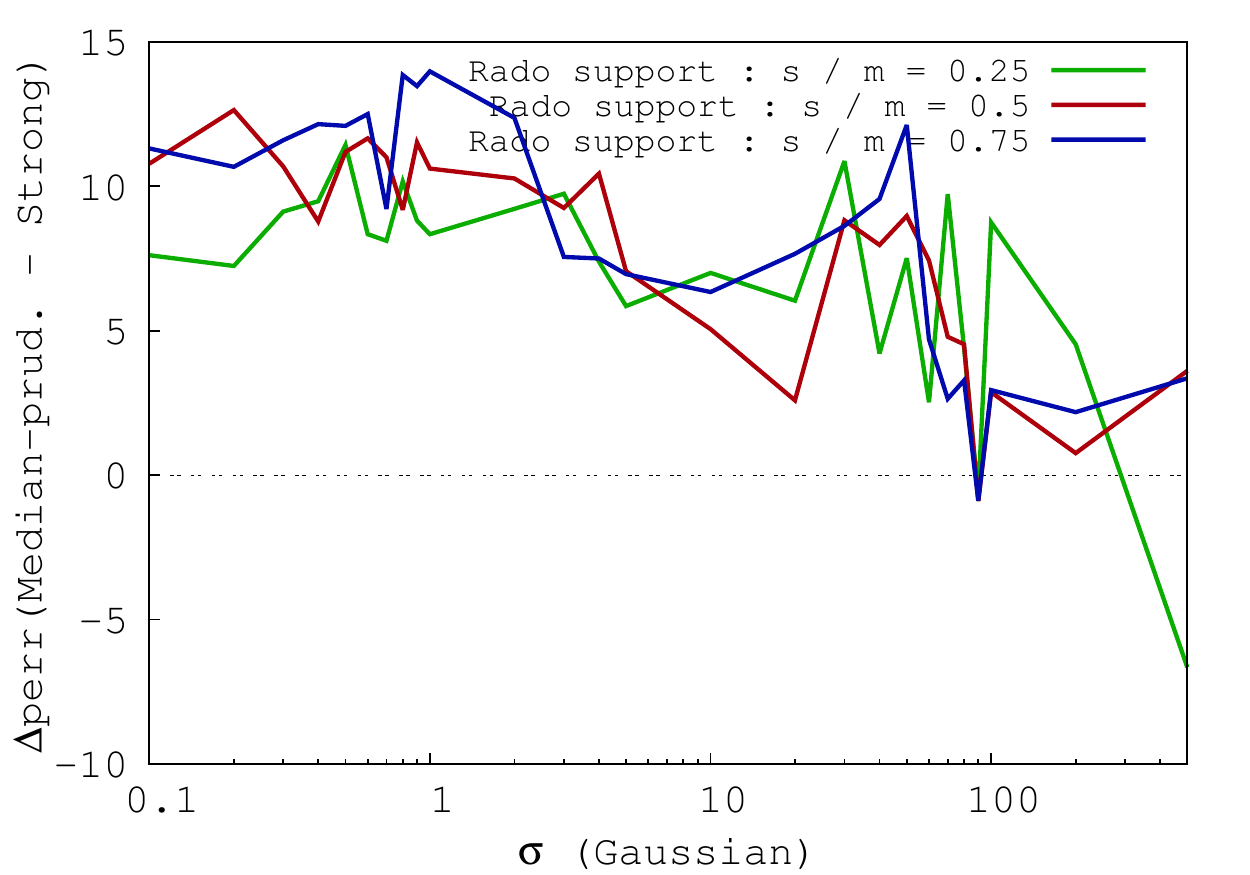}
& \includegraphics[width=0.40\columnwidth]{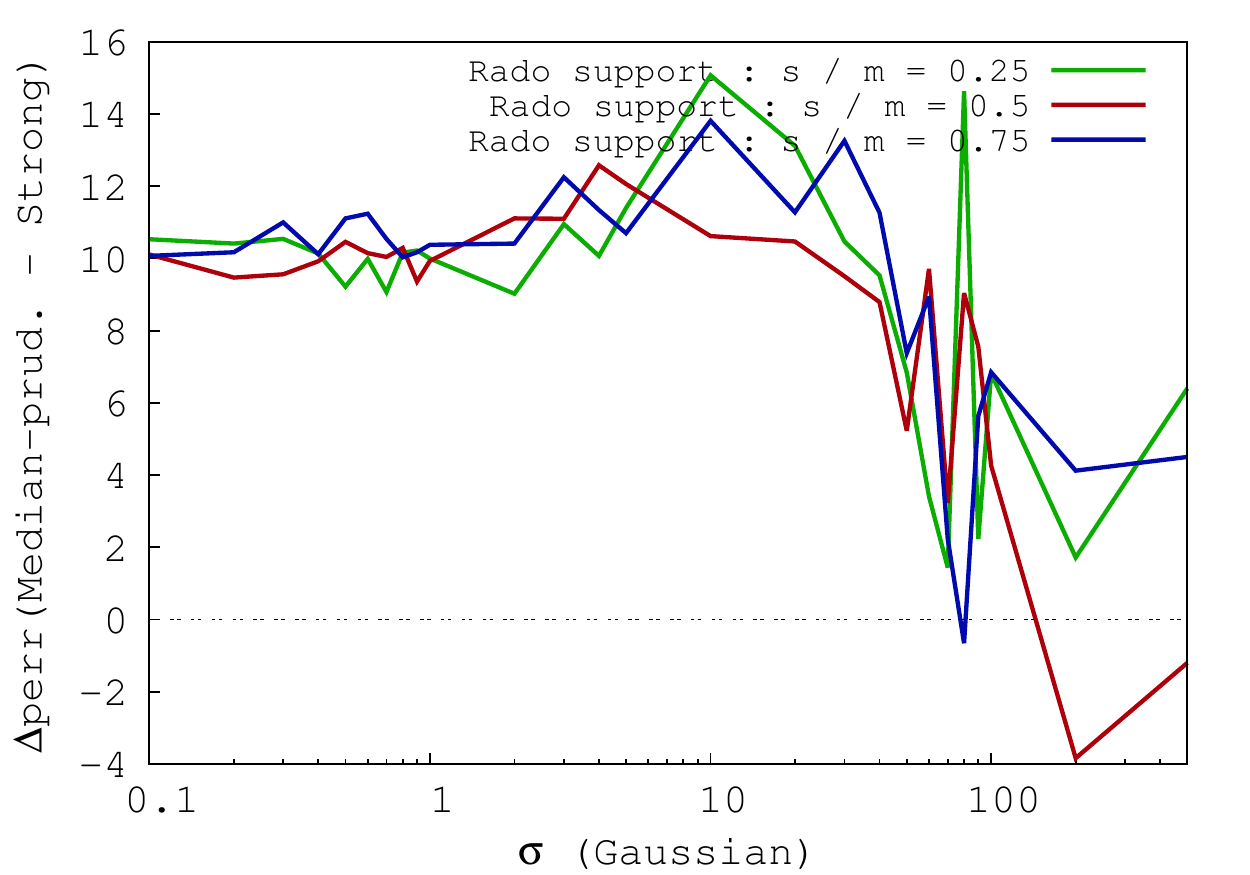}\\ 
Sonar & Winered \\ \hline
 \includegraphics[width=0.40\columnwidth]{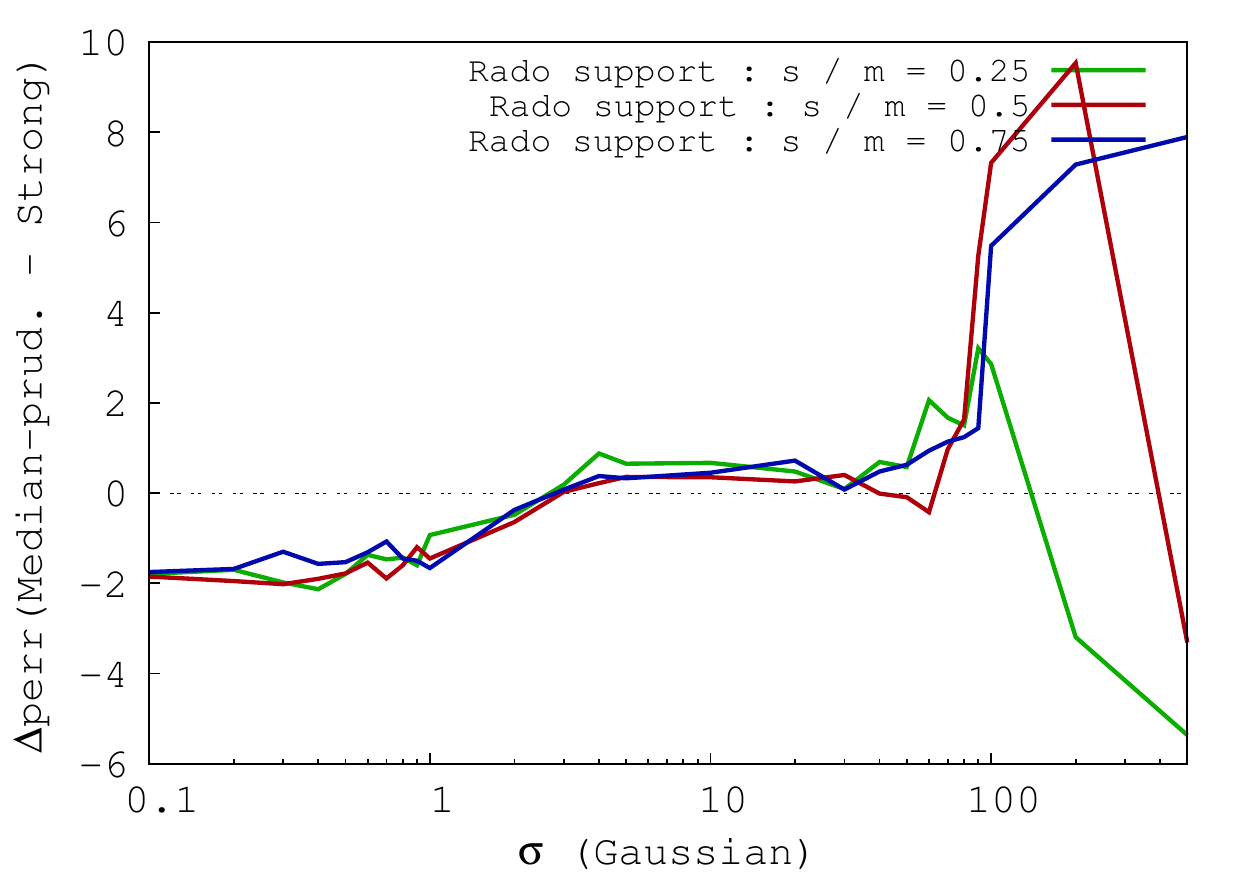}
& \includegraphics[width=0.40\columnwidth]{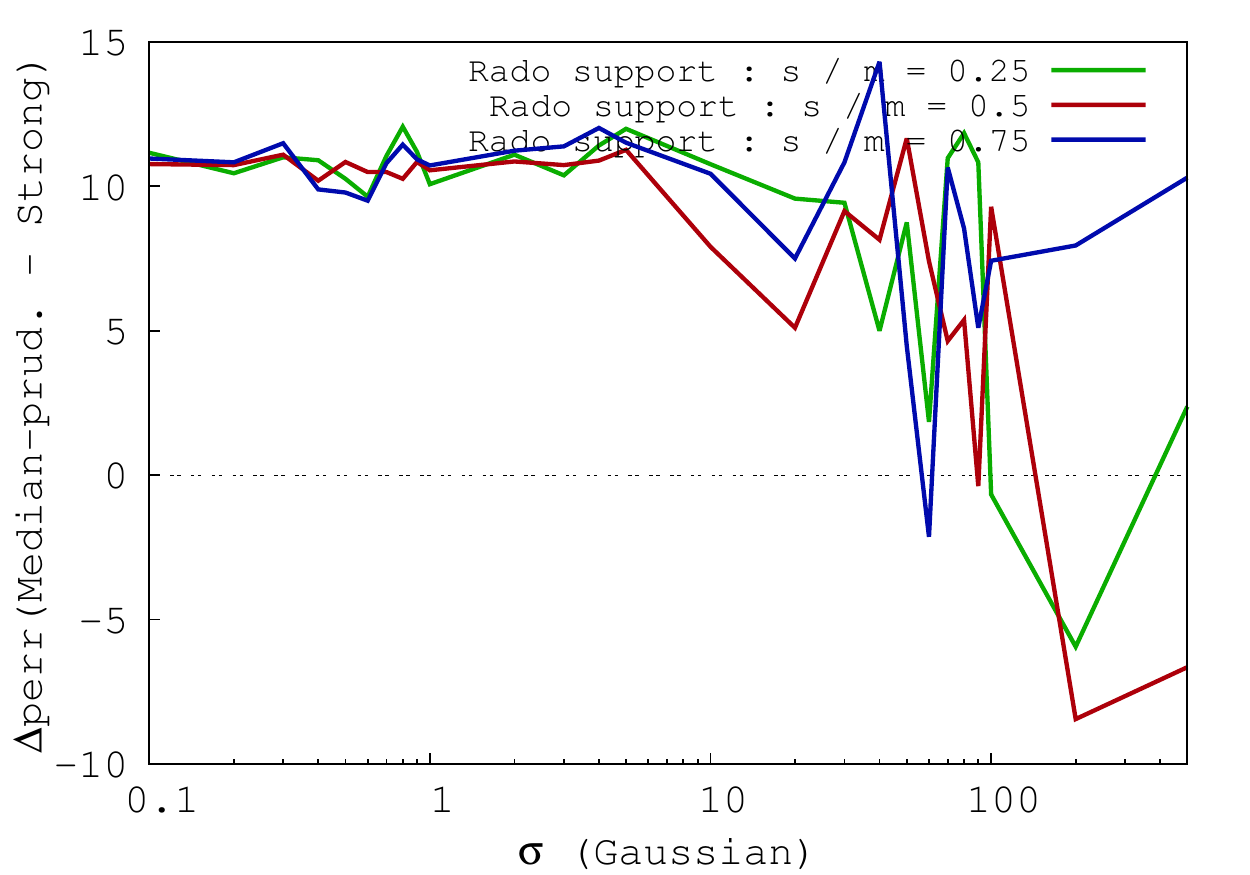}\\
Abalone & Wine-white\\ \hline
\includegraphics[width=0.40\columnwidth]{Plots/Median+SupportMinusStrong+Random/magic_Median+SupportMinusStrong+Random}
& \includegraphics[width=0.40\columnwidth]{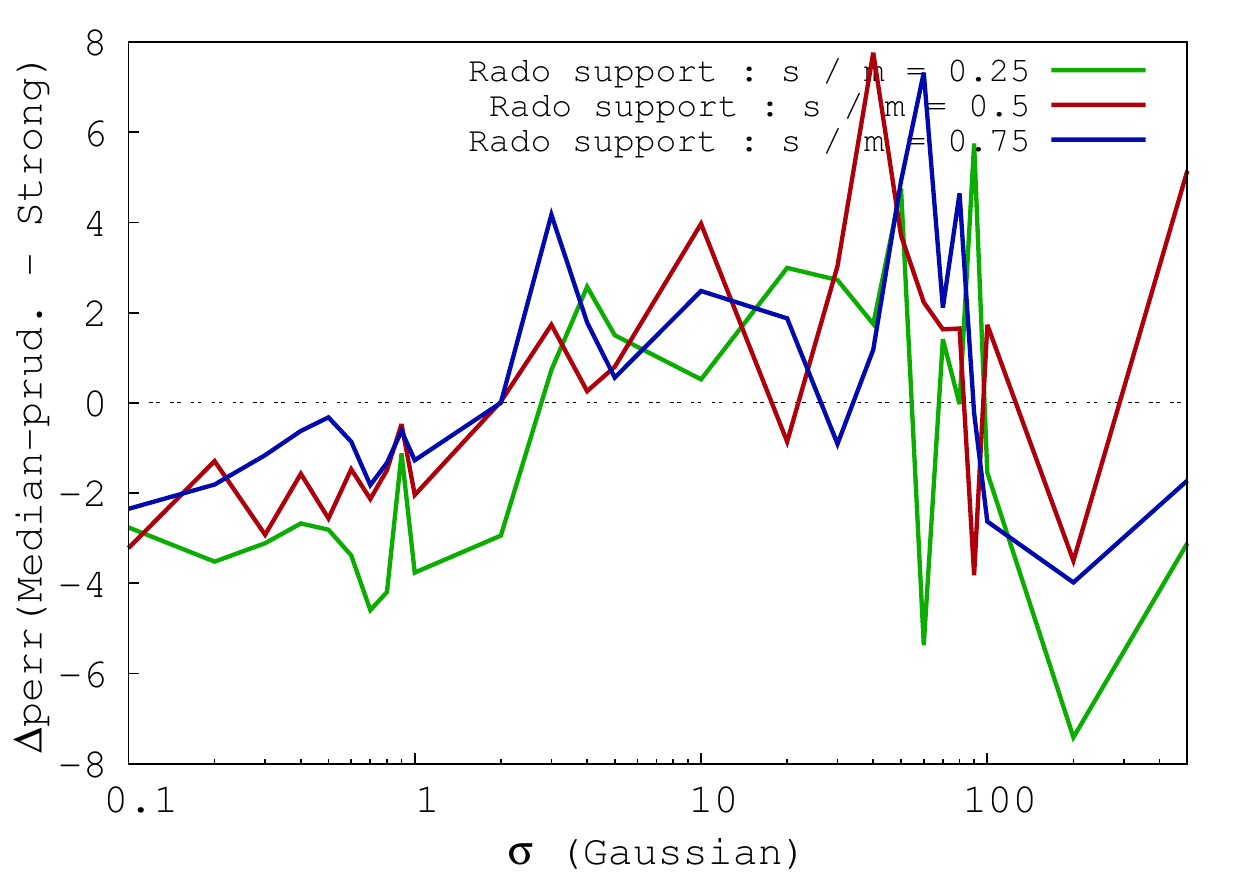}\\
Magic & Eeg\\ 
\hline \hline
\end{tabular} 
\end{center}
\caption{Test error of \radoboost~trained with rados with fixed
  support and Median-prudential weak learner, \textit{minus} test error of
  \radoboost~trained with random rados and the ``Strong'' weak learner
  of Section \ref{exp_boost_rado}. Conventions follow Table \ref{t-s52_7}.}
  \label{t-s52_8} 
\end{table}

\subsection{Supplementary experiments to Section \ref{sradp} --- III / III}\label{exp_sradp3}

Tables \ref{t-s52_5} and \ref{t-s52_6} compare two different rado
generation mechanisms with respect to \radoboost: the random 
generation of arbitrary rados (Section
\ref{exp_boost_rado}), and the random generation of rados with fixed
support (Subsection \ref{sbfdp}). In both Tables, the weak
learner is always the same (contrary to Tables \ref{t-s52_7} and \ref{t-s52_8}), \textit{i.e.} the ``strong'' weak learner
that picks the best feature according to $|r_t|$, at each iteration.

\begin{table}[t]
\begin{center}
\begin{tabular}{|c||c|}
\hline \hline
 \includegraphics[width=0.40\columnwidth]{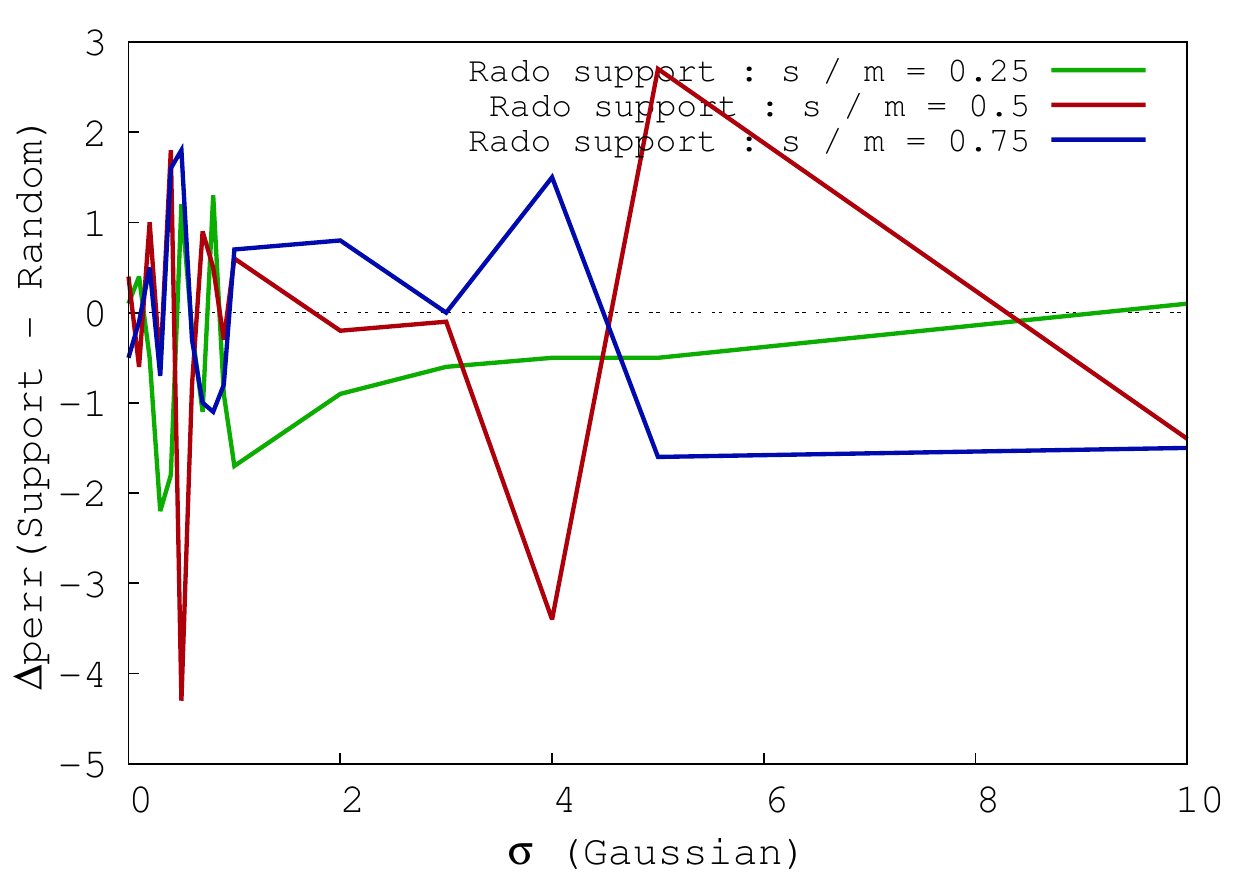}
& \includegraphics[width=0.40\columnwidth]{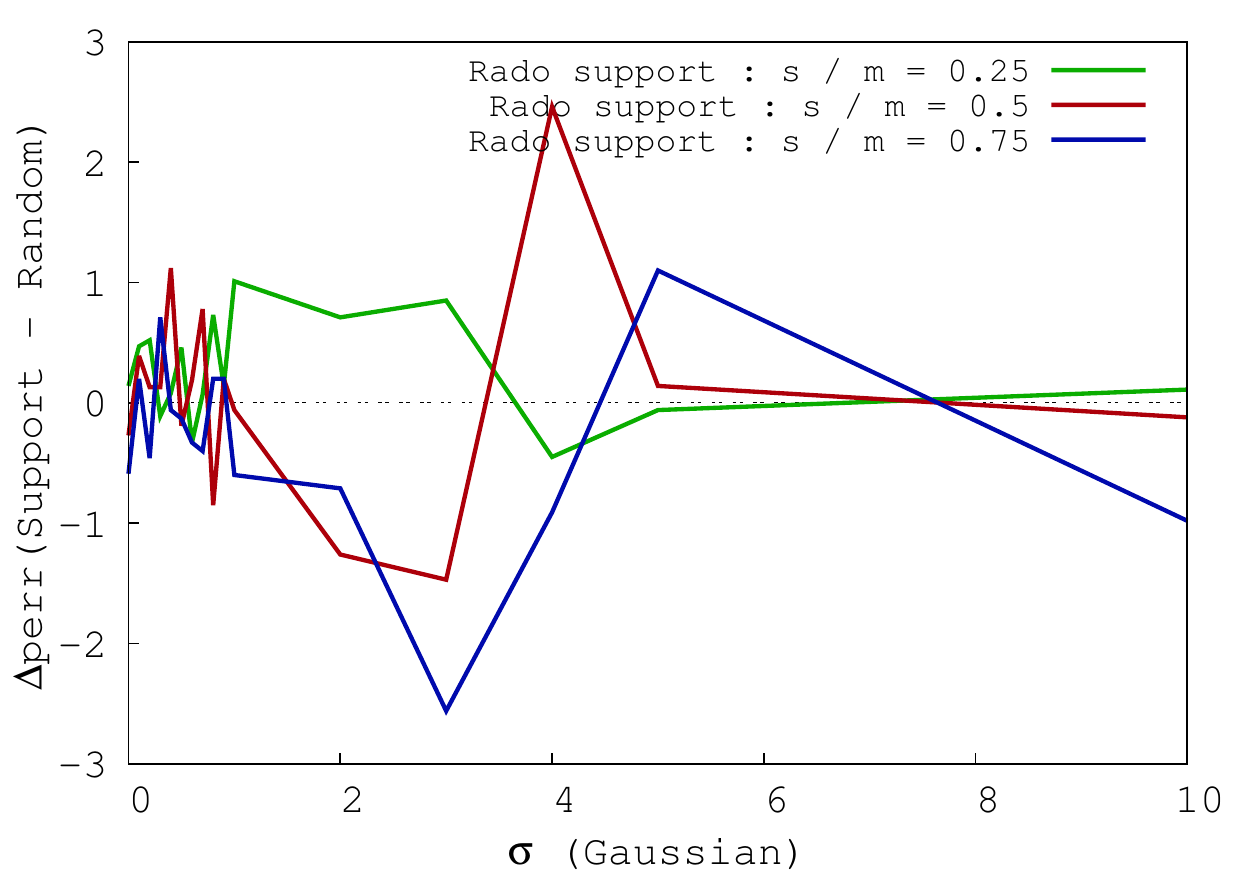}\\ 
Fertility & Haberman \\ \hline
 \includegraphics[width=0.40\columnwidth]{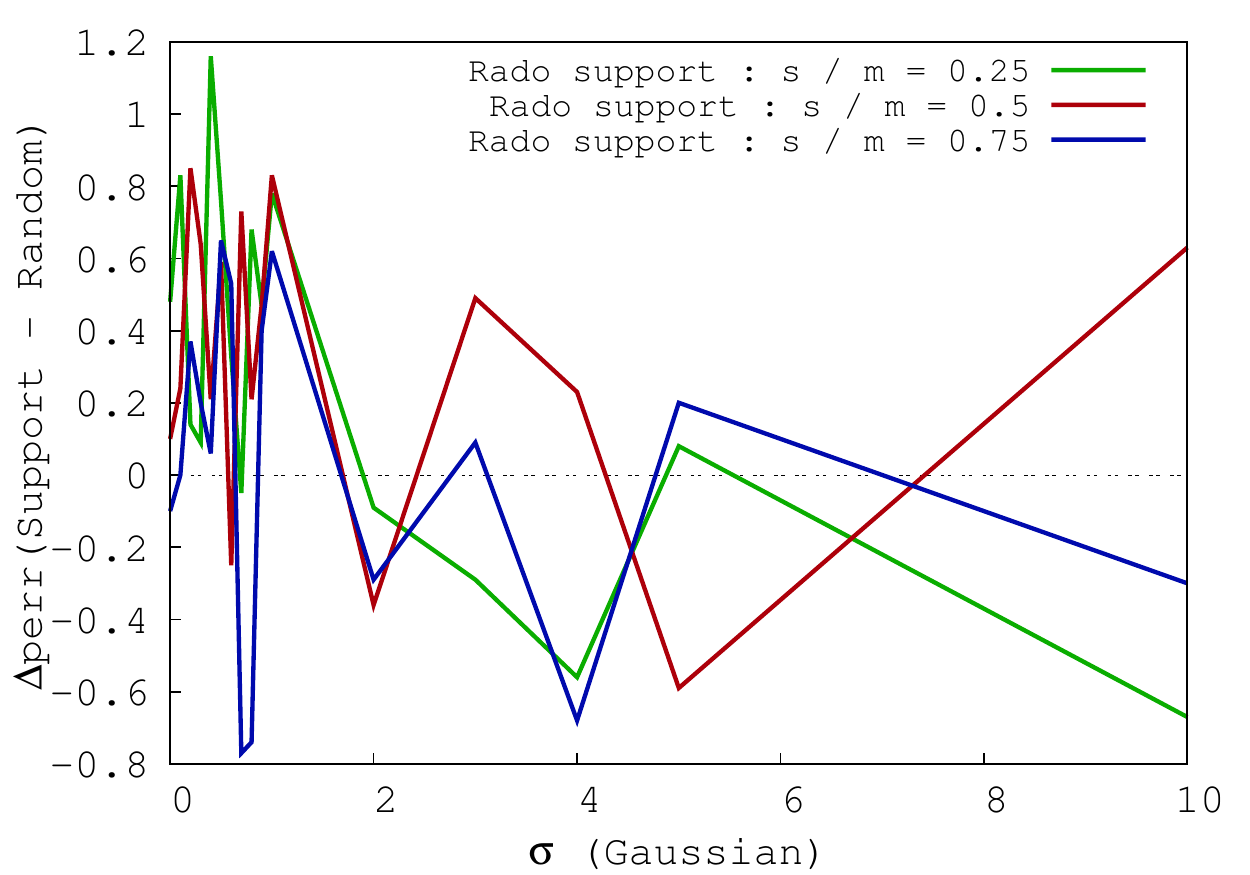}
& \includegraphics[width=0.40\columnwidth]{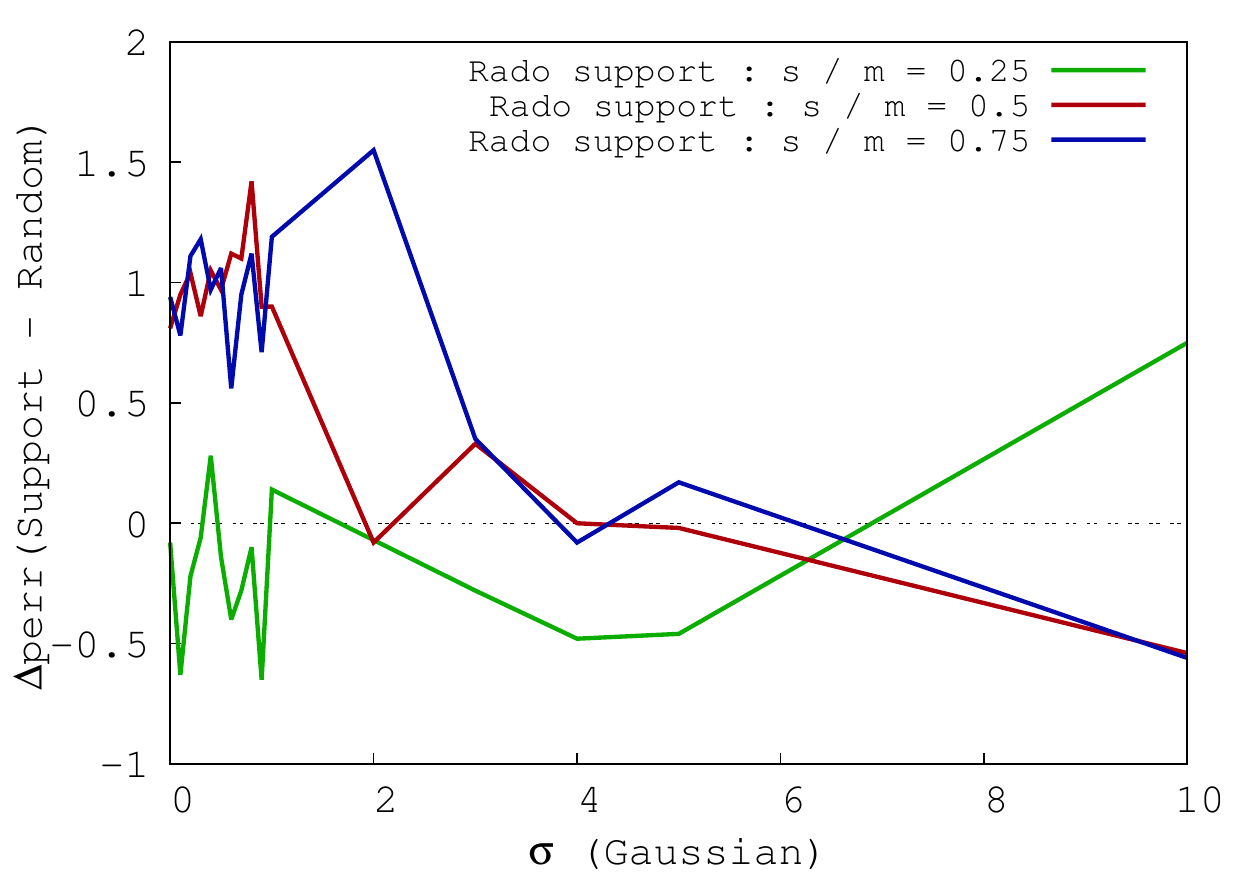}\\
Transfusion & Banknote\\ \hline
\includegraphics[width=0.40\columnwidth]{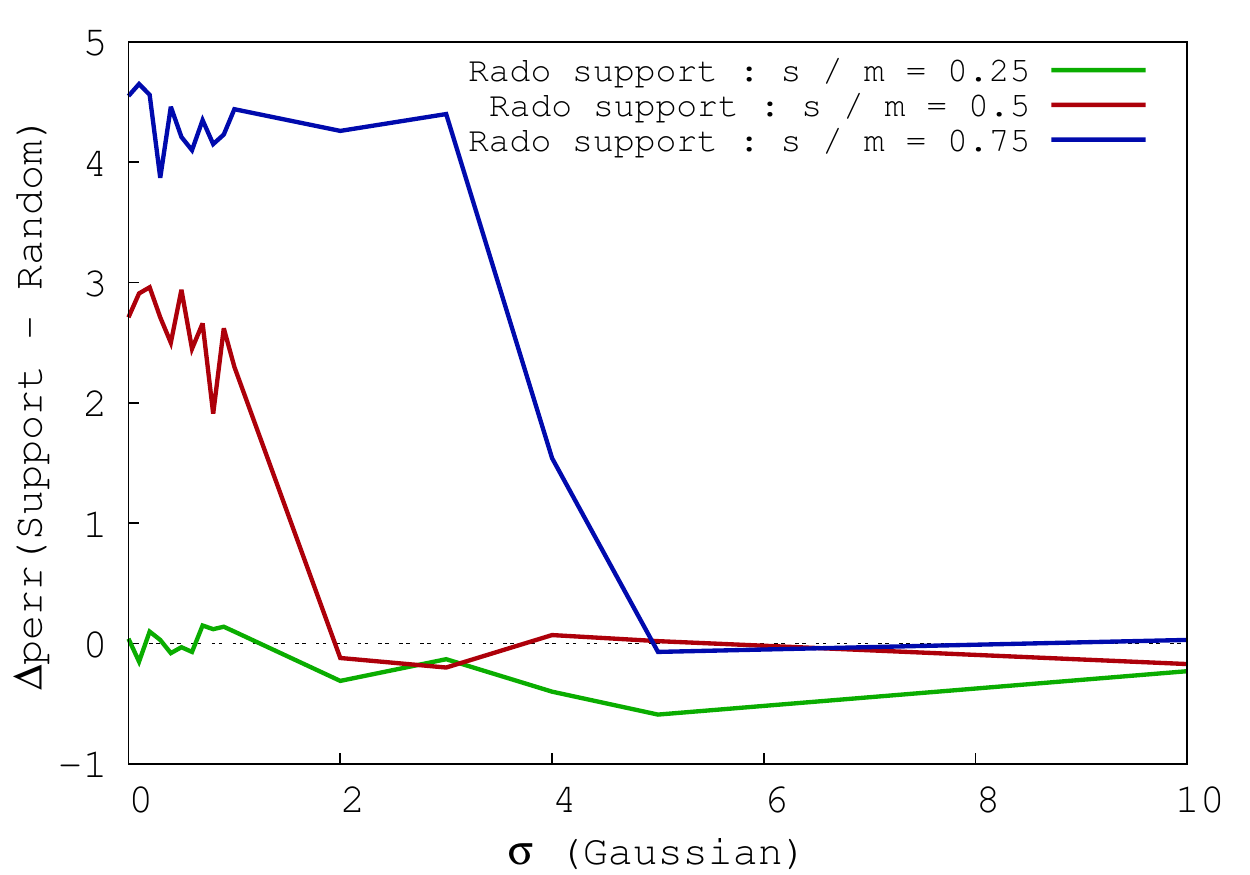}
& \includegraphics[width=0.40\columnwidth]{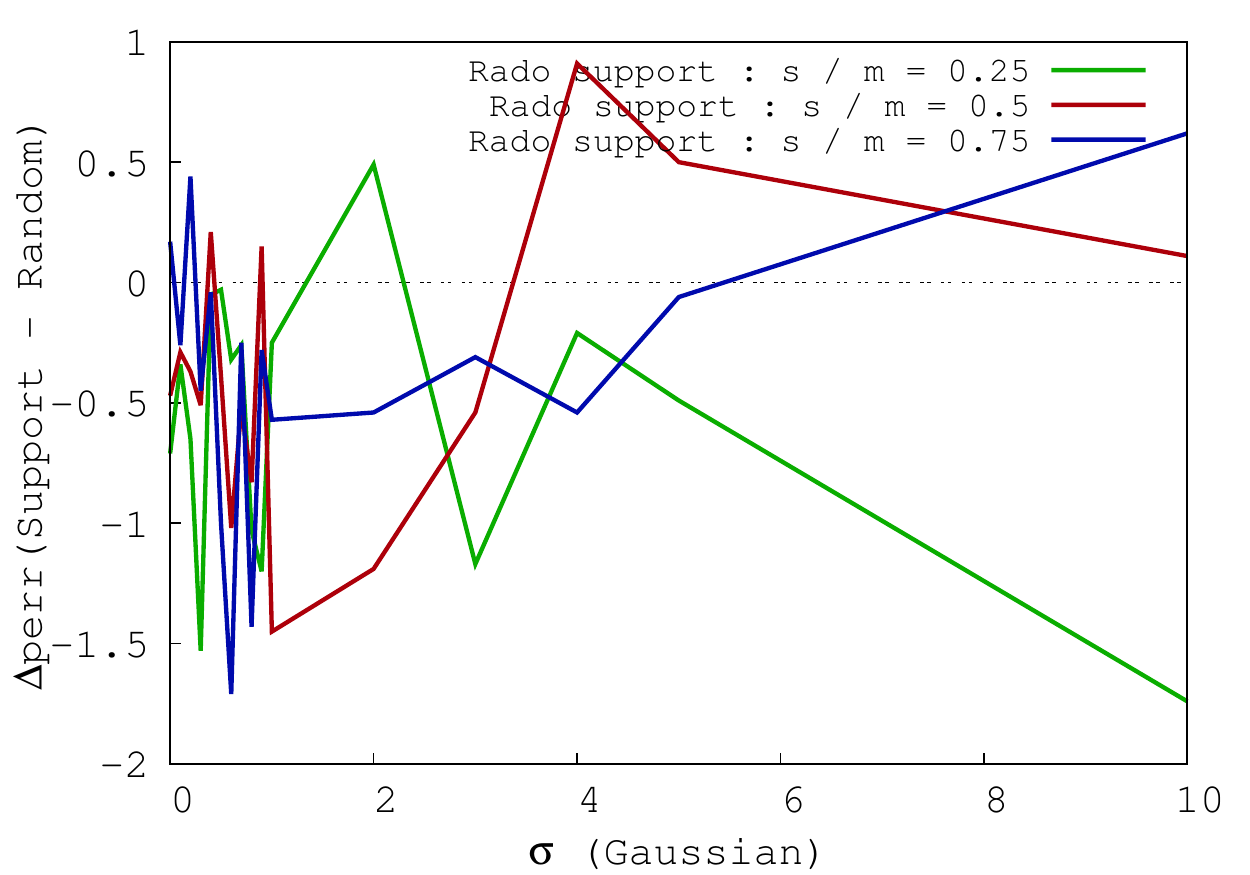}\\
Breastwisc & Ionosphere\\ 
\hline \hline
\end{tabular} 
\end{center}
\caption{Test error of \radoboost~trained with rados with fixed
  support minus test error of \radoboost~trained with plain random
  rados, as a function of the Gaussian mechanism's standard deviation $\upsigma$. Points below the $\Delta$perr = 0 line indicate smaller
  errors for the training with rados of fixed support.
$s$ is the support size ($m$ relates to the size of the training
fold), for three values, $s / m = 0.25$ (green), $s / m = 0.5$ (red)
and $s / m = 0.75$ (blue).}
  \label{t-s52_5} 
\end{table}

\begin{table}[t]
\begin{center}
\begin{tabular}{|c||c|}
\hline \hline
 \includegraphics[width=0.40\columnwidth]{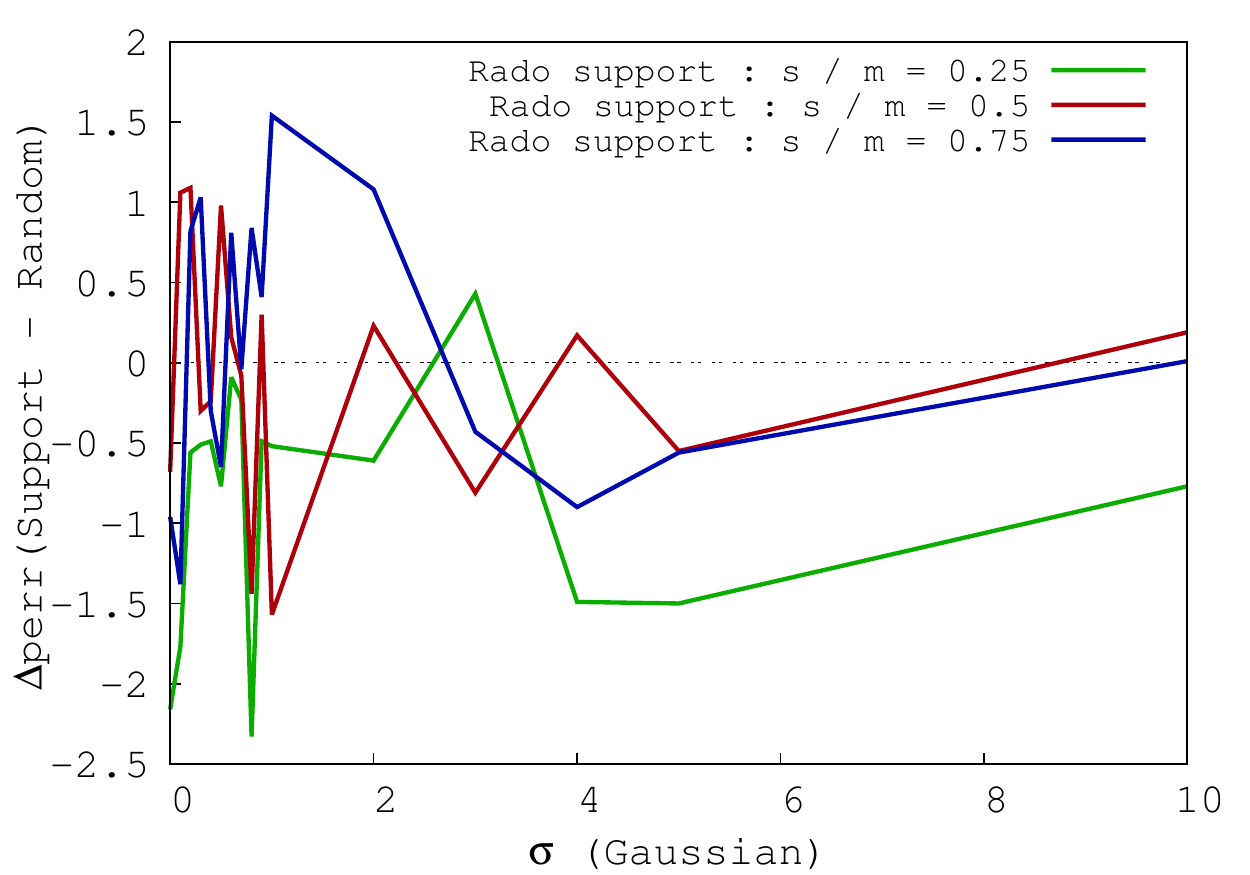}
& \includegraphics[width=0.40\columnwidth]{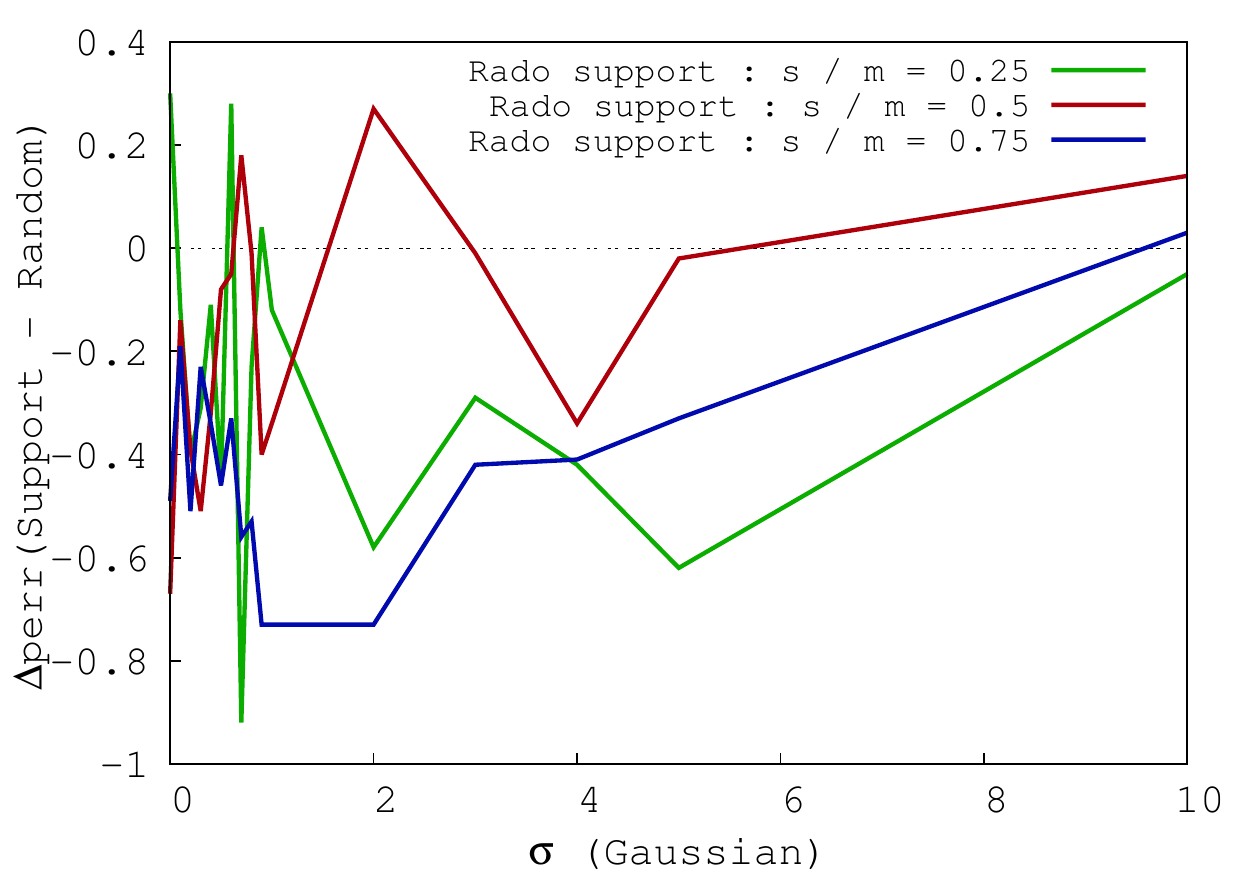}\\ 
Sonar & Winered \\ \hline
 \includegraphics[width=0.40\columnwidth]{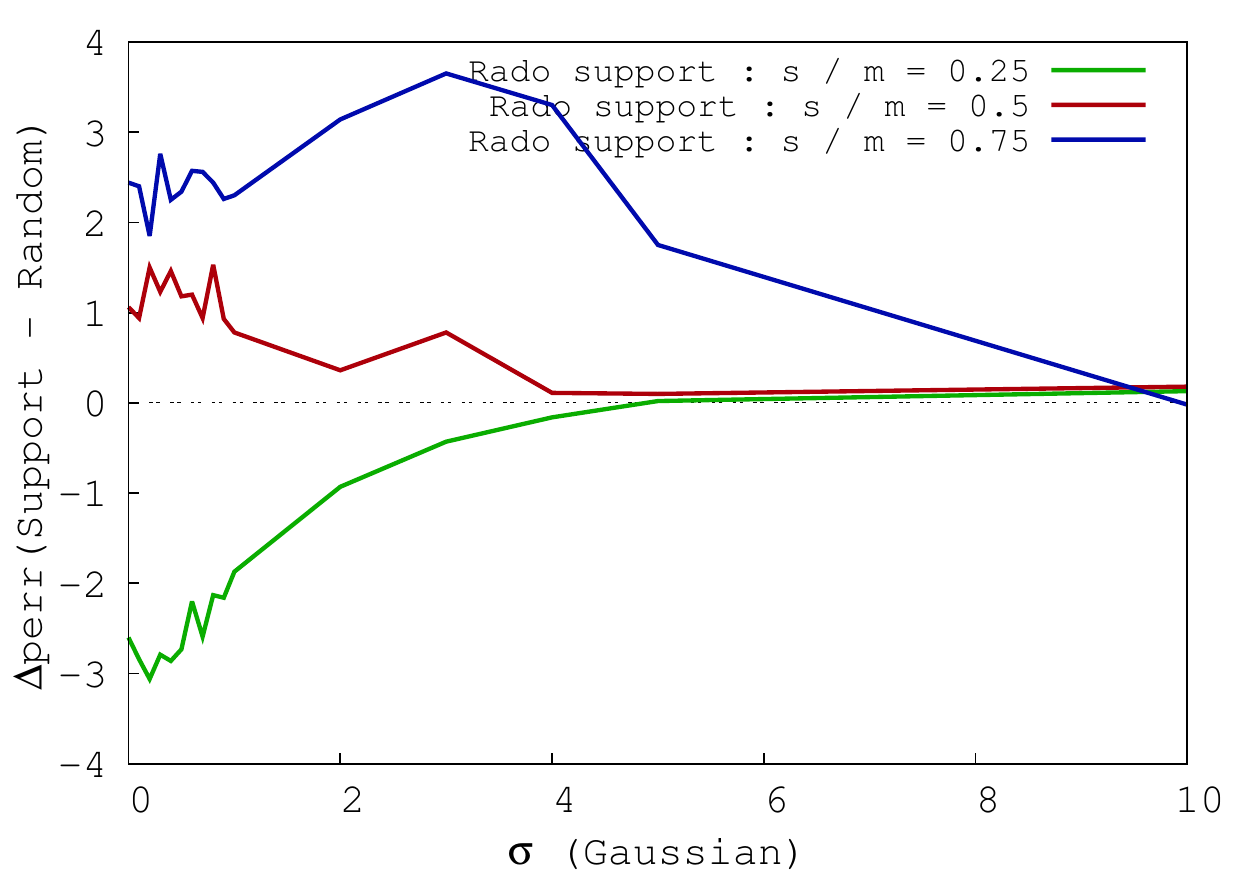}
& \includegraphics[width=0.40\columnwidth]{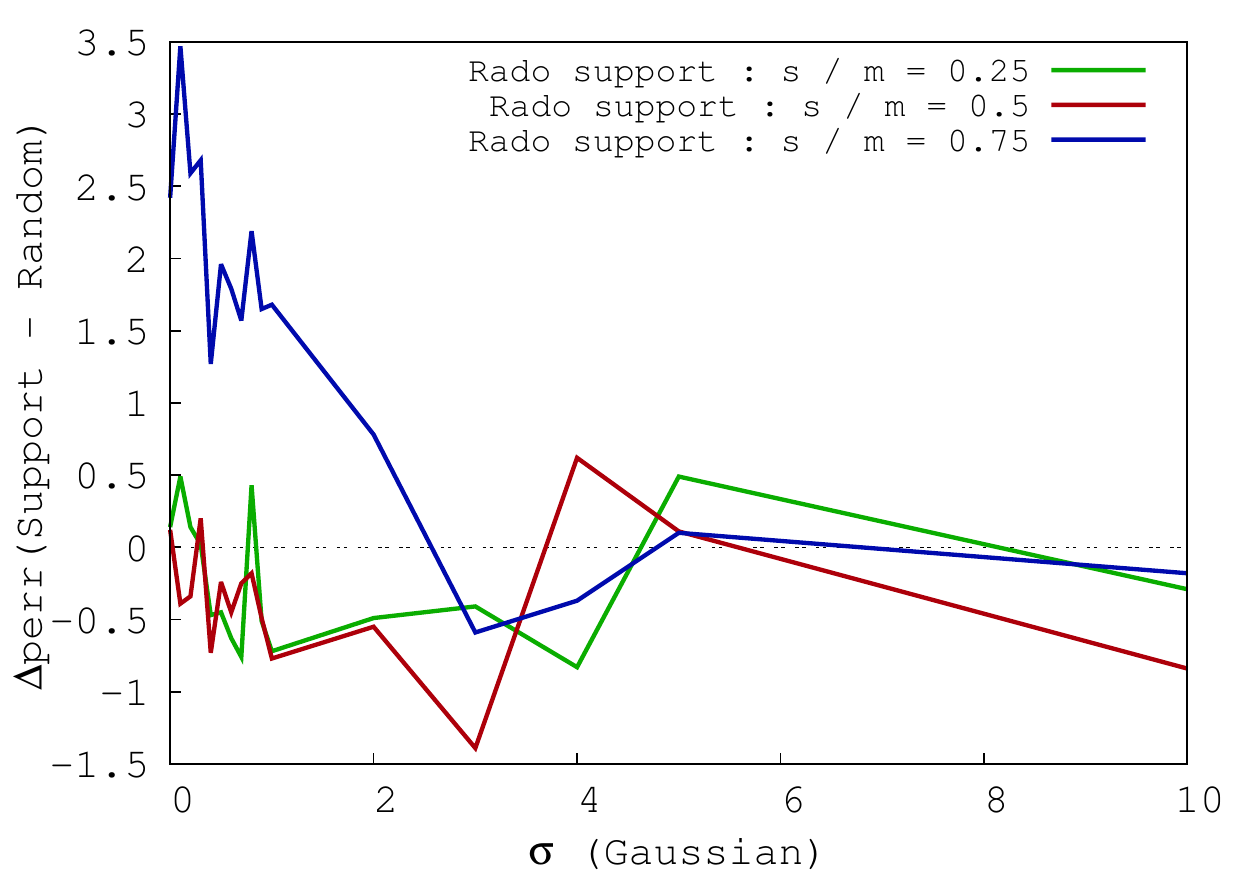}\\
Abalone & Wine-white\\ \hline
\includegraphics[width=0.40\columnwidth]{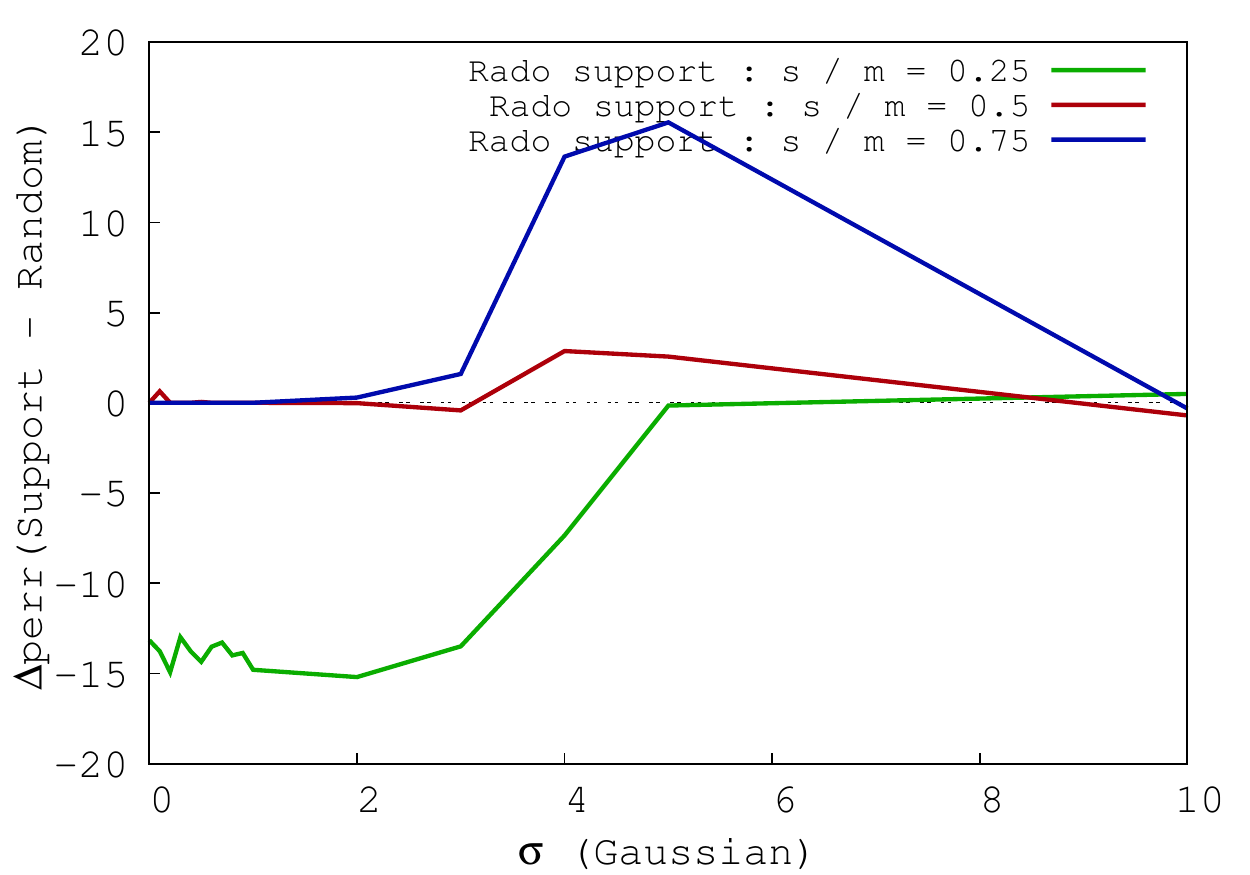}
& \includegraphics[width=0.40\columnwidth]{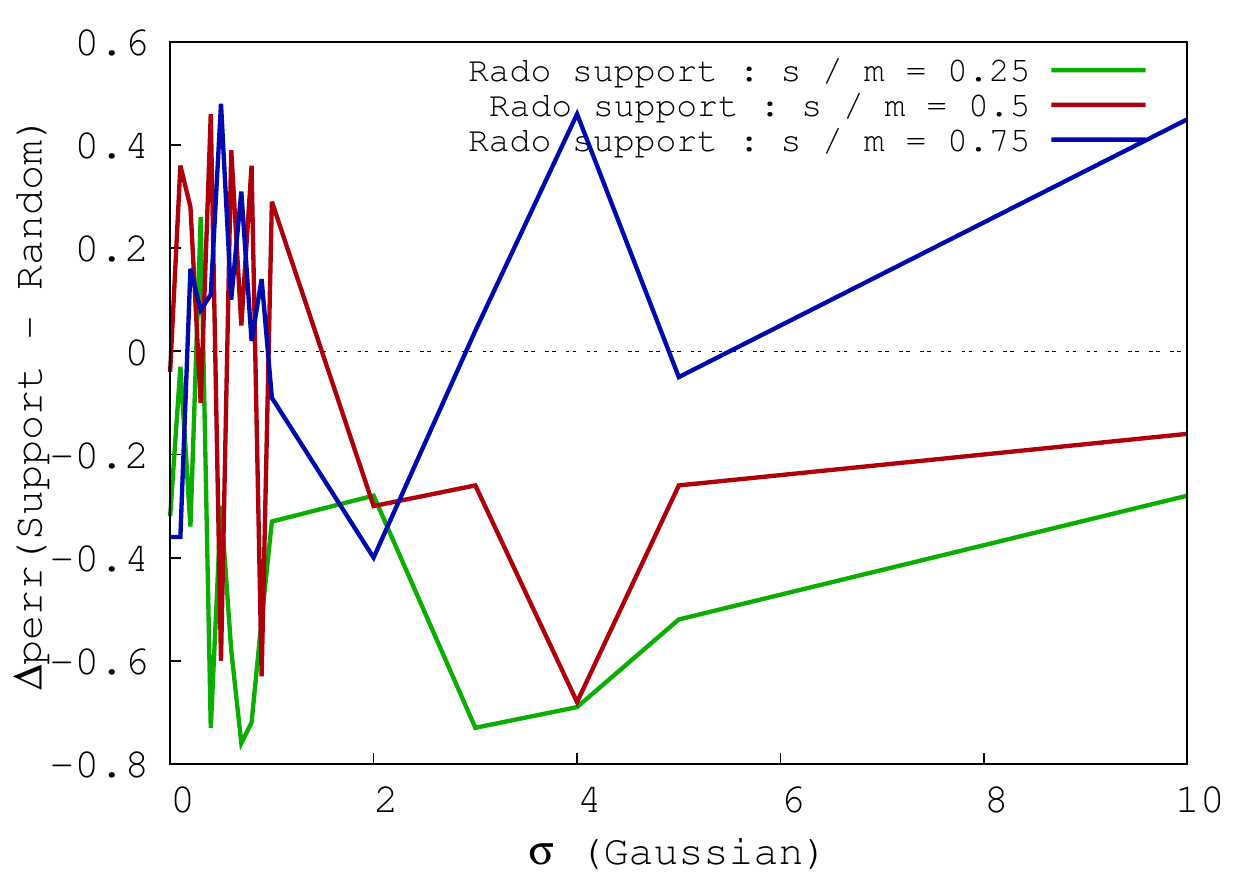}\\
Magic & Eeg\\ 
\hline \hline
\end{tabular} 
\end{center}
\caption{Test error of \radoboost~trained with rados with fixed
  support minus test error of \radoboost~trained with plain random
  rados (continued). Conventions follow Table \ref{t-s52_5}.}
  \label{t-s52_6} 
\end{table}

\end{document}